\documentclass[twoside]{article}

\usepackage[preprint]{aistats2027}
\usepackage[round]{natbib}

\usepackage[utf8]{inputenc}
\usepackage[T1]{fontenc}

\usepackage{amsmath,amssymb,amsthm,mathtools}
\usepackage{booktabs}
\usepackage{array}
\usepackage{graphicx}
\usepackage{float}
\usepackage{placeins}
\usepackage{microtype}
\usepackage{xcolor}
\usepackage{enumitem}
\usepackage{hyperref}
\usepackage{url}
\usepackage{xurl}
\usepackage{aliascnt}
\usepackage{tikz}
\usepackage{pgfplots}
\usepackage{algorithm}
\usepackage{algpseudocode}
\pgfplotsset{compat=1.17}
\usetikzlibrary{arrows.meta,positioning,decorations.pathreplacing,calc}
\hypersetup{hidelinks}

\theoremstyle{plain}
\newtheorem{theorem}{Theorem}[section]
\newaliascnt{proposition}{theorem}
\newtheorem{proposition}[proposition]{Proposition}
\aliascntresetthe{proposition}
\newaliascnt{lemma}{theorem}
\newtheorem{lemma}[lemma]{Lemma}
\aliascntresetthe{lemma}
\newaliascnt{corollary}{theorem}
\newtheorem{corollary}[corollary]{Corollary}
\aliascntresetthe{corollary}
\theoremstyle{definition}
\newaliascnt{definition}{theorem}

\aliascntresetthe{definition}
\theoremstyle{remark}
\newaliascnt{remark}{theorem}
\newtheorem{remark}[remark]{Remark}
\aliascntresetthe{remark}
\theoremstyle{plain}
\newaliascnt{fact}{theorem}
\newtheorem{fact}[fact]{Fact}
\aliascntresetthe{fact}

\usepackage[nameinlink,noabbrev,capitalise]{cleveref}
\crefname{theorem}{Theorem}{Theorems}
\Crefname{theorem}{Theorem}{Theorems}
\crefname{proposition}{Proposition}{Propositions}
\Crefname{proposition}{Proposition}{Propositions}
\crefname{lemma}{Lemma}{Lemmas}
\Crefname{lemma}{Lemma}{Lemmas}
\crefname{corollary}{Corollary}{Corollaries}
\Crefname{corollary}{Corollary}{Corollaries}
\crefname{remark}{Remark}{Remarks}
\Crefname{remark}{Remark}{Remarks}
\crefname{fact}{Fact}{Facts}
\Crefname{fact}{Fact}{Facts}

\definecolor{oiBlue}{RGB}{0,114,178}
\definecolor{oiOrange}{RGB}{213,94,0}
\definecolor{oiGreen}{RGB}{0,158,115}
\definecolor{oiGray}{RGB}{120,120,120}
\definecolor{slotFill}{RGB}{236,240,244}

\DeclareMathOperator{\TV}{TV}
\DeclareMathOperator{\Var}{Var}

\newcommand{\dd}{\,\mathrm d}
\newcommand{\E}{\mathbb E}
\newcommand{\R}{\mathbb R}
\newcommand{\one}{\mathbf 1}
\newcommand{\eps}{\varepsilon}

\newcommand{\Sep}{\operatorname{sep}}
\newcommand{\norm}[1]{\lVert #1\rVert_2}

\newcommand{\methodsFScoarse}{$9.0\%$ to $35.6\%$ and $10$ to $33$\,s}

\definecolor{mbFeedback}{RGB}{204,0,153}

\begin{document}

\twocolumn[
\aistatstitle{Closed-Form Noise Calibration Against Membership Inference for Random-Allocation DP-SGD}
\aistatsauthor{ Murat Bilgehan Ertan$^{*}$ \And Marten van Dijk$^{*}$ }
\runningauthor{Ertan, van Dijk}
\aistatsaddress{ CWI Amsterdam \And CWI Amsterdam } ]

\ifnum\statePaper=1 {\renewcommand{\thefootnote}{\fnsymbol{footnote}}\footnotetext[1]{Affiliated with Vrije Universiteit Amsterdam.}}\fi

\begin{abstract}
DP-SGD protects training data by adding Gaussian noise to clipped gradients. The amount of noise is usually chosen by running a numerical privacy accountant inside a search. We study DP-SGD with random allocation, where each epoch uses every record once, at a randomly chosen step.  For this setting we give a one-line formula that bounds the accuracy of every membership inference attack (MIA) on the trained model. With $M$ steps per epoch, $E$ epochs and noise multiplier $\sigma$, and with membership and non-membership equally likely a priori, the attack accuracy is at most $\frac12+\frac14\sqrt{(1+(e^{1/\sigma^2}-1)/M)^E-1}$. The formula comes from the chi-square divergence between a Gaussian distribution and a Gaussian mixture that dominates random allocation. It is interpretable and gives $\sigma$ in about a microsecond. Where applicable, our formula needs at most about half the noise of the state-of-the-art closed-form bound.  To measure how close the bound is, we also derive an exact expression for the attack accuracy of these two distributions and evaluate it numerically.  Calibrating to this exact expression requires $13.0\%$ to $20.2\%$ less noise than the formula in our main experiments, and since it is exact, no accountant that knows only $M$, $E$ and $\sigma$ can certify a smaller $\sigma$.  In training, the resulting $\sigma$ outperforms the formula and matches a published accountant in test accuracy.  It is found in seconds and certified in minutes, whereas every search we ran with that accountant took longer or returned at least $0.62\%$ more noise. We show that MIAs on the trained models stay below the bound.\footnote{\raggedright Code is available at {\Urlmuskip=0mu\def\UrlBreaks{\do\/}\url{https://github.com/bilgehanertan/closed-form-noise-calibration}}.}
\end{abstract}

\section{Introduction}
\label{sec:introduction}

DP-SGD trains a model on private data by clipping each record's gradient and adding Gaussian noise to every update \citep{AbadiEtAl2016}.  Its privacy analysis depends on how records are grouped into minibatches.  Most analyses assume Poisson subsampling, where each record joins each minibatch independently.  Training loops commonly use every record once per epoch instead \citep{ChuaEtAl2024,ChuaScalable2024}.  \emph{Random allocation}, also called balls-and-bins sampling, keeps this property.  In every epoch each record is sent to one of the $M$ steps, uniformly and independently, so the minibatch sizes vary \citep{ChuaBallsBins2025,FeldmanShenfeld2025,DongChenOzgur2025}.  \Cref{alg:dpsgd} states the algorithm.

For random allocation, \citet{ChuaBallsBins2025} and \citet{FeldmanShenfeld2025} reduce one epoch to two distributions built from Gaussian noise.  In this reduced form, an epoch without the record is $M$ independent standard Gaussian values, one for each step.  With the record, the value of one step, chosen uniformly at random, is also shifted by $1/\sigma$, the clipping norm divided by the noise scale.  We write $P_M$ and $Q_M$ for these two distributions and call them the Gaussian pair.  Detecting the record from one epoch of training is never easier than telling $Q_M$ from $P_M$.

A privacy accountant for random allocation turns this pair into a privacy guarantee for the whole training run.  Given $M$, $E$ and $\sigma$, it computes how the privacy loss of the pair adds up over the $E$ epochs and reports a guarantee such as $(\eps,\delta)$-differential privacy \citep{DworkEtAl2006Calibrating,DworkEtAl2006OurData}.  Accountants based on R\'enyi differential privacy \citep{Mironov2017} use the R\'enyi divergences of $Q_M$ from $P_M$, which are known in closed form at integer orders two and higher \citep{FeldmanShenfeld2025,DongChenOzgur2025}.  Accountants based on the privacy loss distribution \citep{KoskelaEtAl2020} compute the whole distribution of the privacy loss numerically, and \citet{FeldmanShenfeld2026} give a numerical accountant for random allocation.  An accountant tells how private a given $\sigma$ is, but choosing $\sigma$ needs the reverse, the smallest $\sigma$ that meets a privacy target.  The usual approach therefore runs the accountant inside a search over $\sigma$.  In our experiments, one run of the accountant of \citet{FeldmanShenfeld2026} with its default route in one direction at the $\sigma$ it certified for our training took $6.5$ minutes to $7.1$ hours on one CPU core (\cref{tab:calib-timing} in Appendix~\ref{app:experiments}), and a search repeats such runs.  Such a run is slow because its running time grows as the square of the number of grid points \citep[Theorem~4.3]{FeldmanShenfeld2026}, and certifying a $\sigma$ close to the smallest one that meets the target needs a fine grid.  A search with the accountant's fast Fourier route \citep[Appendix~F]{FeldmanShenfeld2026} took only $4.9$ to $10.9$ minutes, but it certified up to $4.07\%$ more noise than the exact $\sigma$, more as $M$ grew.  For a detailed comparison and discussion of both routes, see Appendix~\ref{app:experiments}.  We ask whether this search can be replaced by a formula.  It can.  We give a formula that bounds the accuracy of every membership inference attack on the trained model, at every $M$, $E$ and $\sigma$, and that gives $\sigma$ in about a microsecond.  We also compute the exact value that the formula bounds.  Calibrating to it takes seconds and needs slightly less noise than the $\sigma$ that accountant certifies.  No accountant that knows only $M$, $E$ and $\sigma$ can certify a $\sigma$ below the exact one (\cref{sec:mechanism}).  A finite computation in ball arithmetic certifies in minutes that the exact $\sigma$ meets the target (Appendix~\ref{app:certificate}).

We measure privacy by membership inference.  An attacker sees the final model of a training run and must decide whether a given record was in the training data.  Let $P$ and $Q$ be the probability distributions of the final model without and with the record. Their total variation (TV) distance  is
\begin{equation}
 d_{\TV}(P,Q)=\sup_A\,\lvert P(A)-Q(A)\rvert ,
 \label{eq:tv-def}
\end{equation}
where $A$ ranges over all events.  When both answers are equally likely a priori, the optimal test is correct with probability \citep{ChatzikokolakisEtAl2023,MahloujifarEtAl2022}
\begin{equation}
 \operatorname{Acc}^*(P,Q)=\frac12+\frac12\,d_{\TV}(P,Q).
 \label{eq:sep-tv-acc}
\end{equation}
For any two distributions we call $\operatorname{Acc}^*(P,Q)$ their attack accuracy.  We write $\operatorname{Acc}^*$ alone for the distributions of the final model.

Total variation also has a meaning in $f$-differential privacy \citep{DongRothSuGDP}.  The trade-off function $T(P,Q)(\alpha)$ is the smallest type-II error of a test of $P$ against $Q$ with type-I error at most $\alpha$.  Random guessing gives $1-\alpha$, and
\begin{equation}
 d_{\TV}(P,Q)=\sup_{0\le\alpha\le1}\bigl\{1-\alpha-T(P,Q)(\alpha)\bigr\}
 \label{eq:tv-tradeoff}
\end{equation}
is the largest gap between the two curves (Appendix~\ref{app:background}).  So $d_{\TV}\le\lambda$ means $T(P,Q)(\alpha)\ge1-\lambda-\alpha$ for every $\alpha$. \citet{ErtanVanDijk2026} call the largest distance of the curve from the random-guessing line its separation, which is $d_{\TV}/\sqrt2$.  \Cref{thm:finite-rule} bounds the total variation of the Gaussian pair from above at every $\sigma$, and \cref{thm:exact-tv} gives it exactly.

Our starting point is the chi-square divergence
\begin{equation}
 \chi^2(Q\|P)=\E_P\Bigl(\frac{\dd Q}{\dd P}-1\Bigr)^2 ,
 \label{eq:chi2-def}
\end{equation}
which bounds total variation by $d_{\TV}(P,Q)\le\frac12\sqrt{\chi^2(Q\|P)}$ \citep{SasonVerdu2016}.  For the Gaussian pair it is exactly $(e^{1/\sigma^2}-1)/M$ (\cref{prop:one-epoch}).  Since the allocation is redrawn every epoch, the reference pair for $E$ epochs is the product pair $(P_M^{\otimes E},Q_M^{\otimes E})$, and $1+\chi^2$ multiplies over epochs (\cref{fig:mechanism}).  By a known reduction, DP-SGD with random allocation is never easier to attack than this product pair (\cref{sec:mechanism}).  We show that for every $M$, $E$ and $\sigma$, the final model of DP-SGD with random allocation satisfies
\begin{equation}
 \operatorname{Acc}^*
 \le \frac12+\frac14
 \sqrt{\left(1+\frac{e^{1/\sigma^2}-1}{M}\right)^{\!E}-1}\,.
 \label{eq:main-bound}
\end{equation}
Every quantity on the right is set before training, so the inequality can be solved for $\sigma$.  \Cref{fig:mechanism} shows where the division by $M$ and the power $E$ come from.

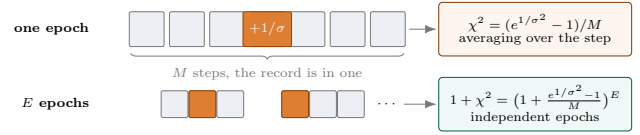
\begin{figure}[t]
\centering
\resizebox{\columnwidth}{!}{%
\begin{tikzpicture}[
 box/.style={draw=oiGray,fill=slotFill,minimum width=6.5mm,minimum height=6.5mm,rounded corners=1pt},
 hit/.style={box,draw=oiOrange!70!black,fill=oiOrange!80,text=white},
 every node/.style={font=\scriptsize}]
\node[anchor=east,font=\scriptsize\bfseries] at (-0.2,0) {one epoch};
\foreach \i in {1,...,7}{\node[box] at (0.75*\i,0) {};}
\node[hit] at (0.75*4,0) {$+1/\sigma$};
\draw[decorate,decoration={brace,amplitude=3pt,mirror},oiGray]
  (0.45,-0.45)--(5.55,-0.45) node[midway,below=4pt] {$M$ steps, the record is in one};
\node[draw=oiOrange!60!black,fill=oiOrange!6,rounded corners=2pt,align=center,
      minimum width=3.6cm,minimum height=1.0cm] at (8.0,0)
      {$\chi^2=(e^{1/\sigma^2}-1)/M$\\[-1pt] averaging over the step};
\draw[-{Latex[length=2mm]},oiGray] (5.65,0)--(6.15,0);
\node[anchor=east,font=\scriptsize\bfseries] at (-0.2,-1.4) {$E$ epochs};
\foreach \x/\h in {0.75/2,3.0/1}{
  \foreach \i in {1,2,3}{\node[box,minimum width=5mm,minimum height=5mm] at (\x+0.52*\i,-1.4) {};}
  \node[hit,minimum width=5mm,minimum height=5mm] at (\x+0.52*\h,-1.4) {};
}
\node at (5.25,-1.4) {$\cdots$};
\node[draw=oiGreen!60!black,fill=oiGreen!6,rounded corners=2pt,align=center,
      minimum width=3.6cm,minimum height=1.0cm] at (8.0,-1.4)
      {$1+\chi^2=\bigl(1+\tfrac{e^{1/\sigma^2}-1}{M}\bigr)^E$\\[1pt] independent epochs};
\draw[-{Latex[length=2mm]},oiGray] (5.55,-1.4)--(6.15,-1.4);
\end{tikzpicture}}
\caption{The record is in one of the $M$ steps of each epoch, chosen uniformly at random, so averaging over the step divides the chi-square divergence by $M$, and $1+\chi^2$ multiplies over the independent epochs.}
\label{fig:mechanism}
\end{figure}

\textbf{Contributions.}
\begin{enumerate}
\item \emph{A closed-form bound and noise multiplier} (\cref{thm:finite-rule}).  We derive the bound \eqref{eq:main-bound} from the known chi-square divergence of one epoch and solve the bound for $\sigma$ in closed form (\cref{cor:inverse}).  Both hold at every finite $M$ and $E$, and the noise multiplier takes about a microsecond to compute, while a search over $\sigma$ takes seconds to hours.
\item \emph{The exact total variation at every $E$} (\cref{thm:exact-tv}).  The bound \eqref{eq:main-bound} uses only the second moment of the likelihood ratio, which multiplies over the epochs.  We derive the exact total variation as an integral of the moments of complex order $\frac12+\mathrm iu$ of the likelihood ratio, which multiply over the epochs in the same way, and we reduce the moment of one epoch to a one-dimensional integral (\cref{thm:exact-tv}).  We find the smallest $\sigma$ that meets a target by a root search on this total variation and call it the exact $\sigma$, the least noise that any accountant that knows only $M$, $E$ and $\sigma$ can certify.  We prove that at fixed $\sigma$ the total-variation bound of \cref{thm:finite-rule} divided by the exact value tends to $\sqrt{\pi/2}$ as $M$ grows with $E/M\to0$.  On nine single-epoch targets the closed-form $\sigma$ of \cref{cor:inverse} needs $7.4\%$ to $16.7\%$ more noise than the exact $\sigma$.
\item \emph{The limit at fixed noise} (\cref{thm:fixed-rate}).  We prove that, as $M\to\infty$ with $E/M\to r<\infty$ and $\sigma$ fixed, the trade-off curve of the pair converges to Gaussian differential privacy with $\mu=\sqrt{r(e^{1/\sigma^2}-1)}$.  \citet[Theorem~4.2]{VanDijkErtan2026} show that in this limit the curve is asymptotically at least that of Gaussian differential privacy with the same $\mu$, so that $E=o(M)$ suffices for the curve to approach random guessing, and we show that $E=o(M)$ is also necessary.
\item \emph{Training and membership inference attacks} (\cref{sec:experiments}).  We train ResNet-18 on CIFAR-10, a linear classifier on CIFAR-100 and RoBERTa-base on SST-2 and AG News.  We show that the exact $\sigma$ gives the mean test or validation accuracy of the $\sigma$ of the accountant of \citet{FeldmanShenfeld2026} to within $0.06$ points, and that it is found in about ten seconds on one CPU core and certified to meet the target in minutes.  Over the closed-form $\sigma$ of \cref{cor:inverse}, the exact $\sigma$ raises the mean test or validation accuracy of every model, by $0.43$ to $2.38$ points.  We also show empirically that none of our membership inference attacks has an attack accuracy above our bound \eqref{eq:main-bound}.  On a linear model the optimal attack almost reaches the exact attack accuracy of the Gaussian pair, while attacks on ResNet-18 stay close to random guessing.
\end{enumerate}

\textbf{Organization.} \Cref{sec:mechanism} sets up the training algorithm, the membership test and the Gaussian pair.  \Cref{sec:closed-form} derives the closed-form bound and its noise multiplier, and \cref{sec:exact} computes the exact total variation.  \Cref{sec:experiments} reports the experiments, and \cref{sec:related} discusses related work and limitations.  

\section{Setting}
\label{sec:mechanism}

We now describe the training algorithm and the membership test, and we show that no membership test on the released model has a higher attack accuracy than the optimal test on a Gaussian pair, so \cref{sec:closed-form,sec:exact} study only this pair.

\textbf{The training algorithm.} \Cref{alg:dpsgd} is the algorithm we analyze, DP-SGD with random allocation.  A \emph{record} is the unit that is added or removed.  We use the zero-out adjacency \citep{ChuaEtAl2024,ChuaBallsBins2025}, and we pad the dataset to a public size $N_0$ with ghost records $\bot$, whose gradient is zero.  The $N_0$ records are put in a uniformly random order, which gives each record an index in $\{1,\dots,N_0\}$ (line~\ref{line:place}).  In every epoch each index is assigned to one of the $M$ steps, uniformly and independently (line~\ref{line:alloc}).  At each step the gradients of the records assigned to it are clipped one by one, $\operatorname{clip}_C(v)=v/\max\{1,\|v\|/C\}$, and summed (line~\ref{line:clip}).  Gaussian noise $N(0,\sigma^2C^2I)$ is added, and the result is divided by the public expected batch size $B_0=N_0/M$ (line~\ref{line:noise}).  A step with no genuine record still adds noise and still updates the model, so every run makes exactly $EM$ updates. Because the divisor $B_0$ is public, the scale of the update does not depend on the data \citep{WangEtAl2026Rethinking}.  Only the final model is released (line~\ref{line:release}).

The data enter the algorithm only through the records.  Everything else is fixed in advance and does not depend on the data.  This includes $N_0$, $M$, $E$, $\sigma$, $C$, the optimizer with its learning-rate schedule, and the distribution of the initial model.  The loss $\ell(w;x)$ depends on one record, so layers that mix the records of a step, such as batch normalization, are excluded, and group normalization can be used instead.  The loss may use data augmentation.  Its randomness, like that of the initial model and of the optimizer, is drawn independently of the data, the indices, the allocation and the noise. The random order, the allocations of all epochs and the noise of all steps are independent of each other, and every step draws fresh noise.

\begin{algorithm}[t]
\caption{DP-SGD with random allocation}
\label{alg:dpsgd}
\begin{algorithmic}[1]
\Require dataset $D$ with $|D|\le N_0$, public $N_0,M,E,\sigma,C$ and $B_0=N_0/M$, an optimizer, an initial model $w$
\State pad $D$ with $N_0-|D|$ ghost records $\bot$ and put the $N_0$ records in uniformly random order \label{line:place}
\For{$e=1,\dots,E$}
  \State assign each index $i\in\{1,\dots,N_0\}$ to a step in $\{1,\dots,M\}$, uniformly and independently \label{line:alloc}
  \For{$m=1,\dots,M$}
    \State $S\gets\sum_{x\text{ in step }m}\operatorname{clip}_C(\nabla\ell(w;x))$ \label{line:clip}
    \State $S\gets\bigl(S+N(0,\sigma^2C^2I)\bigr)/B_0$ \Comment{also if step $m$ has no genuine record} \label{line:noise}
    \State $w\gets$ optimizer update of $w$ with $S$ \label{line:update}
  \EndFor
\EndFor
\State \Return $w$ \label{line:release}
\end{algorithmic}
\end{algorithm}

\textbf{The membership test.} Let $D$ be a dataset and $D^+=D\cup\{x_\star\}$ the same dataset with one target record added.  The attacker receives the model released by a run on $D$ or on $D^+$, each with probability $1/2$, and must say which.  The optimal test is correct with probability \eqref{eq:sep-tv-acc}, and \eqref{eq:tv-tradeoff} relates its error to the trade-off function.

\textbf{The Gaussian pair.} Divide each noisy sum by $\sigma C$.  Then the noise is standard Gaussian, and the target's scaled clipped gradient has length at most $t=1/\sigma$.  The Gaussian pair of the introduction is
\begin{equation}
 P_M=N(0,I_M),\qquad
 Q_M=\frac1M\sum_{j=1}^{M}N(te_j,I_M),
 \label{eq:one-record-pair}
\end{equation}
where $e_j$ is the $j$th standard basis vector.  With the target, one coordinate is shifted by $t$, at a step the attacker does not know.  Allocations are redrawn every epoch, so $E$ epochs give the product pair $(P_M^{\otimes E},Q_M^{\otimes E})$.

Theorem~3.1 of \citet{ChuaBallsBins2025}, a special case of Theorem~3.1 of \citet{FeldmanShenfeld2025}, shows that for one epoch of \cref{alg:dpsgd} the Gaussian pair dominates the noisy updates, even when each gradient depends on the earlier updates.  There exist a model, a loss and an optimizer for which, on every dataset $D$ that does not contain $x_\star$ and has $|D|+1\le N_0$, the model that \cref{alg:dpsgd} releases has exactly the distribution $P_M^{\otimes E}$ on $D$ and $Q_M^{\otimes E}$ on $D^+$ (Appendix~\ref{app:attain}).  An accountant whose guarantee holds for every run of \cref{alg:dpsgd} with the same $M$, $E$ and $\sigma$ must also cover this run, so for a target $\lambda$ on the total variation it cannot certify a $\sigma$ below the exact $\sigma$ of \cref{thm:exact-tv}.  By Lemma~D.4 of \citet{ChuaBallsBins2025}, one randomized map turns the pair into the noisy updates under both hypotheses, so at every type-I error no test on the updates has a smaller type-II error than the optimal test on the pair.  Adaptive composition \citep[Theorem~10]{ZhuDongWang2022}, Lemma~D.4 of \citet{ChuaBallsBins2025} applied once more to the composed pair, and post-processing (Appendix~\ref{app:background}) extend this to the model released after $E$ epochs, so for every dataset $D$ with $|D|+1\le N_0$ and $D^+=D\cup\{x_\star\}$, the probability distributions $P_{\rm out}$ and $Q_{\rm out}$ of the model released on $D$ and on $D^+$ satisfy
\begin{equation}
 \begin{gathered}
 d_{\TV}(P_{\rm out},Q_{\rm out})\le d_{\TV}(P_M^{\otimes E},Q_M^{\otimes E}),\\
 T(P_{\rm out},Q_{\rm out})\ge T(P_M^{\otimes E},Q_M^{\otimes E}).
 \end{gathered}
 \label{eq:transfer}
\end{equation}

\section{The Closed-Form Bound}
\label{sec:closed-form}
\label{sec:fixed-rate}

We now bound the total variation of the Gaussian pair through its chi-square divergence and solve the bound for $\sigma$.  The chi-square divergence has a closed form, because averaging over the $M$ steps divides it by $M$ and the $E$ epochs are independent.  Throughout the paper we defer all proofs to the appendix.

Under $P_M$ the coordinates $X_1,\dots,X_M$ are independent standard Gaussians, and the likelihood ratio of \eqref{eq:one-record-pair} is the average
\begin{equation}
 L_M=\frac{\dd Q_M}{\dd P_M}=\frac1M\sum_{j=1}^{M}Y_j,\qquad
 Y_j=\exp\!\left(tX_j-\tfrac{t^2}{2}\right).
 \label{eq:one-epoch-lr}
\end{equation}
Given the step $j$, only coordinate $j$ changes, and its density ratio is $Y_j$.  Each $Y_j$ has mean $1$ and variance $e^{t^2}-1$, a standard integral that Proposition~D.2 of \citet{VanDijkErtan2026} also computes.  Since $L_M$ has mean $1$, the chi-square divergence is the variance of $L_M$.  Averaging $M$ independent copies divides the variance by $M$, which proves the following (Appendix~\ref{app:one-epoch-proof}).

\begin{proposition}[One epoch]
\label{prop:one-epoch}
For the pair \eqref{eq:one-record-pair}, $\chi^2(Q_M\|P_M)=\E_{P_M}(L_M-1)^2=(e^{1/\sigma^2}-1)/M$.
\end{proposition}

This value follows from the order-two R\'enyi divergence $D_2(Q_M\|P_M)=\ln\E_{P_M}L_M^2$ of random allocation \citep{FeldmanShenfeld2025,DongChenOzgur2025}, since $\chi^2=e^{D_2}-1$.  \citet{LiewTakahashi2022} computed the same divergence as the R\'enyi divergence of the shuffle Gaussian mechanism.  For product distributions R\'enyi divergence adds \citep[Theorem~28]{vanErvenHarremoes2014}, so $1+\chi^2$ multiplies over epochs and
\begin{equation}
 1+\chi^2(Q_M^{\otimes E}\|P_M^{\otimes E})
 =\left(1+\frac{e^{1/\sigma^2}-1}{M}\right)^E=:1+c_E .
 \label{eq:tensorize}
\end{equation}

The chi-square divergence does not determine the total variation, but it bounds it.

\begin{remark}[From chi-square to total variation]
\label{rem:cs-extremal}
For any pair with $\chi^2(Q\|P)=c$ \citep[Eq.~(7)]{SasonVerdu2016},
\begin{equation}
 d_{\TV}(P,Q)\le U(c):=\begin{cases}\frac12\sqrt c,&0\le c\le1,\\[2pt]
 \frac{c}{1+c},&c\ge1.\end{cases}
 \label{eq:sharp-envelope}
\end{equation}
Both branches are attained by distributions on two points, so no bound that depends only on $c$ can be smaller (Appendix~\ref{app:sharp-envelope-proof}).
\end{remark}

\begin{theorem}[Closed-form bound]
\label{thm:finite-rule}
For all integers $M,E\ge1$ and every $\sigma>0$, with $c_E$ as in \eqref{eq:tensorize},
\begin{equation}
 \begin{gathered}
 d_{\TV}(P_M^{\otimes E},Q_M^{\otimes E})\le U(c_E),\\
 \operatorname{Acc}^*(P_M^{\otimes E},Q_M^{\otimes E})\le\frac12+\frac12U(c_E).
 \end{gathered}
 \label{eq:finite-rule}
\end{equation}
\end{theorem}

\Cref{prop:one-epoch}, \eqref{eq:tensorize} and \eqref{eq:sharp-envelope} give the theorem.  Since $U(c)\le\frac12\sqrt c$ for every $c$, the theorem and \eqref{eq:transfer} give the bound \eqref{eq:main-bound} for training runs.  To see the size of $c_E$, put $x=E(e^{1/\sigma^2}-1)/M$.  Then $x\le c_E\le e^x-1$ (Appendix~\ref{app:sharp-envelope-proof}), so $c_E$ is close to $x$ when $x$ is small.

Because $U$ increases with $c$ and $c_E$ decreases as $\sigma$ grows, the bound can be solved for $\sigma$.

\begin{corollary}[Noise multiplier for a target total variation]
\label{cor:inverse}
Let $\lambda\in(0,1)$ be a target for $d_{\TV}$, and let $c_\lambda=4\lambda^2$ if $\lambda\le1/2$ and $c_\lambda=\lambda/(1-\lambda)$ if $\lambda>1/2$, so that $U(c_\lambda)=\lambda$.  The noise multiplier
\begin{equation}
 \sigma(M,E,\lambda)=
 \left[\ln\!\left\{1+M\left((1+c_\lambda)^{1/E}-1\right)\right\}\right]^{-1/2}
 \label{eq:inverse-sigma}
\end{equation}
gives $c_E=c_\lambda$, hence $d_{\TV}(P_M^{\otimes E},Q_M^{\otimes E})\le\lambda$, and $\operatorname{Acc}^*\le\frac12+\frac{\lambda}2$ for the final model of \cref{alg:dpsgd} by \eqref{eq:transfer}.
\end{corollary}

All targets in our experiments are at most $0.4$, so only the first branch of $U$ is used.  The same equation can also be solved for $E$ or for $M$.

To read \eqref{eq:inverse-sigma}, take one epoch and a fixed target $\lambda<1/2$.  Then $\sigma=[\ln(1+4\lambda^2M)]^{-1/2}=(1+o(1))/\sqrt{\ln M}$ as $M$ grows.  The proof of Theorem~6.1 of \citet{ErtanVanDijk2026}, which is stated for one-epoch shuffling and applies to the same Gaussian pair, shows that for $M\ge2$ the pair has $d_{\TV}(P_M,Q_M)\ge\frac12(1-1/\sqrt{4\pi\ln M})(1-O(1/M))$ unless $\sigma\ge1/\sqrt{2\ln M}$.  This lower bound tends to $\frac12$ as $M$ grows, and $\lambda<\frac12$, so every $\sigma<1/\sqrt{2\ln M}$ leaves $d_{\TV}(P_M,Q_M)$ above $\lambda$ for large $M$.  Hence for large $M$ the smallest $\sigma$ that meets the target is at least $1/\sqrt{2\ln M}$, and the closed form is at most $\sqrt2\,(1+o(1))$ times this smallest $\sigma$.

\textbf{The limit at fixed noise.} The bound depends on $M$ and $E$ only through $c_E$, and at fixed $\sigma$, $c_E\to0$ exactly when $E/M\to0$.  The next theorem shows that the exact total variation behaves in the same way, and it finds the limit when $E$ grows in proportion to $M$.  Write $\Phi$ for the standard normal distribution function and $G_\mu(\alpha)=\Phi(\Phi^{-1}(1-\alpha)-\mu)$ for the trade-off function of $N(0,1)$ against $N(\mu,1)$.  It defines Gaussian differential privacy (GDP) \citep{DongRothSuGDP}.

\begin{theorem}[Gaussian differential privacy limit]
\label{thm:fixed-rate}
Fix $\sigma>0$ and let $M\to\infty$.  Let $E_M$ be positive integers with $E_M/M\to r$ for some $r\in[0,\infty)$.  Put $\mu=\sqrt{r(e^{1/\sigma^2}-1)}$.  Then
\begin{equation}
 \sup_{\alpha\in[0,1]}
 \left|T(P_M^{\otimes E_M},Q_M^{\otimes E_M})(\alpha)
       -G_\mu(\alpha)\right|\longrightarrow0 .
 \label{eq:gdp-limit}
\end{equation}
Hence $\operatorname{Acc}^*(P_M^{\otimes E_M},Q_M^{\otimes E_M})\to\Phi(\mu/2)$.  For any sequence $E_M$, the total variation tends to $0$ if and only if $E_M/M\to0$.
\end{theorem}

The proof applies a central limit theorem to the log-likelihood ratio summed over the epochs (Appendix~\ref{app:gdp-proof}).  The parameter $\mu$ of the limit is the one \citet{BuEtAl2020} found for Poisson subsampling.  For the same pair, \citet[Theorem~4.2]{VanDijkErtan2026} prove that the curve is asymptotically at least $G_\mu$, which is one half of the limit.  \Cref{fig:attack-empirical} in Appendix~\ref{app:numerics} checks the limit by Monte Carlo sampling at $M$ as small as $64$.  Setting the limit $\Phi(\mu/2)$ equal to $\frac12+\frac\lambda2$ with $r=E/M$ gives a second one-line noise multiplier, but a limit is not a bound, so unlike the closed form of \cref{cor:inverse}, the second multiplier carries no guarantee at finite $M$ and $E$ (Appendix~\ref{app:gdp-proof}).

When $E_M/M\to0$, the exact total variation and the bound both tend to zero at the same rate.  Put $s_M=\sqrt{E_M(e^{1/\sigma^2}-1)/M}$.  The exact total variation is $s_M(1+o(1))/\sqrt{2\pi}$, and the bound of \cref{thm:finite-rule} is $s_M(1+o(1))/2$ (Appendix~\ref{app:small-div}).  The constant $\sqrt{2\pi}$ comes from the Gaussian limit. For small $s$, the total variation between $N(0,1)$ and $N(s,1)$ is $2\Phi(s/2)-1=s/\sqrt{2\pi}+O(s^3)$, because the standard normal density at zero is $1/\sqrt{2\pi}$.  So the closed form depends on $M$, $E$ and $\sigma$ in the same way as the exact value, and it is larger by a factor that tends to $\sqrt{\pi/2}\approx1.2533$.  \Cref{sec:exact} computes this factor at finite sizes.

\section{The Exact Total Variation}
\label{sec:exact}

We now compute the total variation of the Gaussian pair exactly, to measure how much extra noise the closed form needs and to find the smallest $\sigma$ that meets a target.  The bound of \cref{thm:finite-rule} uses only the second moment of the likelihood ratio. The exact total variation depends on its whole distribution.  For random allocation it can still be written with one-dimensional integrals, because the likelihood ratio is a product of independent factors, one for each epoch, and each factor is an average of independent terms, one for each step.  These integrals are evaluated by computer.

Let $\varphi_Y(u)=\E e^{\mathrm iuY}$ be the characteristic function of the lognormal variable $Y=e^{tX-t^2/2}$, $X\sim N(0,1)$, and write $\Re$ for the real part of a complex number.  In one epoch, \eqref{eq:one-epoch-lr} gives $L_M-1=Z/M$ with $Z=\sum_{j=1}^M(Y_j-1)$.  The $M$ terms $Y_j-1$ are independent, and each has characteristic function $\E e^{\mathrm iv(Y_j-1)}=e^{-\mathrm iv}\varphi_Y(v)$.  Over $E$ epochs there are $E$ independent copies $L_{M,1},\dots,L_{M,E}$ of $L_M$, one for each epoch, and the likelihood ratio is their product $\Lambda=L_{M,1}\cdots L_{M,E}$.  For complex $q$ with $0\le\Re q\le1$ we call $F_M(q)=\E_{P_M}L_M^q$ the one-epoch factor.  The characteristic function of a sum of independent terms is the product of theirs, and the moment of a product of independent factors is the product of their moments, so
\[
 \varphi_Z(v)=\bigl(e^{-\mathrm iv}\varphi_Y(v)\bigr)^M,\qquad
 \E\Lambda^q=F_M(q)^E .
\]

\begin{theorem}[Exact TV]
\label{thm:exact-tv}
For every $E\ge1$,
\begin{equation}
 d_{\TV}(P_M^{\otimes E},Q_M^{\otimes E})=\frac1\pi\int_0^\infty
 \frac{1-\Re\bigl[F_M(\frac12+\mathrm iu)^E\bigr]}{u^2+\frac14}\dd u ,
 \label{eq:exact-tv-E}
\end{equation}
and for $0<\Re q<1$ the one-epoch factor is
\begin{equation}
 F_M(q)=1+\frac{q\,e^{\mathrm i\pi q/2}}{\Gamma(1-q)}\int_0^\infty
 s^{-q-1}e^{\mathrm is}\bigl(1-\varphi_Z(s/M)\bigr)\dd s .
 \label{eq:mellin-cf}
\end{equation}
For $E=1$ the total variation is also
\begin{equation}
 d_{\TV}(P_M,Q_M)=\frac1{\pi M}\int_0^\infty
 \frac{1-\Re\varphi_Z(u)}{u^2}\dd u .
 \label{eq:exact-tv-cf}
\end{equation}
All three integrals converge absolutely, and the integrand of \eqref{eq:exact-tv-E} is nonnegative.
\end{theorem}

Formula \eqref{eq:exact-tv-E} needs only the one-epoch factor $F_M$, raised to the power $E$, one for each epoch.  The accountant of \citet{ZhuDongWang2022} composes steps in the same way, with powers of the characteristic function of the privacy loss.  Formula \eqref{eq:mellin-cf} computes $F_M$ from the characteristic function of $Y$ raised to the power $M$, one factor for each step.  Appendix~\ref{app:exact-tv} gives the proofs.

\textbf{Comparison with the closed form.} We evaluate the total variation numerically, in a few seconds per value (Appendix~\ref{app:exact-numerics} gives the method and its checks).  At one epoch, over ten settings with $\sigma\in\{1.5,2\}$ and $M$ from $4{,}096$ to $1{,}048{,}576$, the ratio of the closed-form bound to the exact total variation lies between $1.253314$ and $1.253413$ and decreases as $M$ grows, toward its limit $\sqrt{\pi/2}=1.253314\ldots$ as $E/M\to0$.  With several epochs the ratio grows with the target.  On the $27$ configurations of \cref{tab:exact27} in Appendix~\ref{app:numerics}, it is about $1.27$ at $\lambda=0.10$ and between $1.41$ and $1.53$ at $\lambda=0.35$.  At a fixed $c_E$, a small enough $\sigma$ makes the exact distance much smaller than the bound (Appendix~\ref{app:chi2-underdetermines}).

We call the smallest $\sigma$ whose exact total variation is at most a target $\lambda$ \textit{the exact} $\sigma$.  Finding it needs a search over $\sigma$, which repeats this evaluation many times.

\section{Experiments}
\label{sec:experiments}

We compare the closed form with the exact value and with other ways of setting $\sigma$ (\cref{sec:tightness}), train four models at each $\sigma$ to measure what the extra noise of the closed form costs (\cref{sec:calib-trial}), and run membership inference attacks on the trained models (\cref{sec:attacks}).

\subsection{Noise and cost of calibration}
\label{sec:tightness}

We compare ways of setting $\sigma$ for a target $\lambda$ at one epoch on nine targets (see Appendix~\ref{app:numerics} for the full comparison.)  A target $\lambda$ caps the attack accuracy of every membership inference attack at $\frac12+\frac\lambda2$, for example at $0.55$ for $\lambda=0.10$.  The closed form takes about a microsecond and needs $7.4\%$ to $16.7\%$ more noise than the exact $\sigma$, which is well defined (Appendix~\ref{app:exact-numerics}) and which a root search on \eqref{eq:exact-tv-cf} finds in $0.27$ to $42$ seconds.  Theorem~4.3 of \citet{FeldmanShenfeld2025}, which we call FS, bounds random allocation through Poisson subsampling at an auxiliary rate $q$.  At $q=1/M$ it needs $9.0\%$ to $34.3\%$ more noise than the exact $\sigma$.  A smaller $q$ brings its bound closer to the exact value, but the bound then needs the total variation of a Poisson pair computed more precisely.  The closed form of \citet{VanDijkErtan2026} needs $128\%$ to $267\%$ more noise where it applies (Appendix~\ref{app:numerics}).

\subsection{What the extra noise costs in training}
\label{sec:calib-trial}

\begin{table*}[t]
\caption{Noise and time of three ways of setting $\sigma$ for the target $\lambda=0.10$, and the test accuracy of training with each.  \emph{Closed form} is \cref{cor:inverse}, \emph{exact} is \cref{thm:exact-tv}, and \emph{accountant} is that of \citet{FeldmanShenfeld2026} with its default route.  Noise and time are ranges over the nine cells $M\in\{100,250,500\}$, $E\in\{10,20,40\}$, with \emph{find} the time to find $\sigma$, \emph{certify} the time to certify that a given $\sigma$ meets the target, and \emph{find and certify} the time for both, on one CPU core.  The closed form needs no run to certify, and \emph{exact} is certified as in Appendix~\ref{app:certificate}.  The accountant's search runs the accountant at each $\sigma$ it tries and certifies the $\sigma$ it returns, so it has no separate find time.  The search ran with no limit on threads, and the accountant's \emph{certify} is the runs at the returned $\sigma$, in one direction, until one certified it.  Under training, the first three rows give the mean test accuracy in percent (validation accuracy for SST-2), and the last two rows give differences in points.}
\label{tab:calib-trial}
\centering
\small
\setlength{\tabcolsep}{3pt}
\begin{tabular}{@{}lcccccccc@{}}
\toprule
& \multicolumn{4}{c}{setting $\sigma$} & \multicolumn{4}{c}{training} \\
\cmidrule(lr){2-5}\cmidrule(lr){6-9}
$\sigma$ & \begin{tabular}[b]{@{}c@{}}noise\\above exact\end{tabular} & \textit{find} & \textit{certify} & \begin{tabular}[b]{@{}c@{}}\textit{find and}\\\textit{certify}\end{tabular} & CIFAR-10 & \begin{tabular}[b]{@{}c@{}}linear\\CIFAR-100\end{tabular} & SST-2 & AG News \\
\midrule
closed form & $14.9$--$25.3$\% & $0.7$--$1.5$\,$\mu$s & none & $0.7$--$1.5$\,$\mu$s & $39.68\pm0.26$ & $53.58\pm0.11$ & $85.89\pm0.87$ & $89.61\pm0.14$ \\
exact & $0$ & $8.8$--$10.4$ s & $70$--$670$ s & $79$--$680$ s & $40.75\pm0.37$ & $55.96\pm0.15$ & $87.47\pm0.52$ & $90.05\pm0.17$ \\
accountant & $0.24$--$0.52$\% & -- & $6.5$m--$7.1$ h & $5.5$--$249$ h & $40.76\pm0.30$ & $55.91\pm0.15$ & $87.47\pm0.59$ & $90.04\pm0.17$ \\
\midrule
\multicolumn{5}{@{}l}{exact minus closed form} & $+1.08$ & $+2.38$ & $+1.58$ & $+0.43$ \\
\multicolumn{5}{@{}l}{exact minus accountant} & $0.00$ & $+0.05$ & $0.00$ & $+0.01$ \\
\bottomrule
\end{tabular}
\end{table*}

We set $\sigma$ for the target $\lambda=0.10$ in three ways and train with each.  A root search finds the exact $\sigma$, rounded up to a multiple of $10^{-6}$, in $8.8$ to $10.4$ seconds on one CPU core, and a finite computation in ball arithmetic certifies in another $70$ to $670$ seconds that this $\sigma$ meets the target (Appendix~\ref{app:certificate}).  The accountant of \citet{FeldmanShenfeld2026} certifies a $\sigma$ when its upper bound on the total variation is at most $\lambda$, and proves that a $\sigma$ misses the target when its lower bound exceeds $\lambda$.  With its default route the accountant certified a $\sigma$ $0.24\%$ to $0.52\%$ above the exact one, and one run in one direction at that $\sigma$, with no search, took $6.5$ minutes to $7.1$ hours on one CPU core.  The accountant's fast Fourier route searched in $4.9$ to $10.9$ minutes but certified $0.26\%$ to $4.07\%$ more noise than the exact $\sigma$, and Appendix~\ref{app:experiments} compares both routes in detail.  The closed form needs $14.9\%$ to $25.3\%$ more noise than the exact $\sigma$, and at its $\sigma$ the exact total variation is only $0.0787$ to $0.0789$, about $79\%$ of $\lambda$ (\cref{tab:exact27} in Appendix~\ref{app:numerics}).  At $\delta=10^{-5}$ the exact $\sigma$ gives $\eps$ between $0.939$ and $1.022$ (\cref{tab:calib-eps} in Appendix~\ref{app:experiments}).  The exact $\sigma$ of \cref{thm:exact-tv} is therefore the smallest of the three.  No accountant that knows only $M$, $E$ and $\sigma$ can certify a $\sigma$ below the exact $\sigma$ before rounding.  All three are certified to meet the target, the closed form by \cref{cor:inverse}, the accountant's $\sigma$ by its upper bound, and the exact $\sigma$ by ball arithmetic.

All training uses \cref{alg:dpsgd}.  ResNet-18 \citep{HeEtAl2016} with group normalization \citep{WuHe2018} is trained from scratch on CIFAR-10 \citep{Krizhevsky2009}, a linear classifier is trained on frozen ResNet-50 features of CIFAR-100, and RoBERTa-base \citep{LiuEtAl2019RoBERTa} is fine-tuned in full on SST-2 \citep{SocherEtAl2013} and on AG News \citep{ZhangZhaoLeCun2015}.  The two image models use the nine cells $M\in\{100,250,500\}$, $E\in\{10,20,40\}$, and the two RoBERTa-base models use $M\in\{100,250,500\}$ at $E=20$.  Each model uses three seeds and one recipe, chosen in advance on public data, and the three runs of a cell and seed differ only in $\sigma$.  \Cref{tab:calib-trial} gives the mean over the cells and seeds of each model, and \cref{tab:calib-cells} in Appendix~\ref{app:experiments} lists every cell.

The exact $\sigma$ raises the mean test accuracy over the closed form by $1.08$ points on CIFAR-10, by $2.38$ points for the linear classifier on CIFAR-100 and by $0.43$ points on AG News, and the validation accuracy on SST-2 by $1.58$ points.  The gain is positive in each of the $24$ cells of the four models, and it persists at the targets $\lambda\in\{0.01,0.05,0.20\}$ (\cref{tab:calib-cells,tab:calib-lambda} in Appendix~\ref{app:experiments}).  Training at the exact $\sigma$ and at the accountant's $\sigma$ gives mean test or validation accuracies within $0.06$ points of each other, as expected for noise levels at most $0.52\%$ apart.  \emph{So the exact $\sigma$ of \cref{thm:exact-tv} gives the training result of the accountant's $\sigma$, and it is found in seconds and certified in minutes, whereas each search with the accountant in our experiments took longer or certified at least $0.62\%$ more noise.}

\subsection{Membership inference attacks (MIA)}
\label{sec:attacks}

We run MIA on two models, with the membership test of \cref{sec:mechanism}.  By \eqref{eq:transfer} no attack has a higher attack accuracy than the Gaussian pair, whose attack accuracy is at most the bound.  \Cref{tab:attacks} collects the results.

\textbf{ResNet-18 on CIFAR-10.} Each trial trains two ResNet-18 models at $M=100$, $E=10$ and $\lambda=0.10$, so $\sigma=1.737$ and the bound is $0.55$.  The two runs differ only in one target record and share all other randomness.  The target is a CIFAR-10 test image of a cat relabeled as a deer, a mislabeled record of the kind used as a canary in privacy audits \citep{NasrEtAl2023Tight}.  The score is the margin between the model's logit for the target's assigned label and its largest other logit \citep{CarliniEtAl2022}.  We chose this score and its threshold on $8$ separate pilot trials, from seven candidates, and then applied them once to $150$ new trials.  The balanced accuracy of the attack, the average of its fraction of correct answers on the $150$ models trained with the target and on the $150$ trained without it, is $0.513$, well below the bound of $0.55$ and below $0.539$, the attack accuracy of the pair itself.  A one-sided $97.5\%$ confidence bound puts it at most $0.534$.  The attack stays near chance because the standard deviation of the score across trials is $62$ times the average change that adding the record causes.

A second attack gives the same picture.  Its target is a ClipBKD canary \citep{JagielskiEtAl2020}, its score is the logit-scaled confidence \citep{CarliniEtAl2022} averaged over six views, and its balanced accuracy is $0.510\pm0.006$, again below $0.539$ and $0.55$.  Appendix~\ref{app:experiments} describes both attacks in detail.

\begin{table*}[t]
\caption{Membership inference attacks on trained models, against the closed-form bound and the
Gaussian pair.  Each trial trains two models that differ only in one target record, and the attacker
sees one of them.  \emph{Bound} is $\frac12+\frac12U(c_E)$ of \cref{thm:finite-rule}, \emph{Gaussian pair} is the
attack accuracy of $(P_M^{\otimes E},Q_M^{\otimes E})$ from \cref{thm:exact-tv}, and \emph{optimal} is the
attack accuracy $\Phi(\mu/2)$ of the optimal attack on the released linear model, with
$\mu=\sqrt{E/M}/\sigma$ for the target $2C\hat x$ and $\mu=0$ for $x_\star=0$.  \emph{Measured} is the
balanced accuracy of our attack with one standard error, over $150$ trials for ResNet-18 and $300$
for each linear row.  Values of $\sigma$ are rounded to three decimals, and the runs and computed columns use the
unrounded values.}
\label{tab:attacks}
\centering
\begin{tabular}{@{}lrrrcccc@{}}
\toprule
model & $M$ & $E$ & $\sigma$ & bound & Gaussian pair & optimal & measured \\
\midrule
ResNet-18 & $100$ & $10$ & $1.737$ & $0.55$ & $0.5394$ & -- & $0.513\pm0.007$ \\
ResNet-18, ClipBKD & $100$ & $10$ & $1.737$ & $0.55$ & $0.5394$ & -- & $0.510\pm0.006$ \\
linear & $128$ & $125$ & $5.039$ & $0.55$ & $0.5394$ & $0.5391$ & $0.533\pm0.007$ \\
linear, exact $\sigma$ & $128$ & $125$ & $3.994$ & $0.564$ & $0.5500$ & $0.5492$ & $0.552\pm0.009$ \\
linear & $128$ & $125$ & $1.561$ & $0.70$ & $0.6374$ & $0.6242$ & $0.613\pm0.012$ \\
linear, $x_\star=0$ & $128$ & $125$ & $1.561$ & $0.70$ & $0.6374$ & $0.5000$ & $0.500\pm0.000$ \\
\bottomrule
\end{tabular}
\end{table*}

\textbf{A linear model with a known optimal attack.} Both attacks on ResNet-18 stay well below the Gaussian pair, so they leave open whether any attack on a trained model can come close to the pair.  We now show that an attack can come close, on a linear model whose optimal attack is known in closed form.  As in the linear-loss analysis of \citet[Theorem~1]{SteinkeEtAl2025}, which treats Poisson subsampling, the gradient of the target is the same at every model, and only the coordinate of the released model along it depends on the target.  Here the model is a vector $w\in\R^{16}$ with loss $-\langle w,x\rangle$, trained by plain SGD from $w=0$ with learning rate $B_0$.  Twelve records have all features zero, so their clipped gradients are zero.  The target has features $2C\hat x$ for a fixed unit vector $\hat x$, so its gradient is clipped to $-C\hat x$ at every model.  The released model is then the sum of the target's contributions and all the noise, and with $\mu=\sqrt{E/M}/\sigma$ the optimal test is correct with probability $\Phi(\mu/2)$ under each hypothesis (Appendix~\ref{app:experiments}).  We take $N_0=16$, $M=128$ steps and $E=125$ epochs, so most steps hold no genuine record and add only noise.

With $300$ trials per setting, the attack reaches $0.533$ against a bound of $0.55$ at $\lambda=0.10$, and $0.613$ against $0.70$ at $\lambda=0.40$, in agreement with the predicted $\Phi(\mu/2)$ of $0.539$ and $0.624$.  At the exact $\sigma=3.994$ for $\lambda=0.10$, where the Gaussian pair has attack accuracy $0.55$, it reaches $0.552\pm0.009$ against the predicted $0.549$.  In every linear row of \cref{tab:attacks} the optimal attack is below the Gaussian pair, which is below the bound, and the measured attack is within one standard error of the optimal one.  At the closed-form $\sigma$ for $\lambda=0.10$ the optimal test on this model is correct with probability $0.5391$, and the optimal test on the pair $(P_M^{\otimes E},Q_M^{\otimes E})$ itself, computed with \cref{thm:exact-tv}, with probability $0.5394$.  So the linear model keeps $99.0\%$ of the pair's advantage over random guessing, and $90.4\%$ at $\lambda=0.40$, where the two are $0.6242$ and $0.6374$.  Training can therefore come close to the pair.  The remaining distance to the bound, from $0.5394$ to $0.55$, is the gap between the closed form and the exact value.

\section{Related Work and Limitations}
\label{sec:related}
\textbf{Random allocation.} Random allocation gives each record a fixed number of participations \citep{ChuaBallsBins2025,FeldmanShenfeld2025,DongChenOzgur2025}.  \citet{ChoquetteChooEtAl2025} use the same sampling for one epoch in matrix mechanisms with correlated noise and compute $(\eps,\delta)$ guarantees by Monte Carlo, and \citet{SchuchardtKalinin2026} compute such guarantees without sampling.  We study DP-SGD with independent noise and a fresh allocation in every epoch.

\textbf{Attack-based measures and calibration.} Total variation is the optimal advantage $2\operatorname{Acc}^*-1$ \citep{YeomEtAl2018} of a membership test for one pair of neighboring datasets \citep{ChatzikokolakisEtAl2023}.  For DP-SGD with Poisson subsampling, \citet{MahloujifarEtAl2022} estimate a bound on this advantage by Monte Carlo, \citet{KulynychEtAl2024} search with an accountant for the $\sigma$ that meets a target advantage, \citet{CherubinEtAl2024} give closed-form record-inference bounds up to an error term, and \citet{VanDijkEtAl2024ProactiveDP} choose parameters before training from a closed-form $(\eps,\delta)$ guarantee.  Appendix~\ref{app:related} gives related work in full.

\textbf{Limitations.} Our bound and the exact $\sigma$ concern the total variation of the pair, which caps the accuracy of every membership test, and we do not obtain an expression for the full trade-off function or the privacy profile.  These may be of separate interest when protection against membership inference is needed at small false-positive rates.  Appendix~\ref{app:limitations} discusses the limitations, including a chi-square bound at small false-positive rates, groups of records and $(\eps,\delta)$ guarantees.

\section{Conclusion}
\label{sec:conclusion}

For DP-SGD with random allocation, one number, $c_E=(1+(e^{1/\sigma^2}-1)/M)^E-1$, bounds the accuracy of every membership test on the released model, and the bound solves for $\sigma$ in closed form.  At fixed $\sigma$ the total-variation bound exceeds the exact value by a factor that tends to $\sqrt{\pi/2}$ as $M$ grows with $E/M\to0$, and on single-epoch targets the closed form needs $7.4\%$ to $16.7\%$ more noise than the exact $\sigma$.  The exact $\sigma$ is the smallest $\sigma$ that any accountant that knows only $M$, $E$ and $\sigma$ could certify.  It takes seconds to find and minutes to certify.  The exact $\sigma$ gives the mean test or validation accuracy of the accountant's $\sigma$ to within $0.06$ points.  Over the closed form it raised the mean test or validation accuracy of every model by $0.43$ to $2.38$ points.  As $M\to\infty$ with $E/M\to r<\infty$ at fixed $\sigma$, the trade-off curve of the pair tends to Gaussian differential privacy, which is random guessing if and only if $r=0$.  A linear model shows that training can come close to the pair, while an attack on ResNet-18 stays well below the bound.

\clearpage
\subsection*{AI use statement}
In this work, we used generative AI tools (large language models) to formulate mathematical claims, to provide critical ingredients for proofs, to assist in writing and checking proofs, to give feedback on the design of experiments, to implement methods by writing and reviewing the code for experiments, and to help interpret results.  We have not used generative AI tools to develop theoretical models or conceptual frameworks or to propose or refine hypotheses, and generating synthetic datasets, translation, cleaning or reformatting datasets and qualitative or thematic data analysis are not applicable to this work.  Additionally, we used generative AI tools to draft and edit text, to write code for analysis, figures and tables, and to search and read the literature.  We have reviewed all AI-assisted work.  The language models wrote the proofs in an iterative process that the authors led.  The authors chose which statements to prove, steered each argument toward an approach, pointed out gaps and errors, and reviewed every result.  All AI-assisted proofs were checked line by line by the authors.  We take responsibility for the final content of this work, including text, claims, or artifacts produced with the aid of generative AI.

\subsubsection*{Acknowledgements}
The contribution of Marten van Dijk and Murat Bilgehan Ertan to this publication is part of the
project CiCS of the research program Gravitation which is (partly) financed by the Dutch Research
Council (NWO) under the grant 024.006.037.
We acknowledge the use of the DAS-6 HPC cluster at VU Amsterdam \citep{bal2016medium}.

\bibliographystyle{plainnat}
\bibliography{references}

\clearpage
\appendix
\thispagestyle{empty}
\onecolumn
\aistatstitle{Closed-Form Noise Calibration Against Membership Inference for Random-Allocation DP-SGD: \\
Supplementary Materials}

\section{Notation and Standard Facts}
\label{app:background}

\textbf{Distributions and likelihood ratios.} $P$ and $Q$ are probability distributions on a common space.  We write $Q\ll P$ ($Q$ is absolutely continuous with respect to $P$) when every event of $P$-probability zero also has $Q$-probability zero.  Then $Q$ has a density $L=\dd Q/\dd P$ with respect to $P$, the likelihood ratio, by the Radon--Nikodym theorem (\cref{fact:radon-nikodym}).  Also $\E_PL=1$, because $Q$ has total mass one.  For the Gaussian pair $L$ is strictly positive, because by \eqref{eq:one-epoch-lr} it is an average of exponentials.

\textbf{Total variation and tests.} The total variation distance is $d_{\TV}(P,Q)=\sup_A|P(A)-Q(A)|$, over all events $A$.  Let $p$ and $q$ be densities of $P$ and $Q$ with respect to a measure $\nu$ that dominates both, for example $\nu=P+Q$, and write $x_+=\max\{x,0\}$.  For every event $A$, $Q(A)-P(A)=\int_A(q-p)\dd\nu\le\int(q-p)_+\dd\nu$, with equality for $A=\{q>p\}$.  Since $\int(q-p)\dd\nu=1-1=0$, the positive and negative parts of $q-p$ have the same integral, so
\[
 P(A)-Q(A)=\int_A(p-q)\dd\nu\le\int(p-q)_+\dd\nu=\int(q-p)_+\dd\nu .
\]
With $|x|=x_++(-x)_+$ this gives
\[
 d_{\TV}(P,Q)=\int(q-p)_+\dd\nu=\frac12\int|q-p|\dd\nu .
\]
When $Q\ll P$ this reads $d_{\TV}(P,Q)=\E_P(L-1)_+=\frac12\E_P|L-1|=\E_P(1-L)_+$, where the last form uses $\E_P(L-1)=0$.

A test is a measurable function $\phi$ with values in $[0,1]$, the probability of deciding for $Q$.  Its type-I error is $\alpha_\phi=\E_P\phi$ and its type-II error is $\beta_\phi=1-\E_Q\phi$.  For every test,
\begin{equation}
 1-\alpha_\phi-\beta_\phi=\int\phi\,(q-p)\dd\nu\le\int(q-p)_+\dd\nu=d_{\TV}(P,Q),
 \label{eq:test-gap}
\end{equation}
with equality for the likelihood-ratio test $\phi=\one\{q>p\}$, which is $\one\{L>1\}$ when $Q\ll P$.  When the two hypotheses are equally likely, a test is correct with probability $\frac12(1-\alpha_\phi)+\frac12(1-\beta_\phi)=\frac12+\frac12(1-\alpha_\phi-\beta_\phi)$. So the optimal test is correct with probability $\frac12+\frac12d_{\TV}(P,Q)$.

\textbf{Trade-off functions and Gaussian differential privacy.} The trade-off function of $P$ against $Q$ is
\[
 T(P,Q)(\alpha)=\inf\{\beta_\phi:\alpha_\phi\le\alpha\},\qquad 0\le\alpha\le1,
\]
the smallest type-II error of a test with type-I error at most $\alpha$ \citep[Definition~2.1]{DongRothSuGDP}.  Random guessing, $\phi\equiv\alpha$, gives $1-\alpha$. When $Q\ll P$, tests that threshold $L$ are optimal, by the Neyman--Pearson lemma \citep{NeymanPearson1933}.  To see this, let $\phi^*$ equal $1$ on $\{L>k\}$ and $0$ on $\{L<k\}$ for some $k\ge0$.  Every test $\phi$ with $\alpha_\phi\le\alpha_{\phi^*}$ then satisfies
\begin{equation}
 \beta_\phi-\beta_{\phi^*}=\E_P[(\phi^*-\phi)L]
 =\E_P[(\phi^*-\phi)(L-k)]+k(\alpha_{\phi^*}-\alpha_\phi)\ge0,
 \label{eq:np}
\end{equation}
because $(\phi^*-\phi)(L-k)\ge0$ everywhere.

For each test with $\alpha_\phi\le\alpha$, \eqref{eq:test-gap} gives $1-\alpha-\beta_\phi\le1-\alpha_\phi-\beta_\phi\le d_{\TV}(P,Q)$, so $1-\alpha-T(P,Q)(\alpha)\le d_{\TV}(P,Q)$.  At the type-I error $\alpha$ of the test $\one\{q>p\}$, equality in \eqref{eq:test-gap} gives the reverse inequality.  Hence
\begin{equation}
 d_{\TV}(P,Q)=\sup_{0\le\alpha\le1}\bigl\{1-\alpha-T(P,Q)(\alpha)\bigr\}.
 \label{eq:tv-sup}
\end{equation}
Since $T(P,Q)\ge0$, a bound $d_{\TV}(P,Q)\le\lambda$ is the same as $T(P,Q)\ge f_{0,\lambda}$, where $f_{0,\lambda}(\alpha)=\max\{0,1-\lambda-\alpha\}$ is the trade-off function of $(0,\lambda)$-differential privacy (\citealp[Proposition~2.5]{DongRothSuGDP} and \citealp[Theorem~2.1]{KairouzOhViswanath2015}). \citet[Definitions~5.1 and~5.2]{ErtanVanDijk2026} define the separation of a trade-off function $f$ as $\Sep(f)=\max_\alpha\{(1-\alpha)-f(\alpha)\}/\sqrt2$, the largest Euclidean distance from the curve of $f$ to the random-guessing line.  By \eqref{eq:tv-sup}, $\Sep(T(P,Q))=d_{\TV}(P,Q)/\sqrt2$.

Let $\Phi$ be the standard normal distribution function.  For $\mu\ge0$, the likelihood ratio $e^{\mu x-\mu^2/2}$ of $N(\mu,1)$ against $N(0,1)$ is nondecreasing in $x$.  So by \eqref{eq:np} the test that rejects when $x>\Phi^{-1}(1-\alpha)$ is optimal, and the trade-off function of $N(0,1)$ against $N(\mu,1)$ is
\[
 G_\mu(\alpha)=\Phi(\Phi^{-1}(1-\alpha)-\mu).
\]
For $\mu>0$ the event $\{q>p\}$ is $A=\{x>\mu/2\}$, so the total variation of this pair is $Q(A)-P(A)=\Phi(\mu/2)-\Phi(-\mu/2)=2\Phi(\mu/2)-1$.

\textbf{Chi-square and R\'enyi divergence.} $\chi^2(Q\|P)=\E_P(L-1)^2=\E_PL^2-1$.  The R\'enyi divergence of order $\alpha>1$ is $D_\alpha(Q\|P)=\frac1{\alpha-1}\ln\E_PL^\alpha$, so $D_2=\ln(1+\chi^2)$.  R\'enyi divergence adds over product distributions \citep[Theorem~28]{vanErvenHarremoes2014}, which is why $1+\chi^2$ multiplies.

\textbf{Privacy profile.} For $\eps\ge0$ the hockey-stick divergence is $H_{e^\eps}(P\|Q)=\int(p-e^\eps q)_+\dd\nu$ for densities $p,q$ with respect to any $\nu$ dominating both.  The privacy profile of a pair is $\delta(\eps)=\max\{H_{e^\eps}(P\|Q),H_{e^\eps}(Q\|P)\}$.  At $\eps=0$ both orders equal $d_{\TV}(P,Q)$.  As in \eqref{eq:test-gap}, each hockey-stick divergence is a supremum over tests, $H_{e^\eps}(P\|Q)=\sup_\phi\{\E_P\phi-e^\eps\E_Q\phi\}$, attained at $\phi=\one\{p>e^\eps q\}$.  A randomized algorithm is $(\eps,\delta)$-differentially private if, for every two datasets that differ by adding or removing one record, the privacy profile of its two output distributions satisfies $\delta(\eps)\le\delta$, that is, $P(A)\le e^\eps Q(A)+\delta$ and $Q(A)\le e^\eps P(A)+\delta$ for every event $A$.  A bound $\lambda$ on the total variation alone gives no such guarantee with $\delta<\lambda$.  On three points the distributions $P=(1-\lambda,\lambda,0)$ and $Q=(1-\lambda,0,\lambda)$ have $d_{\TV}(P,Q)=\lambda$, and $H_{e^\eps}(P\|Q)=H_{e^\eps}(Q\|P)=\lambda$ for every $\eps\ge0$, since each puts mass $\lambda$ on a point where the other puts none.

\textbf{Markov kernels and data processing.} A Markov kernel $K$ from $\mathcal X$ to $\mathcal Y$ assigns to each $x\in\mathcal X$ a probability distribution $K(x,\cdot)$ on $\mathcal Y$, measurably in $x$.  It is a randomized map.  Applying it to $P$ gives $PK(B)=\int K(x,B)\dd P(x)$.  Pushing both distributions through the same randomized map does not make them easier to tell apart, $T(PK,QK)\ge T(P,Q)$ \citep[Lemma~2.9]{DongRothSuGDP}, because a test $\phi$ on $\mathcal Y$ gives the test $x\mapsto\int\phi(y)K(x,\dd y)$ on $\mathcal X$, which has the same type-I and type-II errors.  The same argument proves the inequality for total variation and for both hockey-stick divergences, since each of them is a supremum over tests.

\textbf{Other notation.} Throughout, $N_0$, $M$ and $E$ are positive integers and $\sigma,C>0$. $V_M\Rightarrow V$ denotes convergence in distribution, and for distribution functions $F_n\Rightarrow F$ means that $F_n(y)\to F(y)$ at every $y$ where $F$ is continuous.  $\Re z$ and $\Im z$ are the real and imaginary parts of a complex number $z$, and $\bar z$ its conjugate. $O_\sigma(M^{-2})$ is a quantity bounded in absolute value by $K_\sigma M^{-2}$, with $K_\sigma$ depending only on $\sigma$.  $\varphi_V(u)=\E e^{\mathrm iuV}$ is the characteristic function of $V$.

\textbf{Standard results.} The proofs use the following standard results of measure theory, probability and complex analysis.  Each is stated in the form we use, with its hypotheses, and the proofs refer to it by its number.  All functions in these results are assumed measurable.

\emph{Measure and integration.}

\begin{fact}[{Radon--Nikodym theorem, \citealp[Theorem~A.4.8]{Durrett2019}}]
\label{fact:radon-nikodym}
If $\mu$ and $\nu$ are $\sigma$-finite measures and $\nu\ll\mu$, then there is a measurable $g\ge0$ with $\nu(A)=\int_Ag\dd\mu$ for every measurable set $A$, and $g$ is unique up to a set of $\mu$-measure zero.
\end{fact}

\begin{fact}[{Uniqueness of measures, \citealp[Theorem~A.1.5]{Durrett2019}}]
\label{fact:uniqueness}
Let $\mathcal P$ be a $\pi$-system, a class of sets closed under finite intersections.  If two measures agree on $\mathcal P$, and there are sets $A_n\in\mathcal P$ with $A_n\uparrow\Omega$ and finite measure, then the two measures agree on the $\sigma$-field generated by $\mathcal P$.  In particular, two measures on a product space that agree on every rectangle $A_1\times A_2$ and give the whole space finite measure are equal, since the rectangles are closed under intersection and generate the product $\sigma$-field.
\end{fact}

\begin{fact}[{Fubini--Tonelli theorem, \citealp[Theorem~1.7.2]{Durrett2019}}]
\label{fact:fubini}
Let $\mu_1$ and $\mu_2$ be $\sigma$-finite measures, and let $f$ be measurable on the product space.  If $f\ge0$, or if $\int|f|\dd(\mu_1\otimes\mu_2)<\infty$, then
\[
 \int\Bigl(\int f(x,y)\,\mu_2(\dd y)\Bigr)\mu_1(\dd x)=\int f\dd(\mu_1\otimes\mu_2)=\int\Bigl(\int f(x,y)\,\mu_1(\dd x)\Bigr)\mu_2(\dd y).
\]
The case $f\ge0$ is Tonelli's theorem and the integrable case is Fubini's theorem.  A complex $f$ with $\int|f|\dd(\mu_1\otimes\mu_2)<\infty$ satisfies the same identity, applied to its real and imaginary parts.
\end{fact}

\begin{fact}[{Dominated convergence, \citealp[Theorems~1.5.8 and~1.6.7]{Durrett2019}}]
\label{fact:dominated}
If $f_n\to f$ almost everywhere, $|f_n|\le g$ for all $n$ and $\int g\dd\mu<\infty$, then $\int f_n\dd\mu\to\int f\dd\mu$.  For expectations, if $X_n\to X$ almost surely, $|X_n|\le Y$ for all $n$ and $\E Y<\infty$, then $\E X_n\to\E X$.  Complex $f_n$ satisfy the same, applied to their real and imaginary parts.
\end{fact}

\emph{Inequalities.}

\begin{fact}[{H\"older's inequality, \citealp[Theorem~1.6.3]{Durrett2019}}]
\label{fact:holder}
If $p,q\in(1,\infty)$ and $1/p+1/q=1$, then $\E|XY|\le(\E|X|^p)^{1/p}(\E|Y|^q)^{1/q}$.  The case $p=q=2$ is the Cauchy--Schwarz inequality $\E|XY|\le(\E X^2\,\E Y^2)^{1/2}$.
\end{fact}

\begin{fact}[{Markov and Chebyshev inequalities, \citealp[Theorem~1.6.4]{Durrett2019}}]
\label{fact:markov}
If $g\ge0$ and $A$ is a Borel set, then $\inf_{y\in A}g(y)\,\Pr(X\in A)\le\E g(X)$.  Taking $g(x)=|x|^p$ gives Markov's inequality $\Pr(|X|\ge\xi)\le\E|X|^p/\xi^p$ for $p>0$ and $\xi>0$, and for $\E|X|<\infty$, Markov's inequality for $X-\E X$ with $p=2$ is Chebyshev's inequality $\Pr(|X-\E X|\ge\xi)\le\Var(X)/\xi^2$.
\end{fact}

\emph{Independence.}

\begin{fact}[{Independence from distribution functions, \citealp[Theorem~2.1.8]{Durrett2019}}]
\label{fact:indep-df}
Random variables $X_1,\dots,X_n$ are independent if $\Pr(X_1\le x_1,\dots,X_n\le x_n)=\prod_{i=1}^n\Pr(X_i\le x_i)$ for all $x_1,\dots,x_n\in(-\infty,\infty]$.
\end{fact}

\begin{fact}[{Functions of independent groups, \citealp[Theorem~2.1.10]{Durrett2019}}]
\label{fact:indep-functions}
If the random variables $X_{i,j}$ with $1\le i\le n$ and $1\le j\le m(i)$ are independent and each $f_i:\R^{m(i)}\to\R$ is measurable, then the variables $f_i(X_{i,1},\dots,X_{i,m(i)})$ with $1\le i\le n$ are independent.
\end{fact}

\begin{fact}[{Product rule, \citealp[Theorem~2.1.13]{Durrett2019}}]
\label{fact:product}
If $X_1,\dots,X_n$ are independent and either $X_i\ge0$ for all $i$, or $\E|X_i|<\infty$ for all $i$, then the expectation of $\prod_{i=1}^nX_i$ exists and $\E\prod_{i=1}^nX_i=\prod_{i=1}^n\E X_i$.
\end{fact}

\emph{Convergence in distribution.}

\begin{fact}[{Quantile transform and coupling, \citealp[proofs of Theorems~1.2.2 and~3.2.8]{Durrett2019}}]
\label{fact:quantile}
For a distribution function $F$ let $F^{-1}(v)=\sup\{y:F(y)<v\}$ for $0<v<1$.  If $V$ is uniform on $(0,1)$, then $F^{-1}(V)$ has distribution function $F$.  If distribution functions satisfy $F_n\Rightarrow F_\infty$, then $F_n^{-1}(v)\to F_\infty^{-1}(v)$ for all but countably many $v\in(0,1)$, so $F_n^{-1}(V)\to F_\infty^{-1}(V)$ almost surely.
\end{fact}

\begin{fact}[{Continuous mapping theorem, \citealp[Theorem~3.2.10]{Durrett2019}}]
\label{fact:continuous-mapping}
Let $g$ be measurable with set of discontinuity points $D_g$.  If $X_n\Rightarrow X_\infty$ and $\Pr(X_\infty\in D_g)=0$, then $g(X_n)\Rightarrow g(X_\infty)$.  If in addition $g$ is bounded, then $\E g(X_n)\to\E g(X_\infty)$.
\end{fact}

\begin{fact}[{Converging together lemma, \citealp[Exercises~3.2.12 and~3.2.13]{Durrett2019}}]
\label{fact:converging-together}
Convergence in probability implies convergence in distribution, and $X_n\Rightarrow c$ for a constant $c$ implies $X_n\to c$ in probability.  If $X_n\Rightarrow X$ and $Y_n\Rightarrow c$ for a constant $c$, then $X_n+Y_n\Rightarrow X+c$.
\end{fact}

\begin{fact}[{Integration to the limit, \citealp[Exercise~3.2.5]{Durrett2019}}]
\label{fact:integration-limit}
Let $g$ and $h$ be continuous with $g>0$ and $|h(x)|/g(x)\to0$ as $|x|\to\infty$.  If $F_n\Rightarrow F$ and $\sup_n\int g\dd F_n<\infty$, then $\int h\dd F_n\to\int h\dd F$.
\end{fact}

\begin{fact}[{Central limit theorem, \citealp[Theorem~3.4.1]{Durrett2019}}]
\label{fact:clt}
If $X_1,X_2,\dots$ are independent and identically distributed with $\E X_1=0$ and $\E X_1^2=b\in(0,\infty)$, and $S_n=X_1+\dots+X_n$, then $S_n/\sqrt{nb}\Rightarrow N(0,1)$.
\end{fact}

\begin{fact}[{Lindeberg--Feller central limit theorem, \citealp[Proposition~2.27]{VanDerVaart1998}}]
\label{fact:lindeberg}
For each $j$ let $Y_{j,1},\dots,Y_{j,k_j}$ be independent random variables with finite variances.  If $\sum_{i=1}^{k_j}\E\bigl[Y_{j,i}^2\one\{|Y_{j,i}|>\xi\}\bigr]\to0$ for every $\xi>0$ and $\sum_{i=1}^{k_j}\Var(Y_{j,i})\to s$ as $j\to\infty$, then $\sum_{i=1}^{k_j}(Y_{j,i}-\E Y_{j,i})\Rightarrow N(0,s)$.  The numbers $k_j$ may depend on $j$ in any way.
\end{fact}

\emph{Characteristic functions and Gaussian vectors.}

\begin{fact}[{Characteristic functions, \citealp[Theorem~3.3.1(e) and Examples~3.3.5 and~3.3.16]{Durrett2019}}]
\label{fact:chf}
For real $a$ and $b$, $\varphi_{aX+b}(u)=e^{\mathrm iub}\varphi_X(au)$.  The standard Gaussian distribution has characteristic function $e^{-u^2/2}$, so $N(0,v)$ has $e^{-vu^2/2}$.  The standard Cauchy distribution, with density $1/(\pi(1+x^2))$, has characteristic function $e^{-|u|}$.
\end{fact}

\begin{fact}[{Inversion formula, \citealp[Theorem~3.3.14]{Durrett2019}}]
\label{fact:inversion}
If $\int|\varphi_X(u)|\dd u<\infty$, then $X$ has the bounded continuous density $f(y)=\frac1{2\pi}\int e^{-\mathrm iuy}\varphi_X(u)\dd u$.
\end{fact}

\begin{fact}[{Gaussian vectors, \citealp[Section~3.10]{Durrett2019}}]
\label{fact:gaussian-vector}
If $X$ has independent standard Gaussian coordinates and $H$ is a real $d\times d$ matrix, then $HX$ has the multivariate normal distribution with mean $0$, covariance matrix $\Sigma=HH^\top$ and characteristic function $\E e^{\mathrm i\theta^\top HX}=e^{-\theta^\top\Sigma\theta/2}$.  If $\Sigma$ is nonsingular, $HX$ has the density $(2\pi)^{-d/2}(\det\Sigma)^{-1/2}e^{-w^\top\Sigma^{-1}w/2}$.
\end{fact}

\emph{Complex analysis and the Gamma function.}

\begin{fact}[{Principal logarithm and powers, \citealp[Eqs.~4.2.3 and~4.2.28 and Section~4.2(i)]{DLMF}}]
\label{fact:log}
The principal logarithm $\ln w=\ln|w|+\mathrm i\arg w$, with $-\pi<\arg w<\pi$, is a single-valued analytic function of $w$ off the half-line $(-\infty,0]$.  Powers are principal values, $w^p=e^{p\ln w}$.
\end{fact}

\begin{fact}[{Gamma function, \citealp[Eqs.~5.2.1 and~5.9.1 and Section~5.2(i)]{DLMF}}]
\label{fact:gamma}
For $\Re z>0$ the Gamma function is the convergent integral $\Gamma(z)=\int_0^\infty e^{-t}t^{z-1}\dd t$, and the Gamma function has no zeros.  For $\Re\nu>0$ and $\Re w>0$, $\int_0^\infty e^{-wt}t^{\nu-1}\dd t=\Gamma(\nu)\,w^{-\nu}$ with the principal power.
\end{fact}

\begin{fact}[{Cauchy's theorem, \citealp[Eq.~1.9.29]{DLMF}}]
\label{fact:cauchy}
If $f$ is continuous within and on a simple closed contour $\mathcal C$ and analytic within $\mathcal C$, then $\int_{\mathcal C}f(\zeta)\dd\zeta=0$.
\end{fact}

\section{Proofs}
\label{app:proofs}

\subsection{One epoch and composition}
\label{app:one-epoch-proof}

Let $\varphi_M$ be the density of $N(0,I_M)$ and $\varphi=\varphi_1$.  For the $j$th component of \eqref{eq:one-record-pair} every factor of $\varphi_M$ except coordinate $j$ cancels, so
\[
 \frac{\varphi_M(x-te_j)}{\varphi_M(x)}=\frac{\varphi(x_j-t)}{\varphi(x_j)}=\exp(tx_j-t^2/2).
\]
Averaging over $j$ gives \eqref{eq:one-epoch-lr}.

\begin{proof}[Proof of \cref{prop:one-epoch}]
For $X\sim N(0,1)$ and real $s$, completing the square $sx-x^2/2=s^2/2-(x-s)^2/2$ gives
\[
 \E e^{sX}=\int e^{sx}\varphi(x)\dd x=e^{s^2/2}\int\varphi(x-s)\dd x=e^{s^2/2}.
\]
With $s=t$ and $s=2t$,
\[
 \E Y_j=e^{-t^2/2}\,\E e^{tX_j}=1,\qquad \E Y_j^2=e^{-t^2}\,\E e^{2tX_j}=e^{t^2}.
\]
So each $Y_j-1$ has mean $0$ and second moment $\E Y_j^2-1=e^{t^2}-1$.  By \eqref{eq:one-epoch-lr}, $L_M-1=M^{-1}\sum_j(Y_j-1)$ is an average of independent terms, and by the product rule (\cref{fact:product}) the cross terms $\E[(Y_j-1)(Y_k-1)]$ with $j\ne k$ vanish. Since $t^2=1/\sigma^2$,
\[
 \chi^2(Q_M\|P_M)=\E(L_M-1)^2=\frac1{M^2}\sum_{j=1}^M\E(Y_j-1)^2=\frac{e^{1/\sigma^2}-1}{M} .
\]
\end{proof}

For product distributions the likelihood ratio of the product is the product of the likelihood ratios.  Let $Q_i\ll P_i$ with $L_i=\dd Q_i/\dd P_i$ for $i=1,2$.  For events $A_1,A_2$, Tonelli's theorem (\cref{fact:fubini}) gives
\[
 \int_{A_1\times A_2}L_1(x_1)L_2(x_2)\,\dd(P_1\otimes P_2)(x_1,x_2)
 =\int_{A_1}L_1\dd P_1\int_{A_2}L_2\dd P_2=Q_1(A_1)\,Q_2(A_2).
\]
So the measure with density $L_1(x_1)L_2(x_2)$ with respect to $P_1\otimes P_2$ agrees with $Q_1\otimes Q_2$ on all rectangles $A_1\times A_2$.  With $A_1$ and $A_2$ the whole spaces, both measures give the whole product space measure $1$, so the two measures are equal (\cref{fact:uniqueness}).  Hence the likelihood ratio of the product is $L_1(x_1)L_2(x_2)$.  The two coordinates are independent under $P_1\otimes P_2$, so by the product rule (\cref{fact:product})
\[
 1+\chi^2(Q_1\otimes Q_2\|P_1\otimes P_2)=\E\bigl[L_1^2L_2^2\bigr]=\E L_1^2\,\E L_2^2 ,
\]
and induction over the factors gives \eqref{eq:tensorize}.

\subsection{Chi-square to total variation, and the closed-form bound}
\label{app:sharp-envelope-proof}

We prove \eqref{eq:sharp-envelope}.  \Citet[Eq.~(7)]{SasonVerdu2016} state it in the inverse form, as a lower bound on the chi-square divergence in terms of $|P-Q|=2d_{\TV}(P,Q)$.  As in the definition of chi-square in Appendix~\ref{app:background}, let $Q\ll P$ with likelihood ratio $L$, and let $c=\chi^2(Q\|P)$ and $\tau=d_{\TV}(P,Q)$.  By Appendix~\ref{app:background} and the Cauchy--Schwarz inequality (\cref{fact:holder}),
\[
 \tau=\tfrac12\E_P|L-1|\le\tfrac12\bigl(\E_P(L-1)^2\bigr)^{1/2}=\tfrac12\sqrt c
\]
for every $c$.  This proves the case $c\le1$.  For $c\ge1$, let $A=\{L>1\}$ and $x=Q(A)=\E_P[L\one_A]$, so that $0\le x\le1$.  Then $\tau=\E_P[(L-1)\one_A]=x-P(A)$, and the Cauchy--Schwarz inequality (\cref{fact:holder}) gives $x^2=(\E_P[L\one_A])^2\le\E_PL^2\,P(A)=(1+c)P(A)$.  Hence
\[
 \tau\le x-\frac{x^2}{1+c}\le1-\frac1{1+c}=\frac c{1+c} .
\]
The second inequality holds because $x\mapsto x-x^2/(1+c)$ has derivative $1-2x/(1+c)\ge0$ for $0\le x\le1\le(1+c)/2$, so on $[0,1]$ it is largest at $x=1$.  This proves \eqref{eq:sharp-envelope}.

The constants are attained by pairs on two points.  When $c\le1$, take $P=(\frac12,\frac12)$ and $Q=(\frac12+u,\frac12-u)$ with $u=\sqrt c/2\le\frac12$.  Both entries of $P$ are positive, so $Q\ll P$, and
\[
 L=(1+\sqrt c,\,1-\sqrt c),\qquad
 \chi^2(Q\|P)=\tfrac12(\sqrt c)^2+\tfrac12(-\sqrt c)^2=c,\qquad
 d_{\TV}(P,Q)=\E_P(1-L)_+=\tfrac12\cdot0+\tfrac12\sqrt c=U(c).
\]
When $c\ge1$, take $P=(\frac{c}{1+c},\frac1{1+c})$ and $Q=(0,1)$.  Again $Q\ll P$, and
\[
 \begin{gathered}
 L=(0,\,1+c),\qquad
 \chi^2(Q\|P)=\E_PL^2-1=\frac{(1+c)^2}{1+c}-1=c,\\
 d_{\TV}(P,Q)=\E_P(1-L)_+=\frac c{1+c}\cdot1+\frac1{1+c}\cdot0=U(c).
 \end{gathered}
\]
Finally, $U(c)\le\sqrt c/2$ for every $c\ge0$.  It is an equality for $c\le1$, and for $c\ge1$ the inequality $1+c\ge2\sqrt c$ gives $c/(1+c)\le\sqrt c/2$.

\citet[Section~3.2]{MahloujifarEtAl2022} bound the total variation by Pinsker's inequality and the R\'enyi divergence $D_\alpha$ of order $\alpha>1$, $d_{\TV}\le\sqrt{D_\alpha/2}$.  With the order-two R\'enyi divergence $E\ln(1+(e^{1/\sigma^2}-1)/M)$ of the Gaussian pair this bound is also a closed form, but at $\lambda=0.10$ it needs $25.1\%$ to $38.5\%$ more noise than \eqref{eq:inverse-sigma} in the nine training cells of \cref{sec:calib-trial}.  With the Kullback--Leibler divergence, which has no closed form for the pair, Pinsker's inequality needs $0.67\%$ to $0.94\%$ less noise than \eqref{eq:inverse-sigma} at $\lambda=0.10$ and $5.7\%$ to $8.7\%$ less at $\lambda=0.40$ in the same nine cells, with the divergence taken in either order.

\begin{proof}[Proof of \cref{thm:finite-rule}]
By \eqref{eq:tensorize} the composed divergence is $c_E$, and \eqref{eq:sharp-envelope} and \eqref{eq:sep-tv-acc} give \eqref{eq:finite-rule}.  For the bounds on $c_E$ stated after the theorem, put $y=(e^{1/\sigma^2}-1)/M\ge0$ and $x=Ey$.  All terms of the binomial expansion of $(1+y)^E$ are nonnegative, and its first two terms give $(1+y)^E\ge1+x$, which is Bernoulli's inequality.  The power series of $e^y$ gives $1+y\le e^y$, so $(1+y)^E\le e^x$.  Therefore
\[
 x\le\Bigl(1+\frac{e^{1/\sigma^2}-1}{M}\Bigr)^E-1=c_E\le e^{x}-1 .
\]
\end{proof}

\begin{proof}[Proof of \cref{cor:inverse}]
The two definitions of $c_\lambda$ solve $U(c_\lambda)=\lambda$ on the two branches of $U$. Setting $c_E=c_\lambda$ and solving $e^{1/\sigma^2}-1=M\{(1+c_\lambda)^{1/E}-1\}$ for $\sigma$ gives \eqref{eq:inverse-sigma}.  \Cref{thm:finite-rule} then gives $d_{\TV}(P_M^{\otimes E},Q_M^{\otimes E})\le\lambda$.  By \eqref{eq:transfer} the released model has $d_{\TV}(P_{\rm out},Q_{\rm out})\le\lambda$, and \eqref{eq:sep-tv-acc} gives $\operatorname{Acc}^*\le\frac12+\frac\lambda2$.
\end{proof}

\subsection{The Gaussian differential privacy limit}
\label{app:gdp-proof}

The proof expands the log-likelihood ratio of each epoch to second order.  Across $E_M$ epochs these terms form a triangular array whose sum is Gaussian in the limit, with variance $r(e^{1/\sigma^2}-1)$.  A fourth-moment bound, $\E(L_M-1)^4=O_\sigma(M^{-2})$, makes the higher-order terms vanish.  The trade-off curve is an integral of the quantile function of the likelihood ratio, and convergence of these quantile functions gives uniform convergence of the curve.  We say that a sequence of pairs converges to random guessing when its total variation tends to $0$.  By \eqref{eq:tv-sup} this means that its trade-off function converges to $1-\alpha$ uniformly in $\alpha$.

\begin{lemma}[Central limit theorem for the log-likelihood ratio]
\label{lem:product-gdp}
For each $M$, let $Q'_M\ll P'_M$ have a positive likelihood ratio $L'_M$, and set $U_M=L'_M-1$.  Let $\kappa\in[0,\infty)$ and suppose
\begin{equation}
 M\E_{P'_M}U_M^2\to\kappa,\qquad M\E_{P'_M}U_M^4\to0.
 \label{eq:product-gdp-moments}
\end{equation}
Let $E_M$ be positive integers with $E_M/M\to r<\infty$.  Then $T((P'_M)^{\otimes E_M},(Q'_M)^{\otimes E_M})$ converges uniformly to $G_{\sqrt{r\kappa}}$.
\end{lemma}

\begin{proof}
Put $n=E_M$ and $s=r\kappa$.  Write a sample of $(P'_M)^{\otimes n}$ as $(x_1,\dots,x_n)$ and put $U_{M,e}=L'_M(x_e)-1$.  These are $n$ independent copies of $U_M$, and by Appendix~\ref{app:one-epoch-proof} the likelihood ratio of the product pair is $\Lambda_M=\prod_{e=1}^n(1+U_{M,e})>0$.  Probabilities and expectations are under the null $(P'_M)^{\otimes n}$.  Since $n/M\to r$, \eqref{eq:product-gdp-moments} gives $n\E U_M^2\to s$ and $n\E U_M^4\to0$.

\emph{Step 1, a central limit theorem for the sum.}  The summands are centered, since $\E U_M=\E_{P'_M}L'_M-1=0$.  For every $\xi>0$, the fourth moment gives the Lindeberg condition
\[
 \sum_{e=1}^n\E\bigl[U_{M,e}^2\one\{|U_{M,e}|>\xi\}\bigr]\le\frac{n\E U_M^4}{\xi^2}\to0 ,
\]
since $u^2\one\{|u|>\xi\}\le u^4/\xi^2$.  The condition says that the summands larger than $\xi$ carry a vanishing part of the variance.  Our $n=E_M$ variables $U_{M,e}$ are independent and centered, and their variances sum to $n\E U_M^2\to s$, so for $s>0$ the Lindeberg--Feller central limit theorem (\cref{fact:lindeberg}) gives $\sum_eU_{M,e}\Rightarrow N(0,s)$.  For $s=0$ the limit $N(0,0)$ is the point mass at $0$, and we check this case directly.  The sum has mean $0$ and variance $n\E U_M^2\to0$, so it converges to $0$ in probability by Chebyshev's inequality (\cref{fact:markov}).  The fourth moment also gives, in probability,
\[
 \sum_eU_{M,e}^2\to s,\qquad \max_{e\le n}|U_{M,e}|\to0,\qquad
 \sum_e|U_{M,e}|^3\to0 .
\]
The first holds by Chebyshev's inequality (\cref{fact:markov}), because the left side has mean $n\E U_M^2\to s$ and variance $n\operatorname{Var}(U_M^2)\le n\E U_M^4\to0$.  The second follows from Markov's inequality (\cref{fact:markov}) and a union bound, since for every $\xi>0$ we have $\Pr(\max_e|U_{M,e}|>\xi)\le n\Pr(|U_M|>\xi)\le n\E U_M^4/\xi^4\to0$.  The third follows from Markov's inequality and the bound
\begin{equation}
 n\E|U_M|^3\le n\bigl(\E U_M^2\,\E U_M^4\bigr)^{1/2}
 =\frac nM\bigl(M\E U_M^2\cdot M\E U_M^4\bigr)^{1/2}\to r(\kappa\cdot0)^{1/2}=0 ,
 \label{eq:third-moment}
\end{equation}
which is the Cauchy--Schwarz inequality (\cref{fact:holder}) applied to $|U_M|\cdot U_M^2$.

\emph{Step 2, the log-likelihood ratio.}  For $|u|\le1/2$ the power series of $\ln(1+u)$ gives
\[
 \ln(1+u)=u-\frac{u^2}2+R(u),\qquad
 |R(u)|\le\sum_{k\ge3}\frac{|u|^k}{k}\le\frac{|u|^3}{3(1-|u|)}\le|u|^3 ,
\]
where the last step uses $3(1-|u|)\ge3/2$. Let $\mathcal R_M=\ln\Lambda_M-\sum_eU_{M,e}+\frac12\sum_eU_{M,e}^2$.  Since $\ln\Lambda_M=\sum_e\ln(1+U_{M,e})$, on the event $\{\max_e|U_{M,e}|\le1/2\}$ we have $|\mathcal R_M|\le\sum_e|U_{M,e}|^3$.  By Step~1 this event has probability tending to one and $\sum_e|U_{M,e}|^3\to0$ in probability, so $\mathcal R_M\to0$ in probability.  Hence $-\frac12\sum_eU_{M,e}^2+\mathcal R_M\to-s/2$ in probability.  The converging together lemma (\cref{fact:converging-together}), a form of Slutsky's theorem, gives
\begin{equation}
 \ln\Lambda_M=\sum_eU_{M,e}-\frac12\sum_eU_{M,e}^2+\mathcal R_M\Rightarrow Z\sim N(-s/2,s).
 \label{eq:null-limit}
\end{equation}

\emph{Step 3, the whole trade-off curve.}  Take any pair with $Q\ll P$ and likelihood ratio $\Lambda$.  Let $H$ be the distribution function of $\Lambda$ under $P$, and $H^{-1}(v)=\sup\{y:H(y)<v\}$ for $0<v<1$.  If $V$ is uniform on $(0,1)$, then $H^{-1}(V)$ has distribution function $H$ (\cref{fact:quantile}).  We show that
\begin{equation}
 1-T(P,Q)(\alpha)=\int_{1-\alpha}^1H^{-1}(v)\dd v,\qquad 0\le\alpha\le1 .
 \label{eq:np-quantile}
\end{equation}
At $\alpha=0$ both sides are $0$, because a test with $\E_P\phi=0$ has $\E_Q\phi=\E_P[\Lambda\phi]=0$.  At $\alpha=1$ both sides are $1$, since $\int_0^1H^{-1}(v)\dd v=\E_P\Lambda=1$.  For $0<\alpha<1$, let $k=H^{-1}(1-\alpha)$, which is at least $0$ because $\Lambda\ge0$, and let $\tilde\Lambda=H^{-1}(V)$, which has the distribution of $\Lambda$.  Since $H^{-1}$ is nondecreasing,
\begin{equation}
 \{\tilde\Lambda>k\}\subseteq\{V>1-\alpha\}\subseteq\{\tilde\Lambda\ge k\} .
 \label{eq:quantile-sandwich}
\end{equation}
Put $a=P(\Lambda>k)$ and $b=P(\Lambda=k)$.  Taking probabilities in \eqref{eq:quantile-sandwich} gives $a\le\alpha\le a+b$.  Let $\gamma=(\alpha-a)/b$ if $b>0$ and $\gamma=0$ if $b=0$, so that $0\le\gamma\le1$, and let $\phi^*=\one\{\Lambda>k\}+\gamma\one\{\Lambda=k\}$.  This test has the form required in \eqref{eq:np} and $\alpha_{\phi^*}=a+\gamma b=\alpha$.  By \eqref{eq:quantile-sandwich}, $\Pr(V>1-\alpha,\,\tilde\Lambda=k)=\alpha-a$.  So \eqref{eq:np} gives
\[
 \begin{aligned}
 1-T(P,Q)(\alpha)&=\E_Q\phi^*=\E_P[\Lambda\phi^*]=\E_P\bigl[\Lambda\one\{\Lambda>k\}\bigr]+k\gamma b
 =\E_P\bigl[\Lambda\one\{\Lambda>k\}\bigr]+k(\alpha-a)\\
 &=\E\bigl[\tilde\Lambda\one\{\tilde\Lambda>k\}\bigr]+k\Pr\bigl(V>1-\alpha,\,\tilde\Lambda=k\bigr)
 =\E\bigl[\tilde\Lambda\,\one\{V>1-\alpha\}\bigr]=\int_{1-\alpha}^1H^{-1}(v)\dd v .
 \end{aligned}
\]

Apply \eqref{eq:np-quantile} to the product pair, whose likelihood ratio $\Lambda_M$ has null distribution function $H_M$, and to the pair $N(0,1)$ against $N(\sqrt s,1)$.  The likelihood ratio of the second pair is $e^{\sqrt sX-s/2}$ with $X\sim N(0,1)$ under the null, which has the distribution of $e^Z$ in \eqref{eq:null-limit}.  Let $H$ be its distribution function. The trade-off function of this pair is $G_{\sqrt s}$.  Set $\tilde\Lambda_M=H_M^{-1}(V)$ and $\tilde\Lambda_\infty=H^{-1}(V)$.  Then, for every $\alpha$,
\[
 \bigl|T\bigl((P'_M)^{\otimes n},(Q'_M)^{\otimes n}\bigr)(\alpha)-G_{\sqrt s}(\alpha)\bigr|
 \le\int_0^1\bigl|H_M^{-1}(v)-H^{-1}(v)\bigr|\dd v=\E|\tilde\Lambda_M-\tilde\Lambda_\infty| .
\]
By \eqref{eq:null-limit} and the continuous mapping theorem (\cref{fact:continuous-mapping}), $\Lambda_M\Rightarrow e^Z$.  By \cref{fact:quantile}, then $H_M^{-1}(v)\to H^{-1}(v)$ for all but countably many $v$, so $\tilde\Lambda_M\to \tilde\Lambda_\infty$ almost surely.  Both variables are nonnegative with mean one, since $\E \tilde\Lambda_M=\E\Lambda_M=1$ and $\E \tilde\Lambda_\infty=\E e^Z=1$.  Therefore
\[
 \E|\tilde\Lambda_M-\tilde\Lambda_\infty|=\E(\tilde\Lambda_M-\tilde\Lambda_\infty)+2\E(\tilde\Lambda_\infty-\tilde\Lambda_M)_+=2\E(\tilde\Lambda_\infty-\tilde\Lambda_M)_+\to0
\]
by $|x|=x+2(-x)_+$ and by dominated convergence (\cref{fact:dominated}), since $(\tilde\Lambda_\infty-\tilde\Lambda_M)_+\to0$ almost surely and $0\le(\tilde\Lambda_\infty-\tilde\Lambda_M)_+\le \tilde\Lambda_\infty$ with $\E\tilde\Lambda_\infty=1$.  So the trade-off function of the product pair converges uniformly to $G_{\sqrt s}=G_{\sqrt{r\kappa}}$.
\end{proof}

\begin{proof}[Proof of \cref{thm:fixed-rate}]
We check \eqref{eq:product-gdp-moments}.  Put $W_i=Y_i-1=e^{tX_i-t^2/2}-1$, so $L_M-1=M^{-1}\sum_iW_i$ with independent $W_i$, $\E W_i=0$ and $\E W_i^2=e^{1/\sigma^2}-1$. Hence $M\E(L_M-1)^2=e^{1/\sigma^2}-1$.  Write the fourth power of $\sum_iW_i$ as a sum over ordered quadruples of indices.  A quadruple in which some index appears exactly once has mean zero, because that factor is independent of the others and has mean zero.  The other quadruples either have four equal indices, and there are $M$ of these, or have two distinct indices that each appear twice. There are $3M(M-1)$ of the latter, since the four positions split into two pairs in $3$ ways and the two indices can be chosen in $M(M-1)$ ordered ways.  So
\[
 \E(L_M-1)^4=\frac{M\E W_1^4+3M(M-1)(\E W_1^2)^2}{M^4}=O_\sigma(M^{-2}),
\]
where $\E W_1^4=e^{6/\sigma^2}-4e^{3/\sigma^2}+6e^{1/\sigma^2}-3$ follows from $\E Y_1^k=e^{k(k-1)/(2\sigma^2)}$.  Also $L_M=M^{-1}\sum_iY_i>0$, since every $Y_i$ is positive.  \Cref{lem:product-gdp} with $\kappa=e^{1/\sigma^2}-1$ gives \eqref{eq:gdp-limit}. By \eqref{eq:tv-sup}, the total variations of two pairs differ by at most the largest distance between their trade-off functions.  So uniform convergence of the trade-off functions gives convergence of the total variation to that of $G_\mu$ with $\mu=\sqrt{r(e^{1/\sigma^2}-1)}$, which is $2\Phi(\mu/2)-1$ (Appendix~\ref{app:background}).  By \eqref{eq:sep-tv-acc}, $\operatorname{Acc}^*\to\Phi(\mu/2)$.

For the last sentence of the theorem, let $E_M$ be any sequence of positive integers and put $x_M=E_M(e^{1/\sigma^2}-1)/M$.  If $E_M/M\to0$, then $x_M\to0$, and \cref{thm:finite-rule} with the bounds $U(c)\le\frac12\sqrt c$ and $c_{E_M}\le e^{x_M}-1$ of Appendix~\ref{app:sharp-envelope-proof} gives $d_{\TV}\le\frac12\sqrt{e^{x_M}-1}\to0$.  For the converse, let $S$ be the sum of all $E_MM$ coordinates of a sample.  Under $P_M^{\otimes E_M}$ it is $N(0,E_MM)$.  Under $Q_M^{\otimes E_M}$ each epoch shifts exactly one coordinate by $t$, whichever step it is, so $S$ is $N(E_Mt,E_MM)$.  Dividing by $\sqrt{E_MM}$ gives $N(0,1)$ against $N(\nu_M,1)$ with $\nu_M=t\sqrt{E_M/M}$.  A function of the sample is a Markov kernel, so by data processing and Appendix~\ref{app:background}
\[
 d_{\TV}\bigl(P_M^{\otimes E_M},Q_M^{\otimes E_M}\bigr)\ge2\Phi(\nu_M/2)-1
\]
at every $M$.  If $E_M/M$ does not tend to $0$, there are $c'>0$ and a subsequence along which $E_M/M\ge c'$, and along it the total variation is at least $2\Phi(t\sqrt{c'}/2)-1>0$.  The sum $S$ is the statistic that the optimal test on the linear model of \cref{sec:attacks} thresholds.
\end{proof}

\textbf{A second formula.} Setting the limit $\Phi(\mu/2)$ of \cref{thm:fixed-rate} equal to $\frac12+\frac\lambda2$ with $r=E/M$ gives another one-line noise multiplier, $\sigma=\bigl(\ln(1+4\Phi^{-1}(\tfrac{1+\lambda}2)^2M/E)\bigr)^{-1/2}$.  This multiplier comes from a limit and not from a bound, so a finite run calibrated with it may exceed the target.  In the nine training cells of \cref{sec:calib-trial}, at $\lambda=0.10$, the multiplier lies $0.04\%$ to $0.40\%$ above the exact $\sigma$, but no result shows that it stays above the exact $\sigma$ at other $M$ and $E$.  The closed form of \cref{cor:inverse} pays its extra noise for a bound that holds at every $M$, $E$ and $\sigma$.

\subsection{Small false-positive rates}
\label{app:low-fpr}

By \eqref{eq:tv-sup}, a bound $d_{\TV}\le\lambda$ allows a test with false-positive rate $\alpha$ a true-positive rate up to $\alpha+\lambda$, which is far above $\alpha$ when $\alpha$ is small.  The chi-square value gives a bound that shrinks with $\alpha$, by the Cauchy--Schwarz inequality (\cref{fact:holder}).

\begin{proposition}[Small false-positive rates]
\label{prop:low-fpr}
Let $Q\ll P$ with $\chi^2(Q\|P)=c<\infty$.  Every test $\phi$ with $\E_P\phi=\alpha$ has $\E_Q\phi\le\alpha+\sqrt{c\,\alpha(1-\alpha)}$, and $T(P,Q)(\alpha)\ge1-\alpha-\sqrt{c\,\alpha(1-\alpha)}$ for $0\le\alpha\le\frac12$.  So every test on the model released by \cref{alg:dpsgd} with false-positive rate at most $\alpha\le\frac12$ has true-positive rate at most $\alpha+\sqrt{c_E\alpha(1-\alpha)}$.
\end{proposition}

\begin{proof}
Let $L=\dd Q/\dd P$.  Since $\E_P(L-1)=0$, the Cauchy--Schwarz inequality (\cref{fact:holder}) gives
\[
 \E_Q\phi-\E_P\phi=\E_P[(L-1)\phi]=\E_P[(L-1)(\phi-\alpha)]
 \le\bigl(\E_P(L-1)^2\bigr)^{1/2}\bigl(\E_P(\phi-\alpha)^2\bigr)^{1/2} .
\]
The first factor is $\sqrt c$.  Since $0\le\phi\le1$, we have $\phi^2\le\phi$, so $\E_P(\phi-\alpha)^2=\E_P\phi^2-\alpha^2\le\alpha(1-\alpha)$, and the first claim follows.  The function $g(a)=a+\sqrt{c\,a(1-a)}$ increases on $[0,\frac12]$, because $a(1-a)$ does.  So a test with $\E_P\phi=a\le\alpha\le\frac12$ has type-II error $1-\E_Q\phi\ge1-g(a)\ge1-g(\alpha)$, which bounds $T(P,Q)(\alpha)$.  The pair $(P_M^{\otimes E},Q_M^{\otimes E})$ has $\chi^2=c_E$ by \eqref{eq:tensorize}, and \eqref{eq:transfer} gives $T(P_{\rm out},Q_{\rm out})(\alpha)\ge T(P_M^{\otimes E},Q_M^{\otimes E})(\alpha)\ge1-g(\alpha)$ with $c=c_E$.  A test on the released model with false-positive rate at most $\alpha$ therefore has true-positive rate at most $g(\alpha)$.  The same Cauchy--Schwarz step, with an unbiased estimator in place of $\phi$, gives the variance bound of \citet[Eq.~(31)]{Hammersley1950} and \citet[Eq.~(3)]{ChapmanRobbins1951}.
\end{proof}

\Cref{tab:low-fpr} compares, at the exact $\sigma$ of two training cells of \cref{sec:calib-trial}, the true-positive rate of the optimal test on the pair with the bound of \cref{prop:low-fpr} and with the line $\alpha+\lambda$, $\lambda=0.10$.  The optimal test thresholds $\ln\Lambda$ and randomizes at the threshold.  We take the density of $\ln\Lambda$ under $P_M^{\otimes E}$ that the computation of Appendix~\ref{app:exact-numerics} forms on the grid for the exact $\sigma$, and the density under $Q_M^{\otimes E}$ is $e^s$ times it.  In all nine cells the second grid of Appendix~\ref{app:exact-numerics} gives the same rates to within $10^{-9}$.  At $M=100$ and $E=10$, $2\times10^6$ Monte Carlo samples under each distribution give false-positive and true-positive rates within $0.6$ standard errors of the computed ones.  In all nine cells, at $\alpha=10^{-3}$, the pair's true-positive rate is $0.0022$ to $0.0024$ and the bound is $0.0090$ to $0.0092$, against $\alpha+\lambda=0.101$.

\begin{table}[H]
\caption{True-positive rate of the optimal test and two upper bounds on it at false-positive rate $\alpha$, at the exact $\sigma$ of two cells of \cref{tab:calib-trial} for $\lambda=0.10$.  Pair is the true-positive rate of the optimal test on $(P_M^{\otimes E},Q_M^{\otimes E})$, chi-square is the bound $\alpha+\sqrt{c_E\alpha(1-\alpha)}$ of \cref{prop:low-fpr}, and total variation is $\alpha+\lambda$.  By \eqref{eq:transfer} both bounds hold for every test on the released model.}
\label{tab:low-fpr}
\centering
\small
\begin{tabular}{@{}rcccccc@{}}
\toprule
& \multicolumn{3}{c}{$M=100$, $E=10$, $\sigma=1.4262$} & \multicolumn{3}{c}{$M=500$, $E=10$, $\sigma=0.8344$} \\
\cmidrule(lr){2-4}\cmidrule(lr){5-7}
$\alpha$ & pair & chi-square & total variation & pair & chi-square & total variation \\
\midrule
$0.1$ & $0.1516$ & $0.1767$ & $0.2000$ & $0.1523$ & $0.1771$ & $0.2000$ \\
$0.01$ & $0.0191$ & $0.0354$ & $0.1100$ & $0.0194$ & $0.0356$ & $0.1100$ \\
$0.001$ & $0.0023$ & $0.0091$ & $0.1010$ & $0.0024$ & $0.0091$ & $0.1010$ \\
\bottomrule
\end{tabular}
\end{table}

\subsection{A run of \texorpdfstring{\cref{alg:dpsgd}}{Algorithm 1} that releases the Gaussian pair}
\label{app:attain}

Fix the target record $x_\star$ and a dataset $D$ that does not contain $x_\star$ and has $|D|+1\le N_0$, so that $D$ and $D^+=D\cup\{x_\star\}$ are both valid inputs of \cref{alg:dpsgd}.  Take a model $w$ with $EM$ coordinates, starting at zero.  Take the loss $\ell(w;x)=Cw_1$ for $x=x_\star$ and $\ell(w;x)=0$ for every other record.  The gradient of the target is then $Ce_1$, which has length $C$ and is not changed by clipping, and every record of $D$ has gradient zero, like a ghost record.  At the $k$th of the $EM$ updates, let the optimizer set $w_k$ to $B_0/(\sigma C)$ times the first coordinate of $S$ and leave the other coordinates of $w$ unchanged.  This optimizer is fixed in advance, as \cref{alg:dpsgd} requires.  By line~\ref{line:noise}, the first coordinate of $S$ at the $k$th update is $(CI_k+\sigma CZ_k)/B_0$, where $I_k$ is $1$ if the target is in the step of that update and $0$ otherwise, and $Z_1,\dots,Z_{EM}$ are independent standard Gaussian values.  So the released model has the coordinates $w_k=I_k/\sigma+Z_k$.  On $D$ every $I_k$ is $0$, and $w$ is distributed as $N(0,I_{EM})=P_M^{\otimes E}$.  On $D^+$ the target is in exactly one step of each epoch, chosen uniformly and independently across epochs, so the $M$ coordinates of an epoch are distributed as $Q_M$ with $t=1/\sigma$, independently across epochs, and $w$ is distributed as $Q_M^{\otimes E}$.  The total variation between the released models on $D$ and on $D^+$ is therefore $d_{\TV}(P_M^{\otimes E},Q_M^{\otimes E})$.  Fix a target $0<\lambda<1$ on the total variation.  An accountant whose guarantee holds for every run of \cref{alg:dpsgd} with the same $M$, $E$ and $\sigma$ also covers this run, so every $\sigma$ it certifies for $\lambda$ has $d_{\TV}(P_M^{\otimes E},Q_M^{\otimes E})\le\lambda$.  These $\sigma$ lie in the interval $[\sigma_\lambda,\infty)$ of Appendix~\ref{app:exact-numerics}, whose left end is the exact $\sigma$.

\section{The Exact Total Variation}
\label{app:exact-tv}

This appendix proves \cref{thm:exact-tv}.  We first prove three integral formulas, one for each formula of the theorem, and then combine them.  After that we bound the error from cutting the integrals off at a finite point, describe the computation, and prove the limit $\sqrt{\pi/2}$ stated after \cref{thm:fixed-rate}.

\textbf{Notation.} Under $P_M$ the coordinates $X_1,\dots,X_M$ are independent standard Gaussians, and $t=1/\sigma$.  As in \eqref{eq:one-epoch-lr} and before \cref{thm:exact-tv},
\[
 Y_j=e^{tX_j-t^2/2},\qquad
 L_M=\frac1M\sum_{j=1}^MY_j,\qquad
 Z=\sum_{j=1}^M(Y_j-1),
\]
so $L_M=1+Z/M$.  Each $Y_j$ is lognormal, hence positive, so $L_M>0$.  Also $\E Y_j=1$ and $\E Y_j^2=e^{t^2}$ (Appendix~\ref{app:one-epoch-proof}).  The $M$ terms $Y_j-1$ are independent with mean zero and variance $e^{t^2}-1=e^{1/\sigma^2}-1$.  Hence
\[
 \E L_M=1,\qquad \E Z=0,\qquad \E Z^2=M(e^{1/\sigma^2}-1) .
\]
Over $E$ epochs we have independent copies $L_{M,1},\dots,L_{M,E}$ of $L_M$, one for each epoch.  By Appendix~\ref{app:one-epoch-proof}, the likelihood ratio of the product pair is their product
\[
 \Lambda=\prod_{e=1}^EL_{M,e}.
\]

For $x>0$ and a complex number $q=a+\mathrm iu$ we write $x^q=e^{q\ln x}$, so that $|x^q|=x^a$.  For $0\le a\le1$,
\[
 x^a\le\max\{1,x\}\le1+x ,
\]
because $x^a\le1$ when $x\le1$ and $x^a\le x$ when $x\ge1$.  With $x=L_M$ this gives $|L_M^q|=L_M^a\le1+L_M$, and $\E(1+L_M)=2$.  So for $0\le\Re q\le1$ the one-epoch factor
\[
 F_M(q)=\E L_M^q
\]
is finite.  Since $\ln\Lambda=\sum_e\ln L_{M,e}$, we have $\Lambda^q=\prod_eL_{M,e}^q$.  The factors are independent, and the expectation of a product of independent integrable variables is the product of their expectations (\cref{fact:product}).  This extends to complex factors.  Write $L_{M,e}^q=A_e+\mathrm iB_e$ with $A_e$ and $B_e$ real.  Then $|A_e|\le1+L_{M,e}$ and $|B_e|\le1+L_{M,e}$.  For every set $T$ of epochs, the variables $B_e$ with $e\in T$ and $A_e$ with $e\notin T$ are functions of different epochs, so they are independent (\cref{fact:indep-functions}).  Expanding the product and applying the real case to each term,
\[
 \E\prod_{e=1}^E(A_e+\mathrm iB_e)
 =\sum_{T\subseteq\{1,\dots,E\}}\mathrm i^{|T|}\prod_{e\notin T}\E A_e\prod_{e\in T}\E B_e
 =\prod_{e=1}^E\bigl(\E A_e+\mathrm i\,\E B_e\bigr).
\]
So
\[
 \E\Lambda^q=\prod_{e=1}^E\E L_{M,e}^q=F_M(q)^E .
\]
In particular $\E\Lambda=F_M(1)^E=1$.  By Appendix~\ref{app:background}, $d_{\TV}=\frac12\E|\Lambda-1|=\E(1-\Lambda)_+$.  For every real $x$ we have $(1-x)_+=1-\min\{1,x\}$, since both sides equal $1-x$ when $x\le1$ and $0$ when $x\ge1$.  So also
\[
 d_{\TV}\bigl(P_M^{\otimes E},Q_M^{\otimes E}\bigr)=1-\E\min\{1,\Lambda\}.
\]

\subsection{Three integral formulas}

\begin{lemma}[{The absolute value as an integral, \citealp[Eqs.~(7) and~(8) with $r=1$]{vonBahrEsseen1965}}]
\label{lem:abs-int}
For every real $z$,
\[
 |z|=\frac2\pi\int_0^\infty\frac{1-\cos(uz)}{u^2}\dd u .
\]
\end{lemma}

Taking expectations gives \citet[Lemma~2]{vonBahrEsseen1965} with $r=1$, the step behind \eqref{eq:exact-tv-cf}.  The same formula for $\E|X|$ is the case $n=0$ of \citet[Theorem~1(a)]{LinHu2021}, and Lin and Hu attribute their Theorem~1 to \citet[Theorems~2.1 and~4.1]{Hsu1951}.

\begin{proof}
Both sides are zero at $z=0$.  For $z\ne0$, substitute $v=u|z|$, so that $\dd u=\dd v/|z|$ and $1/u^2=z^2/v^2$.  Since $\cos$ is even, $\cos(uz)=\cos(v)$, and
\[
 \int_0^\infty\frac{1-\cos(uz)}{u^2}\dd u=|z|\int_0^\infty\frac{1-\cos(v)}{v^2}\dd v .
\]
It remains to show that the last integral equals $\pi/2$.  We use two Laplace integrals. First, for every $v>0$, substituting $y=vx$ gives
\[
 \int_0^\infty x\,e^{-vx}\dd x=\frac1{v^2}\int_0^\infty y\,e^{-y}\dd y=\frac1{v^2},
\]
because $-(1+y)e^{-y}$ has derivative $ye^{-y}$, equals $-1$ at $y=0$ and tends to $0$ as $y\to\infty$, so the last integral is $0-(-1)=1$.  Second, for every $x>0$ and real $u$,
\begin{equation}
 \int_0^\infty\cos(uv)\,e^{-xv}\dd v
 =\Re\int_0^\infty e^{-(x-\mathrm iu)v}\dd v
 =\Re\frac1{x-\mathrm iu}
 =\Re\frac{x+\mathrm iu}{x^2+u^2}
 =\frac{x}{x^2+u^2} .
 \label{eq:laplace-cos}
\end{equation}
Here the complex integral equals $1/(x-\mathrm iu)$ because $-e^{-(x-\mathrm iu)v}/(x-\mathrm iu)$ is an antiderivative of the integrand, applied to its real and imaginary parts, so that
\[
 \int_0^Re^{-(x-\mathrm iu)v}\dd v=\frac{1-e^{-(x-\mathrm iu)R}}{x-\mathrm iu}\longrightarrow\frac1{x-\mathrm iu}
 \qquad(R\to\infty),
\]
since $|e^{-(x-\mathrm iu)R}|=e^{-xR}\to0$.

Now write $1/v^2$ with the first integral.  The integrand $(1-\cos(v))\,x\,e^{-vx}$ is nonnegative, so Tonelli's theorem (\cref{fact:fubini}) lets us exchange the order of integration,
\[
 \int_0^\infty\frac{1-\cos(v)}{v^2}\dd v
 =\int_0^\infty\int_0^\infty(1-\cos(v))\,x\,e^{-vx}\dd x\,\dd v
 =\int_0^\infty x\int_0^\infty(1-\cos(v))\,e^{-vx}\dd v\,\dd x .
\]
The inner integral follows from $\int_0^\infty e^{-vx}\dd v=1/x$ and \eqref{eq:laplace-cos} with $u=1$,
\[
 \int_0^\infty(1-\cos(v))\,e^{-vx}\dd v=\frac1x-\frac{x}{1+x^2}
 =\frac{1+x^2-x^2}{x(1+x^2)}=\frac1{x(1+x^2)} .
\]
Hence, since $\arctan$ has derivative $1/(1+x^2)$,
\[
 \int_0^\infty\frac{1-\cos(v)}{v^2}\dd v=\int_0^\infty\frac{\dd x}{1+x^2}
 =\lim_{x\to\infty}\arctan x-\arctan0=\frac\pi2 .
\]
Multiplying by $2/\pi$ gives $|z|$.
\end{proof}

\begin{lemma}[{The minimum as an integral, \citealp[Table~3.1]{Lewis2001}}]
\label{lem:min-int}
For every $x>0$,
\[
 \min\{1,x\}=\frac{\sqrt x}{2\pi}\int_{-\infty}^{\infty}\frac{x^{-\mathrm iu}}{u^2+\frac14}\dd u ,
\]
where $x^{-\mathrm iu}=e^{-\mathrm iu\ln x}$.
\end{lemma}

For $K=1$ the payoff $\min\{e^y,K\}$ in \citet[Table~3.1]{Lewis2001} has the transform $1/(z^2-\mathrm iz)$ for $0<\Im z<1$, and inverting it along $\Im z=\frac12$, as on his p.~11, gives the lemma at $x=e^y$.

\begin{proof}
Let $V$ have the standard Cauchy density $\{\pi(1+v^2)\}^{-1}$.  Then $U=V/2$ has the density $\{2\pi(u^2+\frac14)\}^{-1}$, which is the weight in the integral, so the right side of the lemma is
\[
 \sqrt x\,\E x^{-\mathrm iU}=\sqrt x\,\E e^{-\mathrm i(\ln x/2)V}.
\]
The characteristic function of $V$ is $\E e^{\mathrm isV}=e^{-|s|}$ (\cref{fact:chf}), so the right side of the lemma is $\sqrt x\,e^{-|\ln x|/2}$, which is $1$ for $x\ge1$ and $x$ for $x<1$.

For completeness we derive this characteristic function from \eqref{eq:laplace-cos} and the inversion formula.  The density $\frac12e^{-|y|}$ is even, so by \eqref{eq:laplace-cos}, with $1$ in place of its parameter $x$, $s$ in place of $u$ and $y$ as the variable of integration, it has the characteristic function $\int_0^\infty\cos(sy)\,e^{-y}\dd y=(1+s^2)^{-1}$, which is integrable.  By the inversion formula (\cref{fact:inversion}), the same distribution then has the continuous density
\[
 \frac1{2\pi}\int_{-\infty}^{\infty}\frac{e^{-\mathrm isy}}{1+s^2}\dd s=\frac12\,\E e^{-\mathrm iyV} .
\]
The two continuous densities $\frac12e^{-|y|}$ and $\frac12\E e^{-\mathrm iyV}$ have the same integral over every interval, and a continuous function whose integral over every interval is zero vanishes at every point.  Hence $\E e^{-\mathrm iyV}=e^{-|y|}$ for every real $y$, which is $\E e^{\mathrm isV}=e^{-|s|}$ with $s=-y$.
\end{proof}

\begin{lemma}[{A power as an integral, \citealp[Eqs.~5.9.6 and~5.9.7]{DLMF}}]
\label{lem:pow-int}
Let $q$ be a complex number with $0<\Re q<1$.  For every $x>0$,
\[
 x^q=\frac{q\,e^{\mathrm i\pi q/2}}{\Gamma(1-q)}\int_0^\infty s^{-q-1}\bigl(1-e^{\mathrm isx}\bigr)\dd s ,
\]
and the integral converges absolutely.
\end{lemma}

With $1-q$ in place of the $z$ of those formulas, they give $\int_0^\infty t^{-q}e^{\mathrm it}\dd t=\Gamma(1-q)e^{\mathrm i\pi(1-q)/2}$, and the substitution $t=sx$ and one integration by parts give the lemma.  For real $q$ and $x=1$ the lemma is the case $m=0$, $\nu=q$ of the integral $J(\nu)$ that \citet[pp.~810--811]{vonBahr1965} evaluates.

\begin{proof}
Put $a=\Re q$.  For real $\theta$ we have $|1-e^{\mathrm i\theta}|=|\int_0^\theta\mathrm ie^{\mathrm i\phi}\dd\phi|\le|\theta|$ and $|1-e^{\mathrm i\theta}|\le2$.  With $|s^{-q-1}|=s^{-a-1}$ this gives
\begin{equation}
 \bigl|s^{-q-1}(1-e^{\mathrm isx})\bigr|\le s^{-a-1}\min\{sx,2\}.
 \label{eq:pow-dom}
\end{equation}
The right side is integrable because $0<a<1$,
\[
 \int_0^\infty s^{-a-1}\min\{sx,2\}\dd s
 =x\int_0^{2/x}s^{-a}\dd s+2\int_{2/x}^\infty s^{-a-1}\dd s
 =x\,\frac{(2/x)^{1-a}}{1-a}+2\,\frac{(2/x)^{-a}}a
 =2^{1-a}x^a\Bigl(\frac1{1-a}+\frac1a\Bigr).
\]

We first compute the integral with $e^{\mathrm isx}$ replaced by $e^{-ws}$, where $\Re w>0$, and then let $w$ tend to $-\mathrm ix$.  For $\Re w\ge0$ and $s\ge0$ we have $|e^{-ws}|\le1$, so $|1-e^{-ws}|=|\int_0^sw\,e^{-wr}\dd r|\le|w|s$, and also $|1-e^{-ws}|\le2$.  Hence $|s^{-q-1}(1-e^{-ws})|\le s^{-a-1}\min\{|w|s,2\}$, which is integrable by the computation above with $|w|$ in place of $x$.  Powers of a complex $w$ off the half-line $(-\infty,0]$ are principal values, $w^p=e^{p\ln w}$ with $\ln w=\ln|w|+\mathrm i\arg w$ and $-\pi<\arg w<\pi$ (\cref{fact:log}).  For $\Re w>0$, the Gamma integral of \cref{fact:gamma} with $\nu=1-q$ gives
\[
 \int_0^\infty s^{-q}e^{-ws}\dd s=\Gamma(1-q)\,w^{q-1},
\]
since $\Re(1-q)=1-a>0$.  Since $q\ne0$, the function $-s^{-q}/q$ has derivative $s^{-q-1}$.  Integrating by parts on $[\delta,R]$ with $0<\delta<R$ gives
\[
 \int_\delta^Rs^{-q-1}\bigl(1-e^{-ws}\bigr)\dd s
 =\Bigl[-\frac{s^{-q}}q\bigl(1-e^{-ws}\bigr)\Bigr]_{s=\delta}^{s=R}
 +\frac wq\int_\delta^Rs^{-q}e^{-ws}\dd s .
\]
By the two bounds on $|1-e^{-ws}|$, the boundary term is at most $|w|\delta^{1-a}/|q|$ at $s=\delta$ and at most $2R^{-a}/|q|$ at $s=R$, and both tend to zero because $0<a<1$.  The last integral converges absolutely, since its integrand has modulus $s^{-a}e^{-s\Re w}$.  Letting $\delta\to0$ and $R\to\infty$, for $\Re w>0$
\[
 \int_0^\infty s^{-q-1}\bigl(1-e^{-ws}\bigr)\dd s=\frac wq\,\Gamma(1-q)\,w^{q-1}=\frac{\Gamma(1-q)}q\,w^q ,
\]
because $w\,w^{q-1}=e^{\ln w}e^{(q-1)\ln w}=w^q$.

Take $w=\epsilon-\mathrm ix$ with $0<\epsilon\le1$, so that $1-e^{-ws}=1-e^{-\epsilon s}e^{\mathrm isx}$, and let $\epsilon\downarrow0$.  On the left, the integrand tends to $s^{-q-1}(1-e^{\mathrm isx})$, and its modulus is at most $s^{-a-1}\min\{(1+x)s,2\}$, since $|w|\le1+x$.  This bound is integrable, so dominated convergence (\cref{fact:dominated}), applied to the real and imaginary parts, holds along every sequence $\epsilon\downarrow0$.  On the right, the principal logarithm is analytic, hence continuous, off the half-line $(-\infty,0]$ (\cref{fact:log}).  The point $-\mathrm ix$ lies off this half-line, with $|-\mathrm ix|=x$ and $\arg(-\mathrm ix)=-\pi/2$.  Hence
\[
 (\epsilon-\mathrm ix)^q=e^{q\ln(\epsilon-\mathrm ix)}\to e^{q(\ln x-\mathrm i\pi/2)}=x^q\,e^{-\mathrm i\pi q/2}.
\]
So
\[
 \int_0^\infty s^{-q-1}\bigl(1-e^{\mathrm isx}\bigr)\dd s=\frac{\Gamma(1-q)}{q}\,x^q\,e^{-\mathrm i\pi q/2}.
\]
The number $\Gamma(1-q)$ is finite, because $\Re(1-q)>0$, and it is not zero, because the Gamma function has no zeros (\cref{fact:gamma}).  So we may multiply both sides by $q\,e^{\mathrm i\pi q/2}/\Gamma(1-q)$, which gives the claim.
\end{proof}

\subsection{Proof of \texorpdfstring{\cref{thm:exact-tv}}{the exact total variation theorem}}

\begin{proof}
We prove the three formulas in turn.

\emph{The one-epoch formula \eqref{eq:exact-tv-cf}.} For $E=1$ we have $\Lambda=L_M=1+Z/M$, so
\[
 d_{\TV}(P_M,Q_M)=\frac12\E|L_M-1|=\frac{\E|Z|}{2M}.
\]
Apply \cref{lem:abs-int} with $z=Z$ and take expectations.  The integrand $(1-\cos(uZ))/u^2$ is nonnegative, so Tonelli's theorem (\cref{fact:fubini}) lets us exchange the expectation and the integral,
\[
 \E|Z|=\frac2\pi\int_0^\infty\frac{1-\E\cos(uZ)}{u^2}\dd u .
\]
Finally, $\E\cos(uZ)=\Re\,\E e^{\mathrm iuZ}=\Re\varphi_Z(u)$.  Dividing by $2M$ gives \eqref{eq:exact-tv-cf}.  The integral converges absolutely.  For real $w$, since $\sin(v)\le v$ for $v\ge0$,
\[
 1-\cos(w)=\int_0^{|w|}\sin(v)\dd v\le\int_0^{|w|}v\dd v=\frac{w^2}2 ,
\]
so the integrand $(1-\E\cos(uZ))/u^2$ is at most $\E Z^2/2=M(e^{1/\sigma^2}-1)/2$.  It is also at most $2/u^2$.

\emph{The formula \eqref{eq:exact-tv-E} for every $E$.} Apply \cref{lem:min-int} with $x=\Lambda$, which is positive, and take expectations.  Since $\sqrt\Lambda\,\Lambda^{-\mathrm iu}=\Lambda^{1/2-\mathrm iu}$,
\[
 \E\min\{1,\Lambda\}=\frac1{2\pi}\int_{-\infty}^{\infty}\frac{\E\Lambda^{1/2-\mathrm iu}}{u^2+\frac14}\dd u .
\]
Fubini's theorem (\cref{fact:fubini}), applied to the real and the imaginary part, allows the exchange, because $|\Lambda^{1/2-\mathrm iu}|=\Lambda^{1/2}$ and
\[
 \int_{-\infty}^{\infty}\frac{\E\Lambda^{1/2}}{u^2+\frac14}\dd u=2\pi\,\E\Lambda^{1/2}\le2\pi .
\]
Here $2\arctan(2u)$ is an antiderivative of $(u^2+\frac14)^{-1}$, which gives the factor $2\pi$.  The Cauchy--Schwarz inequality (\cref{fact:holder}) gives $\E\Lambda^{1/2}=\E(\Lambda^{1/2}\cdot1)\le(\E\Lambda)^{1/2}=1$.

By the product formula for $\E\Lambda^q$ in the notation above, the numerator is $F_M(\frac12-\mathrm iu)^E$.  Since $L_M$ is positive, $\ln L_M$ is real, and
\[
 L_M^{1/2-\mathrm iu}=e^{(1/2-\mathrm iu)\ln L_M}=\overline{e^{(1/2+\mathrm iu)\ln L_M}}=\overline{L_M^{1/2+\mathrm iu}} .
\]
Taking expectations, $F_M(\frac12-\mathrm iu)=\overline{F_M(\frac12+\mathrm iu)}$.  Write
\[
 G(u)=F_M\bigl(\tfrac12+\mathrm iu\bigr)^E .
\]
The numerator is then $\overline{G(u)}$, and also $G(-u)=\overline{G(u)}$.  Split the integral at zero and replace $u$ by $-u$ on the negative half,
\[
 \int_{-\infty}^{\infty}\frac{\overline{G(u)}}{u^2+\frac14}\dd u
 =\int_0^\infty\frac{\overline{G(u)}+\overline{G(-u)}}{u^2+\frac14}\dd u
 =\int_0^\infty\frac{\overline{G(u)}+G(u)}{u^2+\frac14}\dd u
 =2\int_0^\infty\frac{\Re G(u)}{u^2+\frac14}\dd u .
\]
Hence
\[
 \E\min\{1,\Lambda\}=\frac1\pi\int_0^\infty\frac{\Re G(u)}{u^2+\frac14}\dd u .
\]
Since $\frac1\pi\int_0^\infty(u^2+\frac14)^{-1}\dd u=\frac1\pi\cdot2\cdot\frac\pi2=1$, subtracting from one gives
\[
 d_{\TV}\bigl(P_M^{\otimes E},Q_M^{\otimes E}\bigr)=1-\E\min\{1,\Lambda\}
 =\frac1\pi\int_0^\infty\frac{1-\Re G(u)}{u^2+\frac14}\dd u ,
\]
which is \eqref{eq:exact-tv-E}.  The integrand is nonnegative, because
\[
 \Re G(u)\le|G(u)|=\bigl|F_M(\tfrac12+\mathrm iu)\bigr|^E\le F_M(\tfrac12)^E\le1 .
\]
Here the modulus of an expectation is at most the expectation of the modulus, so $|F_M(\frac12+\mathrm iu)|\le\E|L_M^{1/2+\mathrm iu}|=\E L_M^{1/2}=F_M(\frac12)$.  The Cauchy--Schwarz inequality (\cref{fact:holder}) gives $F_M(\frac12)=\E L_M^{1/2}\le(\E L_M)^{1/2}=1$.  The integrand is also at most $2/(u^2+\frac14)$, so the integral converges absolutely.

\emph{The one-epoch factor \eqref{eq:mellin-cf}.} Fix $q$ with $0<a=\Re q<1$.  Apply \cref{lem:pow-int} with $x=L_M$, which is positive, and take expectations,
\[
 F_M(q)=\E L_M^q=\frac{q\,e^{\mathrm i\pi q/2}}{\Gamma(1-q)}\int_0^\infty s^{-q-1}\bigl(1-\E e^{\mathrm isL_M}\bigr)\dd s .
\]
Fubini's theorem (\cref{fact:fubini}), applied to the real and the imaginary part, allows the exchange.  Indeed, by \eqref{eq:pow-dom} with $x=L_M$,
\[
 \int_0^\infty s^{-a-1}\,\E\bigl|1-e^{\mathrm isL_M}\bigr|\dd s
 \le\int_0^\infty s^{-a-1}\,\E\min\{sL_M,2\}\dd s
 \le\int_0^\infty s^{-a-1}\min\{s,2\}\dd s<\infty .
\]
The second inequality holds because a minimum is at most each of its two terms, so $\E\min\{sL_M,2\}\le\min\{s\,\E L_M,2\}=\min\{s,2\}$.  The last integral is the case $x=1$ of the integral computed in the proof of \cref{lem:pow-int}.

Since $L_M=1+Z/M$, we have $\E e^{\mathrm isL_M}=e^{\mathrm is}\,\E e^{\mathrm i(s/M)Z}=e^{\mathrm is}\varphi_Z(s/M)$.  Split
\[
 1-e^{\mathrm is}\varphi_Z(s/M)=\bigl(1-e^{\mathrm is}\bigr)+e^{\mathrm is}\bigl(1-\varphi_Z(s/M)\bigr).
\]
The integral before the split converges absolutely, by the Fubini bound above and $|1-\E e^{\mathrm isL_M}|\le\E|1-e^{\mathrm isL_M}|$.  The first part converges absolutely and contributes $1^q=1$, by \cref{lem:pow-int} with $x=1$.  The integrand of the second part is the difference of these two integrands, so its integral also converges absolutely, by the triangle inequality.  It is the integral in \eqref{eq:mellin-cf}, and this proves \eqref{eq:mellin-cf}.
\end{proof}

\subsection{Truncation bounds}

The computation cuts each integral off at a point $u_{\max}$.  The next lemma bounds the part that is cut off.

\begin{lemma}[Truncation]
\label{lem:truncation}
Let $u_{\max}>0$, and recall that $t=1/\sigma$ is the shift in the Gaussian pair \eqref{eq:one-record-pair}.
\begin{enumerate}[label=(\roman*)]
\item Let $\gamma=Et^2/(2M)$.  The part of the integral in \eqref{eq:exact-tv-E} over $u>u_{\max}$ differs from $\frac2\pi\arctan\frac1{2u_{\max}}$ by at most $e^{-\gamma u_{\max}^2}/(\pi u_{\max})$.
\item Let $C_t=2e^{t^2}/(t\sqrt{2\pi})$.  The part of the integral in \eqref{eq:exact-tv-cf} over $u>u_{\max}$ differs from $1/(\pi Mu_{\max})$ by at most $C_t^M/\{\pi M(M+1)u_{\max}^{M+1}\}$.
\end{enumerate}
\end{lemma}

\begin{proof}
(i) The Gaussian decay comes from the mean of the coordinates, which contributes an independent Gaussian term to $\ln L_M$.  Let $\bar X=M^{-1}\sum_jX_j$ be this mean.  Splitting each exponent,
\[
 tX_j-\frac{t^2}2=\Bigl(t\bar X-\frac{t^2}{2M}\Bigr)+\Bigl(t(X_j-\bar X)-\frac{(M-1)t^2}{2M}\Bigr),
\]
where the constants add up because $\frac1M+\frac{M-1}M=1$.  So $L_M=e^{t\bar X-t^2/(2M)}A$ with
\[
 A=\frac1M\sum_{j=1}^M\exp\Bigl(t(X_j-\bar X)-\frac{(M-1)t^2}{2M}\Bigr).
\]
The mean $\bar X$ is independent of $A$.  To see this, let $H$ be an orthogonal $M\times M$ matrix whose first row is $(1,\dots,1)/\sqrt M$.  Such an $H$ exists because every unit vector extends to an orthonormal basis.  Let $W=HX$ with $X=(X_1,\dots,X_M)^\top$.  As a linear image of a vector of independent standard Gaussians, $W$ has the multivariate normal distribution with covariance matrix $HH^\top=I$, which by \cref{fact:gaussian-vector} has the density
\[
 (2\pi)^{-M/2}e^{-\lVert w\rVert^2/2}=\prod_{j=1}^M\frac{e^{-w_j^2/2}}{\sqrt{2\pi}} .
\]
By Tonelli's theorem (\cref{fact:fubini}) the distribution function of $W$ is then the product of $M$ standard normal distribution functions, so $W_1,\dots,W_M$ are independent standard Gaussians (\cref{fact:indep-df}).  The first column of $H^\top$ is $(1,\dots,1)^\top/\sqrt M$, so $W_1=\sqrt M\,\bar X$ and
\[
 (X_1-\bar X,\dots,X_M-\bar X)^\top
 =H^\top W-\frac{W_1}{\sqrt M}(1,\dots,1)^\top
 =H^\top(0,W_2,\dots,W_M)^\top .
\]
Thus $\bar X$ is a function of $W_1$ and $A$ is a function of $W_2,\dots,W_M$, so they are independent (\cref{fact:indep-functions}).  Also $\bar X=W_1/\sqrt M\sim N(0,1/M)$.

Next, let $V\sim N(0,v)$ and $q=a+\mathrm iu$.  Completing the square,
\[
 e^{ay}\,\frac{e^{-y^2/(2v)}}{\sqrt{2\pi v}}=e^{a^2v/2}\,\frac{e^{-(y-av)^2/(2v)}}{\sqrt{2\pi v}} ,
\]
so $\E e^{aV}g(V)=e^{a^2v/2}\,\E g(V+av)$ for bounded $g$.  With $g(y)=e^{\mathrm iuy}$ and the characteristic function $e^{-u^2v/2}$ of $V$ (\cref{fact:chf}),
\[
 \E e^{qV}=e^{a^2v/2}\,e^{\mathrm iuav}\,e^{-u^2v/2}=e^{q^2v/2}.
\]
With $V=t\bar X$ and $v=t^2/M$ this gives $\E e^{qt\bar X}=e^{q^2t^2/(2M)}$.  At $q=1$ this gives $\E e^{t\bar X-t^2/(2M)}=1$.  Since $\bar X$ is independent of $A$ and both factors of $L_M=e^{t\bar X-t^2/(2M)}A$ are positive, $1=\E L_M=\E e^{t\bar X-t^2/(2M)}\,\E A=\E A$ by the product rule (\cref{fact:product}), and the Cauchy--Schwarz inequality (\cref{fact:holder}) gives $\E\sqrt A\le(\E A)^{1/2}=1$.  We need only $q=\frac12+\mathrm iu$.  Both factors of $L_M^q=e^{q(t\bar X-t^2/(2M))}A^q$ are then integrable, since $|e^{qt\bar X}|=e^{t\bar X/2}$ has mean $e^{t^2/(8M)}$ and $|A^q|=\sqrt A$ has mean at most $1$.  Since $\bar X$ is independent of $A$, the product rule for expectations, extended to complex factors as in the notation above, gives
\[
 F_M(q)=e^{-qt^2/(2M)}\,\E e^{qt\bar X}\,\E A^q=e^{q(q-1)t^2/(2M)}\,\E A^q .
\]
For $q=\frac12+\mathrm iu$ we have $q(q-1)=-(u^2+\frac14)$ and $|\E A^q|\le\E\sqrt A\le1$, so
\[
 |F_M(\tfrac12+\mathrm iu)|\le e^{-(u^2+1/4)t^2/(2M)}\le e^{-t^2u^2/(2M)} .
\]
For $G(u)=F_M(\frac12+\mathrm iu)^E$ as in the proof of \cref{thm:exact-tv} this gives
\[
 |G(u)|\le e^{-Et^2u^2/(2M)}=e^{-\gamma u^2}.
\]
The part of \eqref{eq:exact-tv-E} over $u>u_{\max}$ is
\[
 \frac1\pi\int_{u_{\max}}^\infty\frac{\dd u}{u^2+\frac14}
 -\frac1\pi\int_{u_{\max}}^\infty\frac{\Re G(u)}{u^2+\frac14}\dd u .
\]
The first term is
\[
 \frac1\pi\int_{u_{\max}}^\infty\frac{\dd u}{u^2+\frac14}
 =\frac2\pi\Bigl(\frac\pi2-\arctan(2u_{\max})\Bigr)
 =\frac2\pi\arctan\frac1{2u_{\max}} ,
\]
because $2\arctan(2u)$ is an antiderivative of the integrand, and $\arctan y+\arctan(1/y)=\pi/2$ for $y>0$.  The last identity holds because the left side has derivative $\frac1{1+y^2}-\frac{1/y^2}{1+1/y^2}=0$ and equals $\pi/2$ at $y=1$.  In the second term, $|\Re G(u)|\le e^{-\gamma u^2}\le e^{-\gamma u_{\max}^2}$ and $(u^2+\frac14)^{-1}\le u^{-2}$, so it is at most
\[
 \frac1\pi\,e^{-\gamma u_{\max}^2}\int_{u_{\max}}^\infty\frac{\dd u}{u^2}=\frac{e^{-\gamma u_{\max}^2}}{\pi u_{\max}}
\]
in absolute value.

(ii) As in \eqref{eq:one-epoch-lr} and \cref{sec:exact}, $Y=e^{tX-t^2/2}$ with $X\sim N(0,1)$, so $\ln Y=tX-t^2/2\sim N(-t^2/2,t^2)$ and, since $t>0$, $\Pr(Y\le y)=\Pr\bigl(X\le(\ln y+t^2/2)/t\bigr)=\Phi\bigl((\ln y+t^2/2)/t\bigr)$ for $y>0$.  Differentiating in $y$ gives the lognormal density of $Y$,
\[
 f_t(y)=\frac1{yt\sqrt{2\pi}}\exp\Bigl(-\frac{(\ln y+t^2/2)^2}{2t^2}\Bigr),\qquad y>0 .
\]
It tends to zero at both ends of $(0,\infty)$, because as $|\ln y|\to\infty$ the term quadratic in $\ln y$ in the exponent outweighs the factor $1/y=e^{-\ln y}$.  Its logarithmic derivative is
\[
 \frac{f_t'(y)}{f_t(y)}=-\frac1y-\frac{\ln y+t^2/2}{t^2y}=-\frac{\ln y+3t^2/2}{t^2y},
\]
which is positive for $y<e^{-3t^2/2}$ and negative for $y>e^{-3t^2/2}$.  So $f_t$ rises to a single maximum at the mode $y_*=e^{-3t^2/2}$ and then falls.  At the mode $\ln y_*+t^2/2=-t^2$, so
\[
 f_t(y_*)=\frac{e^{3t^2/2}\,e^{-t^2/2}}{t\sqrt{2\pi}}=\frac{e^{t^2}}{t\sqrt{2\pi}} .
\]
Since $f_t$ increases from $0$ to $f_t(y_*)$ and then decreases to $0$,
\[
 \int_0^\infty|f_t'(y)|\dd y=\int_0^{y_*}f_t'(y)\dd y-\int_{y_*}^\infty f_t'(y)\dd y=2f_t(y_*)=C_t .
\]
For $u>0$, integration by parts on $[\delta,R]$ gives
\[
 \int_\delta^Re^{\mathrm iuy}f_t(y)\dd y=\Bigl[\frac{e^{\mathrm iuy}f_t(y)}{\mathrm iu}\Bigr]_{y=\delta}^{y=R}
 -\frac1{\mathrm iu}\int_\delta^Re^{\mathrm iuy}f_t'(y)\dd y .
\]
The boundary term tends to zero as $\delta\to0$ and $R\to\infty$, because $|e^{\mathrm iuy}|=1$ and $f_t$ tends to zero at both ends.  Since $f_t'$ is integrable,
\[
 \varphi_Y(u)=\int_0^\infty e^{\mathrm iuy}f_t(y)\dd y=-\frac1{\mathrm iu}\int_0^\infty e^{\mathrm iuy}f_t'(y)\dd y ,
 \qquad
 |\varphi_Y(u)|\le\frac1u\int_0^\infty|f_t'(y)|\dd y=\frac{C_t}u .
\]
Hence $|\varphi_Z(u)|=|e^{-\mathrm iu}\varphi_Y(u)|^M\le(C_t/u)^M$.  The part of \eqref{eq:exact-tv-cf} over $u>u_{\max}$ is
\[
 \frac1{\pi M}\int_{u_{\max}}^\infty\frac{\dd u}{u^2}
 -\frac1{\pi M}\int_{u_{\max}}^\infty\frac{\Re\varphi_Z(u)}{u^2}\dd u .
\]
The first term equals $1/(\pi Mu_{\max})$.  The second is at most
\[
 \frac{C_t^M}{\pi M}\int_{u_{\max}}^\infty u^{-M-2}\dd u=\frac{C_t^M}{\pi M(M+1)u_{\max}^{M+1}}
\]
in absolute value.
\end{proof}

\subsection{Numerical method and checks}
\label{app:exact-numerics}

\textbf{The exact $\sigma$.} With $t=1/\sigma$, the distribution $P_M^{\otimes E}=N(0,I_{EM})$ does not depend on $\sigma$, and $Q_M^{\otimes E}$ is the average of $N(t\one_S,I_{EM})$ over the allocations $S$, which take one step in each epoch.  At two values $\sigma$ and $\sigma'$, the two averages have the same weights, and the means of their components differ by $|1/\sigma-1/\sigma'|$ in $E$ coordinates.  By the triangle inequality and the convexity of total variation, the exact total variations at $\sigma$ and at $\sigma'$ differ by at most $2\Phi(\sqrt E\,|1/\sigma-1/\sigma'|/2)-1$, so the exact total variation is continuous in $\sigma$.  As in the proof of \cref{thm:fixed-rate}, the sum of all coordinates, divided by $\sqrt{EM}$, is $N(0,1)$ against $N(\sqrt{E/M}/\sigma,1)$, so the exact total variation is at least $2\Phi(\sqrt{E/M}/(2\sigma))-1$, which tends to $1$ as $\sigma\to0$.  By \cref{thm:finite-rule} it tends to $0$ as $\sigma\to\infty$. It also does not grow with $\sigma$.  Scaled by $\sigma$, the pair adds $N(0,\sigma^2)$ to every step and, with the record, shifts one step of each epoch by $1$.  So the pair at $\sigma'>\sigma$ arises from the pair at $\sigma$ by adding independent $N(0,\sigma'^2-\sigma^2)$ to every step under both distributions, which cannot increase their total variation.  Hence, for $0<\lambda<1$ the $\sigma$ that meet the target form an interval $[\sigma_\lambda,\infty)$ with $\sigma_\lambda>0$, and the exact $\sigma$ is $\sigma_\lambda$.

\textbf{One epoch.} We compute $\varphi_Y$ by Gauss--Hermite quadrature, since $Y$ is a function of one standard Gaussian.  The integrand of \eqref{eq:exact-tv-cf} changes most for $u$ of order $(M(e^{1/\sigma^2}-1))^{-1/2}$.  There the factor $e^{-\mathrm iu}\varphi_Y(u)$ is close to one, and subtracting it from one loses precision.  So we compute the difference directly,
\[
 z(u)=e^{-\mathrm iu}\varphi_Y(u)-1=\E\bigl\{e^{\mathrm iu(Y-1)}-1-\mathrm iu(Y-1)\bigr\},
\]
using a power series when $u(Y-1)$ is small.  We then form $\varphi_Z=\exp\{M\ln(1+z)\}$, which equals $(1+z)^M$ for the principal logarithm because $M$ is an integer.  When $|\Re z|+|\Im z|<\frac12$ we compute the logarithm as
\[
 \Re\ln(1+z)=\tfrac12\operatorname{log1p}\bigl(2\Re z+|z|^2\bigr),\qquad
 \Im\ln(1+z)=\operatorname{atan2}(\Im z,1+\Re z),
\]
and otherwise we compute the real part as $\ln|1+z|$ and the imaginary part as above.  Here $\operatorname{log1p}(x)=\ln(1+x)$ and $\operatorname{atan2}(y,x)$ is the angle in $(-\pi,\pi]$ of the point $(x,y)$.  The function $\operatorname{log1p}$ is accurate for small $x$, where taking the logarithm of the rounded value of $1+x$ loses most digits of $x$, and $\operatorname{atan2}$ gives the argument of $1+z$ in every quadrant.
We integrate up to $u_{\max}=\max\{13/\sqrt{M(e^{1/\sigma^2}-1)},40,4C_t\}$.  The integral over $u>u_{\max}$ is close to $1/(\pi Mu_{\max})$, and we add this term exactly.  By part~(ii) of \cref{lem:truncation}, the error of this tail term is smaller than $10^{-12}$ in every reported cell.  The quadrature error has no certified bound, and the checks below test it.

The values pass four checks.  At $M=1$, direct integration of $\frac12\E|Y-1|$ agrees with $2\Phi(t/2)-1$.  The values agree with Monte Carlo at five settings including $M=16{,}384$. They are stable across quadrature orders.  They lie between the bounds
\[
 \frac{(e^{1/\sigma^2}-1)^{3/2}}{2\{\mu_4+3(M-1)(e^{1/\sigma^2}-1)^2\}^{1/2}}\le d_{\TV}(P_M,Q_M)\le\frac12\sqrt{\frac{e^{1/\sigma^2}-1}M},
 \qquad \mu_4=\E(Y-1)^4 .
\]
The upper bound is the Cauchy--Schwarz inequality $\E|Z|\le(\E Z^2)^{1/2}$ (\cref{fact:holder}).  The lower bound is H\"older's inequality (\cref{fact:holder}) with exponents $\frac32$ and $3$ applied to $Z^2=|Z|^{2/3}|Z|^{4/3}$,
\[
 \E Z^2\le(\E|Z|)^{2/3}(\E Z^4)^{1/3},\qquad\text{so}\qquad
 \E|Z|\ge\frac{(\E Z^2)^{3/2}}{(\E Z^4)^{1/2}} ,
\]
with $\E Z^4=M^4\E(L_M-1)^4=M\mu_4+3M(M-1)(e^{1/\sigma^2}-1)^2$ from the proof of \cref{thm:fixed-rate}, where $\mu_4=\E(Y-1)^4=e^{6t^2}-4e^{3t^2}+6e^{t^2}-3$. Dividing the lower bound on $\E|Z|$ by $2M$, with $\E Z^2=M(e^{1/\sigma^2}-1)$, gives the lower bound on $d_{\TV}(P_M,Q_M)=\E|Z|/(2M)$ displayed above.

\textbf{Several epochs.} We invert the characteristic function of the standardized variable $(L_M-1)/\sqrt{\Var L_M}$ on a grid to obtain the density of $L_M$.  To damp the oscillation of $\varphi_Y$, some of the grids below evaluate it on a contour shifted into the complex plane.  For $u\ge0$ and $0<c\le\pi/(2t)$,
\[
 \varphi_Y(u)=\int_{-\infty}^{\infty}\exp\bigl(\mathrm iue^{t(x+\mathrm ic)-t^2/2}\bigr)\,
 \frac{e^{-(x+\mathrm ic)^2/2}}{\sqrt{2\pi}}\dd x ,
\]
and we take $c=\min\{1,\pi/(2t)\}$.  This shift changes the quadrature integrand without changing $\varphi_Y(u)$.  On the shifted line the first factor has modulus $\exp(-ue^{tx-t^2/2}\sin(tc))\le1$, which decays as $x\to\infty$ for $u>0$, so the integrand no longer oscillates with modulus one.

To justify the shift, note that the integrand, as a function of $\zeta=x+\mathrm iy$, is analytic in the whole plane, so its integral around the rectangle with corners $\pm R$ and $\pm R+\mathrm ic$ is zero by Cauchy's theorem (\cref{fact:cauchy}).  On the two vertical sides $0\le ty\le\pi/2$, so $|\exp(\mathrm iue^{t\zeta-t^2/2})|=\exp(-ue^{tx-t^2/2}\sin(ty))\le1$, and $|e^{-\zeta^2/2}|=e^{(y^2-x^2)/2}\le e^{(c^2-R^2)/2}$.  So each of these sides contributes at most $c\,e^{(c^2-R^2)/2}/\sqrt{2\pi}$, which tends to zero as $R\to\infty$.  The same argument applies to $z(u)$ above, whose integrand is also analytic.  On the vertical sides its first factor $e^{\mathrm iu(Y-1)}-1-\mathrm iu(Y-1)$ has modulus at most $2+|u|(1+e^{tR})$, since $|e^{\mathrm iuY}|\le1$ there, and $e^{tR}e^{-R^2/2}$ still tends to zero.

To combine the epochs, we use $\ln\Lambda=\sum_{e=1}^E\ln L_{M,e}$.  The terms are independent, so the density of their sum is the convolution of their one-epoch densities.  We tabulate the density of $\ln L_M$ on a grid with spacing $h$ and multiply it by $h$, which gives approximate probabilities of the grid cells.  We shift this vector cyclically so that $s=0$ comes first, raise its discrete Fourier transform to the power $E$, transform back and undo the shift.  By the discrete convolution theorem this gives the periodic discrete convolution of $E$ copies, and dividing it by $h$ gives an approximation to the density of $\ln\Lambda$, as in the accountant of \citet[Eq.~(4.1) and Section~5.2]{KoskelaEtAl2020}.  Integrating $(1-e^s)_+$ against it approximates $\E(1-\Lambda)_+=d_{\TV}$.  From the same one-epoch density we also evaluate \eqref{eq:exact-tv-E} directly.

\textbf{Grids.} Let $v=(\Var L_M)^{1/2}=\{(e^{1/\sigma^2}-1)/M\}^{1/2}$ and $V=(L_M-1)/v$.  We invert the characteristic function $\varphi_V$ of $V$ at the angular frequencies up to a cutoff, on equally spaced points with spacing $1/512$, and interpolate the density $f_V$ of $V$ by a cubic spline.  The density of $s=\ln L_M$ is then $e^sf_V((e^s-1)/v)/v$, which we tabulate on an equally spaced grid that covers one period $\max\{8,\,24\sqrt E\,v+2Ev^2\}$, starts at minus half of it and has an even number of points, so that $0$ is a grid point.  Negative values are set to zero, and the total mass is not renormalized.  We integrate $(1-e^s)_+$ against the computed density of $\ln\Lambda$ by the sum over the grid points times the spacing $h$, and add $h^2f(0)/12$, where $f(0)$ is the computed density at $0$.  The added term corrects the sum for the kink of $(1-e^s)_+$ at $s=0$.  We use a value only if the removed negative mass is at most $10^{-8}$, the total mass differs from one by at most $10^{-7}$, and $|\varphi_V|$ is at most $10^{-10}$ at the cutoff.  The cells of \cref{tab:exact27} in Appendix~\ref{app:numerics} and the settings $\sigma=1.737$, $5.039$ and $1.561$ of \cref{tab:attacks} use $1{,}001$ Hermite nodes on the shifted contour, the cutoff $48$, $262{,}144$ points for $V$ on $[-32,480)$ and $262{,}144$ points for $s$. Every exact $\sigma$, in \cref{tab:calib-trial} and in \cref{tab:attacks}, uses $4{,}001$ Hermite nodes on the real line, the cutoff $24$, $131{,}072$ points for $V$ on $[-32,224)$ and $131{,}072$ points for $s$.

In every cell of \cref{tab:exact27} in Appendix~\ref{app:numerics} and at these three settings, the two computations and a second grid, with $8{,}001$ Hermite nodes on the real line and the cutoff $24$, agree to within $1.1\times10^{-9}$.  At $M=1$ the formula for $\E e^{qt\bar X}$ in the proof of \cref{lem:truncation} gives $F_1(q)=\E e^{q(tX-t^2/2)}=e^{q(q-1)t^2/2}$.  Also $d_{\TV}=2\Phi(t\sqrt E/2)-1$, because the pair is then $N(0,I_E)$ against $N(t\mathbf 1,I_E)$, whose likelihood ratio depends only on $\sum_eX_e/\sqrt E$, a standard Gaussian that the target shifts by $t\sqrt E$.  With this $F_1$, formula \eqref{eq:exact-tv-E} reproduces $2\Phi(t\sqrt E/2)-1$ to $2.3\times10^{-16}$ over $24$ settings.  This checks the formula and its quadrature, not the density inversion.

All $27$ values of \cref{tab:exact27} in Appendix~\ref{app:numerics} lie between the largest lower bound and the smallest upper bound on the total variation that the privacy loss distribution accountant of \citet{FeldmanShenfeld2026} gives in the add and the remove direction.  Both directions bound the same number, because total variation is symmetric.  They also agree with independent Monte Carlo estimates from $2\times10^7$ samples.  Each estimate averages $(1-\Lambda)_+$ over independent draws under $P_M^{\otimes E}$, since $\E(1-\Lambda)_+=d_{\TV}$, and its standard error is the standard deviation of the draws divided by the square root of their number.  The squared standardized differences sum to $28.8$ over the $27$ cells.  If the computed values are correct, this sum behaves like a chi-square variable with $27$ degrees of freedom, whose mean is $27$, and a value of $28.8$ or more has probability $0.37$.  At three settings of \cref{tab:attacks}, $\sigma=1.737$, $5.039$ and $1.561$, estimates from $10^8$ samples agree with the computed values, each within one standard error.  \Cref{tab:exact27} in Appendix~\ref{app:numerics} lists the values to five decimal places.

\subsection{A certificate that the exact \texorpdfstring{$\sigma$}{sigma} meets the target}
\label{app:certificate}

The exact $\sigma$ of \cref{tab:calib-trial} comes from a numerical evaluation of \cref{thm:exact-tv}.  A second, finite computation certifies that each of the nine values meets the target $\lambda=0.10$.  It uses no quadrature, and it runs in ball arithmetic \citep{Johansson2017Arb}, which encloses every computed number in an interval that contains the exact value.  Like any certificate computed by software, it trusts that software, here FLINT $3.6.0$ \citep{FLINT} and python-flint $0.9.0$.

\begin{lemma}[Mean-preserving rounding lowers a separately concave expectation]
\label{lem:spread}
Let $Y_1,\dots,Y_n$ be nonnegative random variables, and let $g$ be a Borel measurable function from $[0,\infty)^n$ to $[0,1]$ that is concave in each coordinate when the others are fixed.  For each $i$ let $a_i$ and $b_i$ be Borel measurable functions with $0\le a_i(y)\le y\le b_i(y)$.  Given $Y_1,\dots,Y_n$, let $R_1,\dots,R_n$ be independent, with $R_i=Y_i$ if $a_i(Y_i)=b_i(Y_i)$, and otherwise $R_i=a_i(Y_i)$ with probability $(b_i(Y_i)-Y_i)/(b_i(Y_i)-a_i(Y_i))$ and $R_i=b_i(Y_i)$ with the remaining probability.  Then $\E g(R_1,\dots,R_n)\le\E g(Y_1,\dots,Y_n)$.
\end{lemma}

\begin{proof}
For a concave function $f$ and $a\le y\le b$ with $a<b$, concavity gives $\frac{b-y}{b-a}f(a)+\frac{y-a}{b-a}f(b)\le f(y)$.  For $i=1,\dots,n$, given $Y_1,\dots,Y_n$ and $R_1,\dots,R_{i-1}$, the conditional expectation of $g(R_1,\dots,R_i,Y_{i+1},\dots,Y_n)$ is this combination for the concave function $y\mapsto g(R_1,\dots,R_{i-1},y,Y_{i+1},\dots,Y_n)$, so it is at most $g(R_1,\dots,R_{i-1},Y_i,\dots,Y_n)$.  Taking expectations and chaining the $n$ inequalities gives the lemma.
\end{proof}
The same argument, applied coordinate by coordinate to a convex function of independent random variables on a fixed rectangle, gives the upper bound of \citet[Section~5]{Madansky1959}.

Write $Y_{e,j}$ for the $Y_j$ of \eqref{eq:one-epoch-lr} in epoch $e$, so that $L_{M,e}=M^{-1}\sum_{j=1}^MY_{e,j}$, $\Lambda=\prod_{e=1}^EL_{M,e}$ and $d_{\TV}=1-\E\min\{1,\Lambda\}$ (Appendix~\ref{app:exact-tv}).  With all but one $y=Y_{e,j}$ fixed, $\min\{1,\Lambda\}=\min\{1,c(y+d)\}$ for some $c,d\ge0$, which is concave in $y$.  With all but one $\ell=L_{M,e}$ fixed, it is $\min\{1,c\ell\}$, which is concave in $\ell$.  The computation lowers $\E\min\{1,\Lambda\}$ in four steps.
\begin{enumerate}[label=(\roman*)]
\item Fix a grid spacing $h>0$.  Each $Y_{e,j}$ in $[(k-1)h,kh]$ is rounded to $(k-1)h$ or $kh$ as in \cref{lem:spread}, and its mass above a cap, rounded up to a multiple of $h$, is deleted.  The masses at the grid points have closed forms, since $\Pr(Y\le y)=\Phi((\ln y+t^2/2)/t)$ and $\E[Y\one\{Y\le y\}]=\Phi((\ln y-t^2/2)/t)$ with $t=1/\sigma$.  The computation encloses $\Phi$ with the complementary error function of FLINT \citep{Johansson2019Hypergeometric}.
\item The rounded values of one epoch sum to a multiple of $h$, and the masses of the sum are the coefficients of $(\sum_kp_kz^k)^M$ after deleting sums above $h\lceil8M/h\rceil$, where $p_k$ is the mass that step (i) puts at $kh$.
\item Each epoch value after steps (i) and (ii), the rounded sum of the epoch divided by $M$, is rounded to the endpoints of its interval on the grid $e^{k\eta}$ for a spacing $\eta>0$, again as in \cref{lem:spread}, and values outside $[e^{k_-\eta},e^{k_+\eta}]$ are deleted, where $k_-=\lfloor-1/\eta\rfloor$ and $k_+=\lceil2/\eta\rceil$.
\item The product over the epochs adds the exponents $k$, so the masses of the product are the coefficients $b_k$ of $(\sum_kr_kz^k)^E$, where $r_k$ is the mass that step (iii) puts at $e^{k\eta}$, and $\E\min\{1,\Lambda\}$ after the rounding is $\sum_kb_k\min\{1,e^{k\eta}\}$.
\end{enumerate}
Each deletion only lowers the sum, since $\min\{1,\Lambda\}\ge0$, and so does replacing every computed mass by the lower end of its ball, or by zero if that end is negative, rounded down to a multiple of $2^{-90}$.  The masses are then those of a measure of total mass at most one, which is never renormalized, and \cref{lem:spread} holds for such measures too, since its proof integrates an inequality that holds at every point.  So the lower end $O$ of the final sum satisfies $O\le\E\min\{1,\Lambda\}$, and $d_{\TV}\le1-O$.  The accountant of \citet[Section~4]{FeldmanShenfeld2026} also discretizes one step and convolves the steps of an epoch by repeated squaring.  For a pair $(P,Q)$ with likelihood ratio $L=\dd Q/\dd P$, rounding $L$ to the endpoints of its interval on a grid, as in \cref{lem:spread}, keeps $H_{e^\eps}(Q\|P)=\E_P(L-e^\eps)_+$ at every grid point $e^\eps$ and makes it linear in $e^\eps$ between grid points, as the pessimistic discretization of \citet[Section~4]{DoroshenkoEtAl2022} does.

\Cref{tab:certificate} gives the grids and $1-O$ at the exact $\sigma$ of each cell.  The bound $1-O$ is at most $0.0999999977$, so every exact $\sigma$ meets the target.  One run on these grids took $36$ to $598$ seconds on one CPU core at $100$ bits of precision.  The grids were set by hand after a first run at the same $\sigma$ on a coarser grid, which took $23$ to $433$ seconds and certified the bound in only one of the nine cells.  The two runs together took $70$ to $670$ seconds, the certification time of \cref{tab:calib-trial}.  The computation does not certify that $10^{-6}$ less noise misses the target, which rests on the numerical evaluation of Appendix~\ref{app:exact-numerics}.

\begin{table}[t]
\caption{The certificate of Appendix~\ref{app:certificate} at the exact $\sigma$ of each cell of \cref{tab:calib-trial}.  The grid spacing $h$ and the cap on $Y$ are those of step (i), and $\eta$ is the spacing of the logarithms of the epoch grid of step (iii).  The bound is $1-O$, rounded up, and every bound is below $\lambda=0.10$.  \emph{Certify} is the time of one run on these grids, and \emph{first grid} is the time of an earlier run at the same $\sigma$ on a coarser grid, which certified the bound only at $M=500$, $E=10$.  Each run used one CPU core and $100$ bits of precision.}
\label{tab:certificate}
\centering
\small
\setlength{\tabcolsep}{4pt}
\begin{tabular}{@{}rrrrrrrrr@{}}
\toprule
$M$ & $E$ & exact $\sigma$ & $h$ & $\eta$ & cap & bound on $d_{\TV}$ & certify (s) & first grid (s) \\
\midrule
$100$ & $10$ & $1.426215$ & $0.001$ & $5\times10^{-5}$ & $512$ & $0.0999999543$ & $47$ & $23$ \\
$100$ & $20$ & $1.907073$ & $0.0005$ & $3\times10^{-5}$ & $128$ & $0.0999999580$ & $36$ & $81$ \\
$100$ & $40$ & $2.610518$ & $0.0003$ & $2\times10^{-5}$ & $32$ & $0.0999999743$ & $43$ & $345$ \\
$250$ & $10$ & $1.024495$ & $0.0003$ & $1\times10^{-5}$ & $2000$ & $0.0999999977$ & $598$ & $71$ \\
$250$ & $20$ & $1.309428$ & $0.001$ & $3\times10^{-5}$ & $1024$ & $0.0999999624$ & $119$ & $108$ \\
$250$ & $40$ & $1.732809$ & $0.0004$ & $2\times10^{-5}$ & $256$ & $0.0999999826$ & $95$ & $316$ \\
$500$ & $10$ & $0.834432$ & $0.002$ & $5\times10^{-5}$ & $4000$ & $0.0999999316$ & $181$ & $170$ \\
$500$ & $20$ & $1.025964$ & $0.0015$ & $4\times10^{-5}$ & $4000$ & $0.0999999445$ & $252$ & $231$ \\
$500$ & $40$ & $1.310146$ & $0.0006$ & $2\times10^{-5}$ & $1024$ & $0.0999999769$ & $212$ & $433$ \\
\bottomrule
\end{tabular}
\end{table}

\subsection{The ratio of the bound to the exact value as \texorpdfstring{$E/M\to0$}{E/M to 0}}
\label{app:small-div}

\begin{proposition}
\label{prop:small-div}
Fix $\sigma$ and let $M\to\infty$ with positive integers $E_M$ such that $E_M/M\to0$.  Put $s_M^2=E_M(e^{1/\sigma^2}-1)/M$.  Then
\[
 d_{\TV}\bigl(P_M^{\otimes E_M},Q_M^{\otimes E_M}\bigr)=\frac{s_M}{\sqrt{2\pi}}\,\bigl(1+o(1)\bigr),
 \qquad
 U(c_{E_M})=\frac{s_M}2\,\bigl(1+o(1)\bigr),
\]
so the bound divided by the exact total variation tends to $\sqrt{\pi/2}$.
\end{proposition}

\begin{proof}
Write $E=E_M$, $s=s_M$ and $b=e^{1/\sigma^2}-1$.  Then $s\to0$, and $n=EM\to\infty$ because $E\ge1$.  Write $V_e=L_{M,e}-1$ for the $E$ independent epochs, so that $\Lambda=\prod_e(1+V_e)$.  Under $P_M$ we have $\E V_e=0$ and $\E V_e^2=\Var(L_M)=b/M$.  The linear part of $\Lambda-1$ is $\ell=\sum_eV_e$, with $\E\ell^2=Eb/M=s^2$.  We show that $\Lambda-1$ differs from $\ell$ by a remainder that is small in mean square, and that $\ell/s$ is a normalized sum of independent copies of one variable.

\emph{The remainder.}  For each $e$, by independence and the product rule (\cref{fact:product}),
\[
 \E[(\Lambda-1)V_e]=\E[\Lambda V_e]=\E[(1+V_e)V_e]\prod_{e'\ne e}\E(1+V_{e'})=\E V_e^2=b/M ,
\]
so $\E[(\Lambda-1)\ell]=Eb/M=s^2$.  With $\E(\Lambda-1)^2=c_E$ from \eqref{eq:tensorize},
\[
 \E(\Lambda-1-\ell)^2=\E(\Lambda-1)^2-2\E[(\Lambda-1)\ell]+\E\ell^2=c_E-s^2 .
\]
Since $1+x\le e^x$, we have $c_E\le e^{s^2}-1$.  For $0\le y\le1$,
\[
 e^y-1-y=\sum_{k\ge2}\frac{y^k}{k!}\le y^2\sum_{k\ge2}\frac1{k!}=(e-2)\,y^2\le y^2 .
\]
With $y=s^2$ this gives $c_E-s^2\le s^4$ once $s\le1$.  Since $\bigl||x|-|y|\bigr|\le|x-y|$, the Cauchy--Schwarz inequality (\cref{fact:holder}) then gives
\[
 \bigl|\E|\Lambda-1|-\E|\ell|\bigr|\le\E|\Lambda-1-\ell|\le(c_E-s^2)^{1/2}\le s^2 .
\]

\emph{The linear part.}  Since $\ell=M^{-1}\sum_{e=1}^E\sum_{j=1}^M(Y_{e,j}-1)$, it is $1/M$ times a sum of $n=EM$ independent copies of $Y-1$, whose distribution does not depend on $M$.  Let $S_n$ be the sum of $n$ independent copies of $Y-1$.  Since $s^2=nb/M^2$, the variable $\ell/s$ has the distribution of $S_n/\sqrt{nb}$.  The central limit theorem (\cref{fact:clt}) gives $S_n/\sqrt{nb}\Rightarrow N(0,1)$ as $n\to\infty$.  These variables have second moment $1$, so \cref{fact:integration-limit} with $g(x)=1+x^2$ and $h(x)=|x|$ gives $\E|S_n|/\sqrt{nb}\to\E|N(0,1)|=2\int_0^\infty y\,e^{-y^2/2}\dd y/\sqrt{2\pi}=\sqrt{2/\pi}$.  The conditions of \cref{fact:integration-limit} hold because $g$ and $h$ are continuous, $g>0$, $h(x)/g(x)\to0$ as $|x|\to\infty$, and the means of $g$ equal $2$.  Since $n=EM\to\infty$, also $\E|\ell|/s\to\sqrt{2/\pi}$.  By the remainder bound, $\E|\Lambda-1|/s$ differs from $\E|\ell|/s$ by at most $s\to0$.  So
\[
 \frac{\E|\Lambda-1|}{s}\to\sqrt{2/\pi},
 \qquad
 d_{\TV}=\frac12\E|\Lambda-1|=\frac{s}{\sqrt{2\pi}}\bigl(1+o(1)\bigr).
\]

\emph{The bound.}  Since $c_E\le e^{s^2}-1\to0$, eventually $c_E\le1$ and $U(c_E)=\frac12\sqrt{c_E}=\frac s2\sqrt{c_E/s^2}=\frac s2(1+o(1))$, because $0\le c_E-s^2\le s^4$ by the remainder step, so $1\le c_E/s^2\le1+s^2$.  The ratio $U(c_E)/d_{\TV}$ therefore tends to $\frac12\sqrt{2\pi}=\sqrt{\pi/2}$.
\end{proof}

For intuition, $\Lambda-1$ is close to $N(0,s^2)$, and half the mean absolute value of $N(0,s^2)$ is $s/\sqrt{2\pi}$.  This is also the leading term of $2\Phi(s/2)-1=s/\sqrt{2\pi}+O(s^3)$, the total variation between $N(0,1)$ and $N(s,1)$.  The bound $s/2$ is larger by the factor $\sqrt{\pi/2}$.  For $E=1$ the proposition gives the limit of the single-epoch ratios computed after \cref{thm:exact-tv}.

\section{One Chi-Square Value Does Not Determine the Distance}
\label{app:chi2-underdetermines}

Fix $c>0$, take $E=1$, and choose $t_M^2=\ln(1+cM)$, so $(e^{t_M^2}-1)/M=c$ for every $M$.  Write $Y_j=e^{t_MX_j-t_M^2/2}$ and $L_M=M^{-1}\sum_jY_j$.  The chi-square divergence $c=\Var L_M$ of the pair is carried by rare large values of $Y_j$.  We show that these values carry a vanishing part of the mean of $L_M$, and that without them $L_M$ concentrates at its mean, so that $\E|L_M-1|\to0$.

Fix $a\in(1,2)$ and write $L_M=R_M+D_M$, where $R_M=M^{-1}\sum_jY_j\one\{X_j\le at_M\}$ retains the smaller values and $D_M=M^{-1}\sum_jY_j\one\{X_j>at_M\}$ collects the rare large ones.  We will show that $\Var R_M\to0$ and $\E D_M\to0$.

Completing the square gives $\E[e^{bX-b^2/2}g(X)]=\E g(X+b)$ for $X\sim N(0,1)$.  With $b=t_M$, this gives $\E D_M=\bar\Phi((a-1)t_M)\to0$, where $\bar\Phi=1-\Phi$.  With $b=2t_M$ and $Y_j^2=e^{t_M^2}e^{2t_MX_j-2t_M^2}$, independence gives $\Var R_M\le M^{-1}\E[Y_j^2\one\{X_j\le at_M\}]=M^{-1}e^{t_M^2}\Phi((a-2)t_M)\to0$, since $M^{-1}e^{t_M^2}=c+1/M$.

Since $\E R_M+\E D_M=\E L_M=1$ and $|D_M-\E D_M|\le D_M+\E D_M$, the Cauchy--Schwarz inequality (\cref{fact:holder}) gives
\[
 \E|L_M-1|\le\E|R_M-\E R_M|+2\E D_M\le(\Var R_M)^{1/2}+2\E D_M\to0 .
\]
So $L_M\to1$ in $L^1$ and $d_{\TV}(P_M,Q_M)=\frac12\E|L_M-1|\to0$ while $\chi^2$ stays at $c$.  At $M=1$ the same $c$ gives $d_{\TV}=2\Phi(\sqrt{\ln(1+c)}/2)-1>0$.  So no function of $c$ alone is within a fixed factor of the distance for all $M$ and $\sigma$ with $\chi^2=c$.  In particular the ratio $U(c)/d_{\TV}(P_M,Q_M)$ is unbounded along this sequence.  Numerically, at $c=0.04$ the bound is $0.100$ at every $M$, while the exact value is $0.0787$ at $M=7$ ($\sigma=2.013$) and $0.0586$ at $M=5{,}360$ ($\sigma=0.431$).

\section{Experimental Details}
\label{app:experiments}

\textbf{Training runs (\cref{sec:calib-trial,sec:attacks}).} ResNet-18 \citep{HeEtAl2016} with group normalization \citep{WuHe2018}, sixteen groups and no batch normalization, on CIFAR-10 \citep{Krizhevsky2009}.  The padded size is $N_0=45{,}000$, the clipping norm $C=0.3$, and $B_0=N_0/M$.  The loss is the cross-entropy of each record, and the optimizer is SGD with momentum $0.9$ and no weight decay.  The learning rate follows $\mathrm{lr}(B_0)=0.05\,(B_0/180)^{0.5}$, where $180$ is $B_0$ at $M=250$.  We chose this rule, $C$ and the weight decay before any trial on a public split of $5{,}000$ training images, $4{,}000$ to train and $1{,}000$ to score, at $E=20$ and $\lambda=0.20$, among $72$ candidates with base rate $0.05$, $0.1$, $0.25$ or $0.5$, exponent $0$, $0.5$ or $1$ in $\mathrm{lr}(B_0)=\text{base}\cdot(B_0/180)^{\text{exponent}}$, $C\in\{0.1,0.3,1\}$ and weight decay $0$ or $5\times10^{-4}$.  The rule gives the initial rate, which is multiplied by $\frac12(1+\cos(\pi k/(EM)))$ at update $k$ of the $EM$ updates.  Random crops and flips are applied.  Noise and allocations come from deterministic pseudorandom streams in single precision, so the runs implement the Gaussian mechanism of \cref{alg:dpsgd} up to pseudorandom number generation and floating-point rounding.  In the comparison of \cref{tab:calib-trial}, the three runs of a cell and seed of every model share their initialization, their allocation and the Gaussian draws behind their noise, so only $\sigma$ differs.  The three runs of a cell and seed of an image model ran one after another on one GPU, an RTX A5000 for CIFAR-10 and an RTX A4000 for CIFAR-100.  The accountant $\sigma$ of each cell is the smallest $\sigma$ the accountant certified in a binary search that stops at relative width $2\times10^{-3}$.  Its upper bound on the total variation, the larger of the add and remove directions, is at most $\lambda$.  The accountant $\sigma$ is $0.24\%$ to $0.52\%$ above the exact $\sigma$ of \cref{thm:exact-tv}.

\textbf{The exact $\sigma$ (\cref{sec:calib-trial}).} As \citet{BalleWang2018} do for the Gaussian mechanism, we solve the exact condition numerically.  We find the exact $\sigma$ of a cell by Brent's method on the total variation of \cref{thm:exact-tv}, computed as in Appendix~\ref{app:exact-numerics}, and round it up to a multiple of $10^{-6}$.  At the nine exact values the total variation is $0.099999917$ to $0.099999998$, and $0.100000008$ to $0.100000146$ at the value $10^{-6}$ lower.  We timed this search on one core of the machine that ran the accountant's search for the same cell, while the accountant's search set no limit on threads.  To time the accountant without a search, we also ran it once at its trained $\sigma$ on one core, in the add direction, which suffices because at $\eps=0$ either direction gives the total variation.  We tried the numerical accuracies $0.003$, $0.0025$ and $0.002$ in turn and stopped at the first that certified the $\sigma$.  \Cref{tab:calib-timing} gives the time of the exact search and of this run in each cell.  For one run the numerical accuracy is the library's loss discretization, and for a search it is the width that the accountant's interval on $\eps$ must reach.  The search that gave the trained $\sigma$ uses $0.02$ and repeats an undecided $\sigma$ at $0.005$.  \Cref{tab:calib-eps} gives the $\eps$ of each $\sigma$ at $\delta=10^{-5}$.

\textbf{The accountant's fast Fourier route.} PLD\_accounting $0.5.0$ convolves the steps of an epoch in two ways, which it can also combine \citep[Appendix~F]{FeldmanShenfeld2026}.  Its default, used above, convolves directly on a grid evenly spaced in the privacy loss.  The other uses the fast Fourier transform (FFT) on a grid evenly spaced in the exponential of the privacy loss.  Its bounds are valid, but \citet[Appendix~F]{FeldmanShenfeld2026} give no guarantee on how close they are.  The library caps the grid of each step of an epoch at about $1/M$ of a fixed budget of points, and we set this budget to $4\times10^8$, $400$ times its default.  We searched with this route as in the coarse search, at loss discretization $10^{-5}$ and tail truncation $10^{-8}$, on the machine that ran the accountant's search for the same cell, with one thread.  The search starts at the closed form, steps down by the factor $0.75$ to the first $\sigma$ that the route does not certify, and then bisects and recertifies.  At each $E$ the noise of its $\sigma$ above the exact one grew with $M$, from at most $0.28\%$ at $M=100$ to $1.94\%$ to $4.07\%$ at $M=500$ (\cref{tab:calib-timing}).  With a budget of $10^8$ points it grew in the same way, from at most $0.28\%$ to $6.54\%$ to $11.70\%$, so a smaller grid per step costs noise.  Run once at the accountant's trained $\sigma$ right after its search, the route took $26$ to $60$ seconds and certified that $\sigma$ in four of the nine cells.  \Cref{tab:accountant-regimes} extends the comparison to $M$ up to $10^5$ and $E$ up to $100$.  There the fast Fourier route certified $1.005$ times the exact $\sigma$ only at $M=100$, and with $1{,}000$ or more steps per epoch it did not certify even $1.02$ times the exact $\sigma$.  One run of the default route in both directions certified $1.02$ times the exact $\sigma$ in seven of the nine settings, in $7$ seconds to $2.0$ hours, and did not finish within $4$ hours at $E=100$ with $M=1{,}000$ and $M=10{,}000$.  The exact $\sigma$ took $6.5$ to $42.9$ seconds.  The growth with $M$ of the noise that the fast Fourier route certifies, at a fixed target and grid budget, fits the explanation of \citet[Appendix~F]{FeldmanShenfeld2026}, that the route works on the exponential of the privacy loss, whose range is much wider.  A larger $M$ needs a smaller $\sigma$, which widens this range further, and it also leaves fewer grid points for each step.  Our searches with the accountant certify a $\sigma$ only when the upper bounds of the add and the remove direction are both at most $\lambda$, although at $\eps=0$ either bound alone bounds the total variation.  Accepting either bound shortened the fast Fourier search by $7\%$ to $27\%$ and the coarse search by $21\%$ to $37\%$, gave the same $\sigma$ on the fast Fourier route at $M\le250$, and lowered the noise by $0.11$ to $0.24$ percentage points on that route at $M=500$ and by $0.11$ to $0.37$ percentage points in the coarse search.  The searches stop when the bracket around $\sigma$ is narrower than $2\times10^{-3}$ of its upper end.  At $M=250$ and $E=10$, where the fast Fourier search certified the least noise among the searches faster than the exact $\sigma$, accepting either bound and stopping at $10^{-3}$ and at $10^{-4}$ certified $0.65\%$ and $0.62\%$ more noise than the exact $\sigma$ in $326$ and $456$ seconds, against $680$ seconds to find and certify the exact $\sigma$.  The route did not certify, in either direction, the last $\sigma$ that the finer search rejected, $0.6155\%$ above the exact one.

\begin{table*}[t]
\caption{Time to find $\sigma$ in each cell of \cref{tab:calib-trial}.  The first group ran on the machine that ran the accountant's search for that cell.  \emph{Exact} is the root search on \cref{thm:exact-tv}, \emph{accountant} is the search that gave the trained $\sigma$, \emph{coarse} is the same search at numerical accuracy $0.05$ without escalation, and \emph{FFT} is the search on the accountant's fast Fourier route.  \emph{Coarse} and \emph{FFT} are each followed by the noise of their $\sigma$ above the exact $\sigma$.  The second group is one run of the accountant with its default route in one direction at its trained $\sigma$, with no search, at the largest loss discretization that certifies that $\sigma$.  Every run except the accountant's search used one thread.  Appendix~\ref{app:experiments} describes the searches and defines the numerical accuracy and the loss discretization.}
\label{tab:calib-timing}
\centering
\small
\setlength{\tabcolsep}{2.3pt}
\begin{tabular}{@{}rrrrrrrrrr@{}}
\toprule
& & \multicolumn{6}{c}{search} & \multicolumn{2}{c}{one run of the accountant} \\
\cmidrule(lr){3-8}\cmidrule(lr){9-10}
$M$ & $E$ & exact (s) & accountant (h) & coarse (min) & \begin{tabular}[b]{@{}c@{}}coarse\\above exact\end{tabular} & FFT (min) & \begin{tabular}[b]{@{}c@{}}FFT\\above exact\end{tabular} & time (min) & \begin{tabular}[b]{@{}c@{}}loss\\discretization\end{tabular} \\
\midrule
$100$ & $10$ & $9.8$ & $5.5$ & $4.5$ & $4.31\%$ & $4.9$ & $0.26\%$ & $6.5$ & $0.003$ \\
$100$ & $20$ & $9.4$ & $12.6$ & $9.5$ & $4.75\%$ & $6.5$ & $0.27\%$ & $19.6$ & $0.0025$ \\
$100$ & $40$ & $8.8$ & $27.5$ & $20.4$ & $5.04\%$ & $10.6$ & $0.28\%$ & $40.2$ & $0.0025$ \\
$250$ & $10$ & $10.1$ & $19.0$ & $18.3$ & $2.90\%$ & $6.5$ & $0.71\%$ & $36.1$ & $0.003$ \\
$250$ & $20$ & $10.4$ & $45.4$ & $42.4$ & $3.39\%$ & $8.1$ & $0.44\%$ & $131.7$ & $0.003$ \\
$250$ & $40$ & $9.5$ & $103.3$ & $93.0$ & $3.73\%$ & $10.9$ & $0.60\%$ & $180.4$ & $0.003$ \\
$500$ & $10$ & $9.6$ & $40.5$ & $38.9$ & $2.38\%$ & $5.8$ & $4.07\%$ & $106.2$ & $0.0025$ \\
$500$ & $20$ & $10.1$ & $105.3$ & $99.7$ & $2.90\%$ & $7.4$ & $2.90\%$ & $160.7$ & $0.003$ \\
$500$ & $40$ & $10.4$ & $248.9$ & $231.7$ & $3.36\%$ & $10.7$ & $1.94\%$ & $425.4$ & $0.003$ \\
\bottomrule
\end{tabular}
\end{table*}

\begin{table}[t]
\caption{The $\eps$ at $\delta=10^{-5}$ of the three values of $\sigma$ in each cell of \cref{tab:calib-trial}.  The closed-form $\sigma$ is the largest of the three and the exact $\sigma$ the smallest, so in each cell the closed form has the smallest $\eps$ and the exact $\sigma$ the largest.  Each entry is the upper end of the interval that the privacy loss distribution accountant of \citet{FeldmanShenfeld2026} gives at numerical accuracy $0.01$, in the larger of the add and the remove direction.}
\label{tab:calib-eps}
\centering
\small
\begin{tabular}{@{}rrccc@{}}
\toprule
$M$ & $E$ & closed & exact & accountant \\
\midrule
$100$ & $10$ & $0.727$ & $0.949$ & $0.942$ \\
$100$ & $20$ & $0.724$ & $0.940$ & $0.934$ \\
$100$ & $40$ & $0.723$ & $0.939$ & $0.933$ \\
$250$ & $10$ & $0.733$ & $0.970$ & $0.963$ \\
$250$ & $20$ & $0.725$ & $0.944$ & $0.938$ \\
$250$ & $40$ & $0.724$ & $0.939$ & $0.934$ \\
$500$ & $10$ & $0.745$ & $1.022$ & $1.015$ \\
$500$ & $20$ & $0.728$ & $0.952$ & $0.946$ \\
$500$ & $40$ & $0.724$ & $0.941$ & $0.935$ \\
\bottomrule
\end{tabular}
\end{table}

\begin{table}[t]
\caption{The accountant of \citet{FeldmanShenfeld2026} with its two routes against the exact $\sigma$ over larger settings, on one CPU core.  The fast Fourier column gives the smallest of $1.005$ and $1.02$ times the exact $\sigma$ and the closed form \eqref{eq:inverse-sigma} that the fast Fourier route certifies at grid budget $4\times10^8$, loss discretization $10^{-5}$ and tail truncation $10^{-8}$, and \emph{none} means that it certifies none of the three.  The default column times one run of the default route in both directions, the library's default, at $1.02$ times the exact $\sigma$, loss discretization $0.01$ and tail truncation $10^{-8}$, which certifies that $\sigma$ in every setting where it finished within $4$ hours.  At $E=1$ the exact $\sigma$ is found with \eqref{eq:exact-tv-cf} as in \cref{tab:methods} in Appendix~\ref{app:numerics}, and otherwise as in \cref{tab:calib-trial}.  Each setting ran on one machine with one thread.}
\label{tab:accountant-regimes}
\centering
\small
\begin{tabular}{@{}rrrrrlr@{}}
\toprule
$M$ & $E$ & $\lambda$ & exact $\sigma$ & exact (s) & fast Fourier certifies & default route (time) \\
\midrule
$100$ & $1$ & $0.10$ & $0.6770$ & $6.5$ & $1.005\times$ exact & $7$ s \\
$1{,}000$ & $1$ & $0.10$ & $0.4520$ & $42.9$ & none & $47$ s \\
$10{,}000$ & $1$ & $0.10$ & $0.3532$ & $32.8$ & none & $2$ min \\
$100{,}000$ & $1$ & $0.10$ & $0.2970$ & $38.0$ & none & $5$ min \\
$100$ & $100$ & $0.10$ & $4.0400$ & $9.9$ & $1.005\times$ exact & $24$ min \\
$1{,}000$ & $10$ & $0.10$ & $0.7048$ & $12.0$ & none & $31$ min \\
$10{,}000$ & $10$ & $0.10$ & $0.4821$ & $12.6$ & none & $2.0$ h \\
$1{,}000$ & $100$ & $0.10$ & $1.4289$ & $12.3$ & closed form \eqref{eq:inverse-sigma} & more than $4$ h \\
$10{,}000$ & $100$ & $0.10$ & $0.7085$ & $12.0$ & none & more than $4$ h \\
\bottomrule
\end{tabular}
\end{table}

\textbf{CIFAR-100 linear classifier (\cref{sec:calib-trial}).} A linear layer on the $2{,}048$ features of a ResNet-50 \citep{HeEtAl2016} pretrained on ImageNet \citep{DengEtAl2009}.  The network is frozen, and only the linear layer is trained. The images of CIFAR-100 \citep{Krizhevsky2009} are resized to $224\times224$ before the network.  A fixed random split of the $50{,}000$ training images gives the $N_0=35{,}000$ images of the training set and $15{,}000$ public images.  Test accuracy is on the $10{,}000$ test images.  The clipping norm is $C=0.3$, and the learning rate follows $\mathrm{lr}(B_0)=0.1\,(B_0/140)^{0.5}$, where $140$ is $B_0$ at $M=250$.  This rule was chosen before any trial, among twelve candidates with base rate $0.05$, $0.1$ or $0.25$, $C\in\{0.3,1\}$ and exponent $0$ or $0.5$, each trained at $E=20$ and $\lambda=0.20$ on $13{,}000$ of the public images and scored on the other $2{,}000$.  The loss, the optimizer, the cosine decay and the shared randomness of a cell are as for CIFAR-10, without augmentation.

\textbf{RoBERTa-base on SST-2 and AG News (\cref{sec:calib-trial}).} RoBERTa-base \citep{LiuEtAl2019RoBERTa} with a new classification head is fine-tuned in full, with all $124.6$ million parameters trained and each text cut or padded to $128$ tokens. Dropout is switched off everywhere in the model.  SST-2 \citep{SocherEtAl2013} is taken from GLUE \citep{WangEtAl2019GLUE}.  A fixed random split of its $67{,}349$ training examples gives the $N_0=45{,}000$ records of the training set, $20{,}000$ public records to train on when choosing the recipe and $1{,}000$ public records to score it. We report validation accuracy on the $872$ validation sentences.  AG News \citep{ZhangZhaoLeCun2015} has four classes.  The same kind of split of its $120{,}000$ training articles gives sets of the same sizes, and test accuracy is on its $7{,}600$ test articles.  For each dataset the recipe was chosen on its public records before any trial, at $E=20$ and $\lambda=0.20$, among six candidates, with base rate $10^{-3}$, $3\times10^{-3}$ or $10^{-2}$, clipping norm $C=1$ and exponent $0.5$ or $1$ in $\mathrm{lr}(B_0)=\text{base}\cdot(B_0/180)^{\text{exponent}}$.  For both datasets the rule chosen is $\mathrm{lr}(B_0)=0.003\,(B_0/180)$.  The score of a candidate is its mean top-1 accuracy on the $1{,}000$ scoring records over one run at each $M\in\{100,250,500\}$.  On AG News we stopped after $14$ of the $18$ runs, since no candidate with a missing run could reach the score of the best candidate even if each missing run matched the best finished run with the same $M$.  The loss, the optimizer and the cosine decay are as for CIFAR-10.  All runs used RTX A6000 GPUs.

\textbf{Other targets $\lambda$ (\cref{sec:calib-trial}).} \Cref{tab:calib-lambda} repeats the comparison of the closed form and the exact $\sigma$ at $\lambda\in\{0.01,0.05,0.20\}$, without the accountant and with the training setup above.  The exact $\sigma$ raises the test accuracy in each of the $99$ (cell, seed) pairs.  At these targets the mean gain is $1.08$ to $4.77$ points on CIFAR-100 and $0.94$ to $1.97$ points on CIFAR-10, largest at $\lambda=0.05$ for both models.
\begin{table}[t]
\caption{Test accuracy at the closed-form and the exact $\sigma$ for four targets $\lambda$.  The linear
classifier on CIFAR-100 uses the nine cells and ResNet-18 on CIFAR-10 the cells $(M,E)=(100,10)$ and
$(500,20)$, each with three seeds.  In each cell and seed the two runs share the seed, initialization and
allocation, so only $\sigma$ differs.  \emph{Noise} is how much larger the closed-form $\sigma$ is than the exact one, over the cells.  Test accuracies are means
over cells and seeds in percent, \emph{gain} is exact minus closed form in points, and \emph{positive} counts the (cell, seed) pairs
with a positive gain.  The rows at $\lambda=0.10$ use the runs of \cref{tab:calib-trial} in these cells.}
\label{tab:calib-lambda}
\centering
\small
\begin{tabular}{@{}llccccc@{}}
\toprule
model & $\lambda$ & noise (\%) & closed form & exact & gain & positive \\
\midrule
CIFAR-100 linear & $0.01$ & $24.9$ to $25.4$ & $6.12$ & $8.63$ & $+2.51$ & $27/27$ \\
 & $0.05$ & $19.9$ to $25.3$ & $41.83$ & $46.60$ & $+4.77$ & $27/27$ \\
 & $0.10$ & $14.9$ to $25.3$ & $53.58$ & $55.96$ & $+2.38$ & $27/27$ \\
 & $0.20$ & $13.3$ to $25.6$ & $58.58$ & $59.66$ & $+1.08$ & $27/27$ \\
\midrule
CIFAR-10 ResNet-18 & $0.01$ & $25.1$ to $25.3$ & $16.25$ & $17.98$ & $+1.73$ & $6/6$ \\
 & $0.05$ & $22.1$ to $24.2$ & $33.78$ & $35.76$ & $+1.97$ & $6/6$ \\
 & $0.10$ & $17.8$ to $21.8$ & $38.90$ & $39.91$ & $+1.01$ & $6/6$ \\
 & $0.20$ & $14.6$ to $19.4$ & $41.55$ & $42.48$ & $+0.94$ & $6/6$ \\
\bottomrule
\end{tabular}
\end{table}

\textbf{Comparison with published results (\cref{sec:calib-trial}).} At the exact $\sigma$ the RoBERTa-base runs have $\eps$ between $0.940$ and $0.952$ at $\delta=10^{-5}$.  On SST-2 they reach $84.97\%$ to $89.23\%$ validation accuracy, and \citet{LiEtAl2022} report $86.12\%$ validation accuracy at $\eps=3$ for full fine-tuning of RoBERTa-base with a standard classification head.  On AG News they reach $89.64\%$ to $90.58\%$ test accuracy from the title and description of $45{,}000$ private articles, and \citet{KurakinEtAl2023} report $89.8\%$ at $\eps=1$ and $90.5\%$ at $\eps=3$ for BERT fine-tuned with differential privacy on the titles of $108{,}000$ training articles.

\textbf{Confidence intervals (\cref{sec:calib-trial}).} For each model the cells are fixed in advance and the seed is the unit of replication.  So we average a difference in test or validation accuracy over the cells of each seed, nine for the image models and three for RoBERTa-base, and give the nominal $95\%$ $t$ interval, with $2$ degrees of freedom, for the mean of the three seed averages.  The exact $\sigma$ minus the closed form gives $+1.08$ points with interval $[0.75,1.41]$ on CIFAR-10, $+2.38$ points with interval $[2.26,2.50]$ for the linear classifier on CIFAR-100, $+1.58$ points with interval $[0.70,2.46]$ on SST-2 and $+0.43$ points with interval $[0.34,0.53]$ on AG News.  The exact $\sigma$ minus the accountant's $\sigma$ gives $-0.005$ points with interval $[-0.18,0.17]$ on CIFAR-10, $+0.053$ points with interval $[0.040,0.065]$ on CIFAR-100, $0.000$ points with interval $[-0.19,0.19]$ on SST-2 and $+0.006$ points with interval $[-0.011,0.022]$ on AG News, all far below the differences to the closed form.

\begin{table*}[t]
\caption{Noise multipliers and test accuracy in each cell of \cref{tab:calib-trial}.  The three values of $\sigma$ depend only on $M$, $E$ and $\lambda=0.10$, so every model uses the same $\sigma$ in a cell.  Test accuracy is in percent (validation accuracy for SST-2).  It is the mean over three seeds with their standard deviation.  Each difference is the exact run minus the other run, in points, averaged over the seeds and computed before rounding, so it can differ by $0.01$ from the difference of the printed means.}
\label{tab:calib-cells}
\centering
\small
\setlength{\tabcolsep}{3.5pt}
\begin{tabular}{@{}rrccccccrr@{}}
\toprule
& & \multicolumn{3}{c}{$\sigma$} & \multicolumn{3}{c}{ResNet-18 on CIFAR-10} & \multicolumn{2}{c}{exact minus} \\
\cmidrule(lr){3-5}\cmidrule(lr){6-8}\cmidrule(lr){9-10}
$M$ & $E$ & closed & exact & accountant & closed & exact & accountant & closed & accountant \\
\midrule
$100$ & $10$ & $1.7370$ & $1.4262$ & $1.4333$ & $37.63\pm1.27$ & $39.11\pm1.00$ & $39.16\pm0.96$ & $+1.48$ & $-0.05$ \\
$100$ & $20$ & $2.3621$ & $1.9071$ & $1.9169$ & $39.13\pm0.54$ & $40.28\pm0.65$ & $40.33\pm0.65$ & $+1.14$ & $-0.06$ \\
$100$ & $40$ & $3.2689$ & $2.6105$ & $2.6241$ & $40.60\pm0.63$ & $42.14\pm0.53$ & $42.20\pm0.64$ & $+1.54$ & $-0.06$ \\
$250$ & $10$ & $1.2088$ & $1.0245$ & $1.0282$ & $38.46\pm0.46$ & $39.17\pm0.35$ & $39.09\pm0.40$ & $+0.71$ & $+0.08$ \\
$250$ & $20$ & $1.5826$ & $1.3094$ & $1.3152$ & $40.13\pm0.45$ & $40.64\pm0.74$ & $40.70\pm0.75$ & $+0.51$ & $-0.06$ \\
$250$ & $40$ & $2.1352$ & $1.7328$ & $1.7411$ & $41.00\pm0.19$ & $42.43\pm0.18$ & $42.39\pm0.12$ & $+1.44$ & $+0.04$ \\
$500$ & $10$ & $0.9592$ & $0.8344$ & $0.8365$ & $38.66\pm0.08$ & $39.35\pm0.67$ & $39.27\pm0.66$ & $+0.69$ & $+0.08$ \\
$500$ & $20$ & $1.2093$ & $1.0260$ & $1.0298$ & $40.17\pm0.60$ & $40.71\pm0.67$ & $40.63\pm0.45$ & $+0.54$ & $+0.08$ \\
$500$ & $40$ & $1.5829$ & $1.3101$ & $1.3155$ & $41.33\pm0.57$ & $42.97\pm0.52$ & $43.06\pm0.44$ & $+1.64$ & $-0.09$ \\
\bottomrule
\end{tabular}
\par\medskip
\begin{tabular}{@{}rrccccccrr@{}}
\toprule
& & \multicolumn{3}{c}{$\sigma$} & \multicolumn{3}{c}{linear classifier on CIFAR-100} & \multicolumn{2}{c}{exact minus} \\
\cmidrule(lr){3-5}\cmidrule(lr){6-8}\cmidrule(lr){9-10}
$M$ & $E$ & closed & exact & accountant & closed & exact & accountant & closed & accountant \\
\midrule
$100$ & $10$ & $1.7370$ & $1.4262$ & $1.4333$ & $52.56\pm0.46$ & $53.72\pm0.58$ & $53.68\pm0.63$ & $+1.16$ & $+0.04$ \\
$100$ & $20$ & $2.3621$ & $1.9071$ & $1.9169$ & $55.23\pm0.18$ & $56.92\pm0.15$ & $56.87\pm0.16$ & $+1.69$ & $+0.05$ \\
$100$ & $40$ & $3.2689$ & $2.6105$ & $2.6241$ & $56.14\pm0.34$ & $58.83\pm0.22$ & $58.77\pm0.22$ & $+2.69$ & $+0.06$ \\
$250$ & $10$ & $1.2088$ & $1.0245$ & $1.0282$ & $53.63\pm0.20$ & $55.08\pm0.17$ & $55.05\pm0.16$ & $+1.45$ & $+0.03$ \\
$250$ & $20$ & $1.5826$ & $1.3094$ & $1.3152$ & $55.05\pm0.18$ & $57.40\pm0.26$ & $57.35\pm0.25$ & $+2.35$ & $+0.05$ \\
$250$ & $40$ & $2.1352$ & $1.7328$ & $1.7411$ & $54.17\pm0.57$ & $57.64\pm0.53$ & $57.55\pm0.53$ & $+3.47$ & $+0.09$ \\
$500$ & $10$ & $0.9592$ & $0.8344$ & $0.8365$ & $52.34\pm0.34$ & $54.18\pm0.20$ & $54.14\pm0.22$ & $+1.84$ & $+0.04$ \\
$500$ & $20$ & $1.2093$ & $1.0260$ & $1.0298$ & $53.05\pm0.72$ & $55.86\pm0.74$ & $55.81\pm0.72$ & $+2.81$ & $+0.05$ \\
$500$ & $40$ & $1.5829$ & $1.3101$ & $1.3155$ & $50.02\pm0.22$ & $54.00\pm0.05$ & $53.94\pm0.02$ & $+3.98$ & $+0.06$ \\
\bottomrule
\end{tabular}
\par\medskip
\begin{tabular}{@{}rrccccccrr@{}}
\toprule
& & \multicolumn{3}{c}{$\sigma$} & \multicolumn{3}{c}{RoBERTa-base on SST-2 (validation accuracy)} & \multicolumn{2}{c}{exact minus} \\
\cmidrule(lr){3-5}\cmidrule(lr){6-8}\cmidrule(lr){9-10}
$M$ & $E$ & closed & exact & accountant & closed & exact & accountant & closed & accountant \\
\midrule
$100$ & $20$ & $2.3621$ & $1.9071$ & $1.9169$ & $86.85\pm1.32$ & $88.49\pm0.98$ & $88.49\pm1.08$ & $+1.64$ & $0.00$ \\
$250$ & $20$ & $1.5826$ & $1.3094$ & $1.3152$ & $86.28\pm1.10$ & $87.61\pm0.41$ & $87.58\pm0.48$ & $+1.34$ & $+0.04$ \\
$500$ & $20$ & $1.2093$ & $1.0260$ & $1.0298$ & $84.56\pm1.13$ & $86.31\pm1.16$ & $86.35\pm1.20$ & $+1.76$ & $-0.04$ \\
\bottomrule
\end{tabular}
\par\medskip
\begin{tabular}{@{}rrccccccrr@{}}
\toprule
& & \multicolumn{3}{c}{$\sigma$} & \multicolumn{3}{c}{RoBERTa-base on AG News} & \multicolumn{2}{c}{exact minus} \\
\cmidrule(lr){3-5}\cmidrule(lr){6-8}\cmidrule(lr){9-10}
$M$ & $E$ & closed & exact & accountant & closed & exact & accountant & closed & accountant \\
\midrule
$100$ & $20$ & $2.3621$ & $1.9071$ & $1.9169$ & $89.60\pm0.53$ & $90.13\pm0.47$ & $90.13\pm0.48$ & $+0.54$ & $0.00$ \\
$250$ & $20$ & $1.5826$ & $1.3094$ & $1.3152$ & $89.74\pm0.15$ & $90.09\pm0.06$ & $90.09\pm0.07$ & $+0.36$ & $0.00$ \\
$500$ & $20$ & $1.2093$ & $1.0260$ & $1.0298$ & $89.50\pm0.23$ & $89.91\pm0.22$ & $89.90\pm0.22$ & $+0.41$ & $+0.01$ \\
\bottomrule
\end{tabular}

\end{table*}

\textbf{Licenses of the datasets and models.} The download pages of CIFAR-10 and CIFAR-100 \citep{Krizhevsky2009} (\url{https://www.cs.toronto.edu/~kriz/cifar.html}) and of the Stanford Sentiment Treebank \citep{SocherEtAl2013} (\url{https://nlp.stanford.edu/sentiment/}) state no license, and the CIFAR page asks users to cite the technical report.  The GLUE benchmark \citep{WangEtAl2019GLUE} refers users of SST-2 to the license of the original dataset (\url{https://gluebenchmark.com/faq}).  AG News \citep{ZhangZhaoLeCun2015} is built from the AG corpus of A.~Gulli, whose page permits any non-commercial use and leaves the copyright of the articles with the original news sources (\url{http://groups.di.unipi.it/~gulli/AG_corpus_of_news_articles.html}).  We use the torchvision ResNet-50 weights IMAGENET1K\_V2.  The torchvision code is released under the BSD-3-Clause license, torchvision states no separate license for these weights, and the ImageNet terms of access allow only non-commercial research and educational use.  RoBERTa-base \citep{LiuEtAl2019RoBERTa} is released through fairseq, whose README states that its MIT license also applies to the pretrained models (\url{https://github.com/facebookresearch/fairseq}).  Our use of all these assets is non-commercial research, and we do not redistribute any of the datasets or weights.

\textbf{Computing infrastructure and software.} The experiments ran on an institutional cluster.  Each training run of ResNet-18, of the linear classifier on CIFAR-100 and of RoBERTa-base used one NVIDIA GPU, an RTX A4000, A5000 or A6000, and the models of the attack trials were trained on RTX A5000 and RTX A6000 GPUs.  The linear model of \cref{sec:attacks} was trained without a GPU, on AMD EPYC 7402P processors, and the accountant runs and the timings of the calibration methods ran on the same processors.  The code uses PyTorch \citep{PaszkeEtAl2019}, torchvision, Opacus \citep{YousefpourEtAl2021} for the per-record gradients, Transformers \citep{WolfEtAl2020} for RoBERTa-base, NumPy \citep{HarrisEtAl2020} and SciPy \citep{VirtanenEtAl2020}, and the certificate of Appendix~\ref{app:certificate} uses python-flint and FLINT \citep{FLINT}.  The accountant of \citet{FeldmanShenfeld2026} is the library PLD\_accounting $0.5.0$, and FS uses the library dp\_accounting $0.6.0$ of the Google Differential Privacy Team.  PyTorch, NumPy and SciPy are released under BSD licenses, Opacus, Transformers and dp\_accounting under the Apache License 2.0, PLD\_accounting and python-flint under the MIT license, and FLINT under the GNU Lesser General Public License, version 3 or later.

\textbf{The ResNet-18 attack.} In the attack trials $D$ consists of $44{,}999$ CIFAR-10 training images and $D^+$ adds the target, so $|D|+1=N_0$.  The models of these trials reach $37.8\%$ test accuracy on average.  The target is CIFAR-10 test image $0$ with its label changed from $3$ to $4$.  The seven candidate scores are the target's negative loss, its logit-scaled confidence $\log\{p_y/(1-p_y)\}$, where $p_y$ is the probability the model gives the assigned label $y$, and its margin, each on the image itself and averaged over six fixed views, and minus the norm of its gradient.  The views are the image, its horizontal flip and its circular shifts by two pixels up, down, left and right.  Four candidates tied for the highest balanced accuracy on the pilot trials, each at its best threshold, and the selection took the first in a fixed order, the margin, with the threshold $-1.16715$.  On the $150$ evaluation trials the attack is correct on $104$ of the models trained with the target and on $50$ of those trained without it.  The attack classifies both models of a trial correctly in $4$ trials and exactly one of them in the other $146$.  Let $p$ be the probability that the attack classifies both models of a trial correctly and $p_0$ the probability that it classifies neither.  The balanced accuracy is half the expected number of models of a trial classified correctly, so it equals $\frac12\{(1-p-p_0)+2p\}=\frac12+\frac12(p-p_0)\le\frac12+\frac12p$.  The one-sided $97.5\%$ Clopper--Pearson bound \citep{ClopperPearson1934} on $p$, from $4$ of $150$ trials, puts the balanced accuracy at most $0.534$.  For $k$ successes in $n$ independent trials, the one-sided upper Clopper--Pearson bound at level $\alpha$ is the $p$ with $\Pr\{\mathrm{Bin}(n,p)\le k\}=\alpha$ if $k<n$ and $1$ if $k=n$.  The standard errors in \cref{tab:attacks} are those of the mean over trials of the fraction of the two models of a trial that the attack classifies correctly, which is the balanced accuracy.

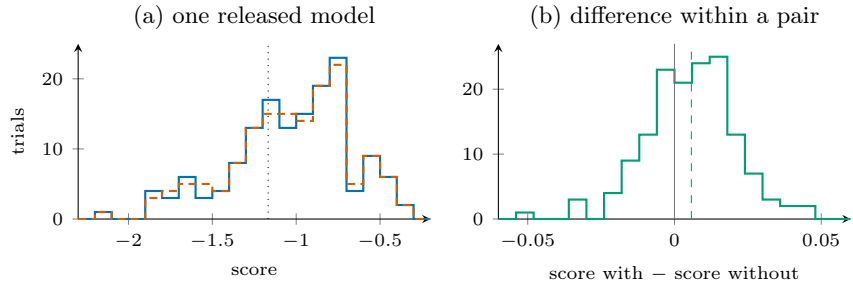
\begin{figure}[H]
\centering
\begin{minipage}{0.7\textwidth}
\centering
\begin{tikzpicture}
\begin{axis}[  height=3.9cm, width=0.52\columnwidth, scale only axis=false,
  tick label style={font=\scriptsize}, label style={font=\scriptsize},
  title style={font=\footnotesize, yshift=-4pt}, axis lines=left, ymin=0,
  name=a, ymax=25, xmin=-2.3, xmax=-0.2,
  title={(a) one released model}, xlabel={score}, ylabel={trials},
  xtick={-2,-1.5,-1,-0.5},
]
\addplot[const plot, oiBlue, thick] coordinates {(-2.3000,0) (-2.2000,1) (-2.1000,0) (-2.0000,0) (-1.9000,4) (-1.8000,3) (-1.7000,6) (-1.6000,3) (-1.5000,4) (-1.4000,8) (-1.3000,13) (-1.2000,17) (-1.1000,13) (-1.0000,15) (-0.9000,19) (-0.8000,23) (-0.7000,4) (-0.6000,9) (-0.5000,6) (-0.4000,2) (-0.3000,0) (-0.2000,0)};
\addplot[const plot, oiOrange, thick, densely dashed] coordinates {(-2.3000,0) (-2.2000,1) (-2.1000,0) (-2.0000,0) (-1.9000,3) (-1.8000,4) (-1.7000,5) (-1.6000,5) (-1.5000,4) (-1.4000,8) (-1.3000,13) (-1.2000,15) (-1.1000,15) (-1.0000,14) (-0.9000,19) (-0.8000,22) (-0.7000,5) (-0.6000,9) (-0.5000,6) (-0.4000,2) (-0.3000,0) (-0.2000,0)};
\addplot[black, thin, dotted, forget plot] coordinates {(-1.1671,0) (-1.1671,25)};
\end{axis}
\begin{axis}[  height=3.9cm, width=0.52\columnwidth, scale only axis=false,
  tick label style={font=\scriptsize}, label style={font=\scriptsize},
  title style={font=\footnotesize, yshift=-4pt}, axis lines=left, ymin=0,
  at={(a.east)}, anchor=west, xshift=0.9cm,
  ymax=27, xmin=-0.06, xmax=0.06,
  title={(b) difference within a pair}, xlabel={score with $-$ score without},
  xtick={-0.05,0,0.05}, scaled x ticks=false,
  xticklabel style={/pgf/number format/fixed, /pgf/number format/precision=2},
]
\addplot[const plot, oiGreen, thick] coordinates {(-0.0600,0) (-0.0540,1) (-0.0480,0) (-0.0420,0) (-0.0360,3) (-0.0300,0) (-0.0240,4) (-0.0180,9) (-0.0120,13) (-0.0060,23) (0.0000,21) (0.0060,24) (0.0120,25) (0.0180,13) (0.0240,7) (0.0300,3) (0.0360,2) (0.0420,2) (0.0480,0) (0.0540,0) (0.0600,0)};
\addplot[oiGray, thin, forget plot] coordinates {(0,0) (0,27)};
\addplot[oiGreen, thin, dashed, forget plot] coordinates {(0.00576,0) (0.00576,27)};
\end{axis}
\end{tikzpicture}
\end{minipage}
\caption{The membership score of the target in $150$ trials of ResNet-18.  (a) The score of one released model, trained with the target (blue, solid) or without it (orange, dashed).  The dotted line is the threshold chosen on the pilot trials.  (b) The difference between the two models of each trial.  It is positive on average, but much smaller than the spread in (a).  The two panels have different horizontal scales.}
\label{fig:attack}
\end{figure}

\Cref{fig:attack} shows why the margin attack stays near chance.  In $97$ of the $150$ trials, adding the record raises the score of the target (paired $t=4.54$), but the standard deviation of the score across trials is $62$ times its average change.  The attacker sees only one model, and in $146$ of the $150$ trials both models of the trial land on the same side of the threshold.

The second attack of \cref{tab:attacks} was chosen on $24$ pilot trials, eight for each of three canaries (a white image, a black image and the ClipBKD image), over the same seven scores, by the largest paired $t$.  It selected the ClipBKD canary, an image along the direction in which the public CIFAR-10 images vary least, whose label is the class a reference model trained on public data finds least likely \citep[Algorithm~4]{JagielskiEtAl2020}, the averaged logit-scaled confidence and the threshold $-2.0553$, the midpoint of the pilot means.  On the $150$ evaluation trials it is correct on $86$ models trained with the canary and on $67$ trained without it.  The attack classifies both models of a trial correctly in $3$ trials and exactly one of them in the other $147$, so the same bound puts its balanced accuracy at most $0.529$.  The second attack has a larger paired $t$ than the margin attack, $8.5$ against $4.54$, yet a smaller balanced accuracy from one model.

\textbf{The linear model.} Without the target, $w\sim N(0,EM\sigma^2C^2I_{16})$.  With the target, $w\sim N(EC\hat x,EM\sigma^2C^2I_{16})$, since every epoch contributes $C\hat x$ once, regardless of its allocated step.  Put $T=\langle w,\hat x\rangle/(\sigma C\sqrt{EM})$ and $\mu=\sqrt{E/M}/\sigma$.  Since $\hat x$ is a unit vector, $T$ is $N(0,1)$ without the target and $N(\mu,1)$ with it.  The likelihood ratio of the second distribution of $w$ to the first is $\exp\bigl((\langle w,m\rangle-\norm{m}^2/2)/(EM\sigma^2C^2)\bigr)$ with $m=EC\hat x$, which is $\exp(\mu T-\mu^2/2)$.  By the Neyman--Pearson lemma \citep{NeymanPearson1933} the optimal tests threshold this ratio, and with the two hypotheses equally likely, \eqref{eq:test-gap} in Appendix~\ref{app:background} shows that the optimal test declares membership when the ratio exceeds $1$.  Here that is the test $T>\mu/2$, and it is correct with probability $\Phi(\mu/2)$ under each hypothesis.  The model has $d=16$ and is trained with $C=1$, $N_0=16$, $M=128$, $E=125$, the constant learning rate $B_0$ and no momentum.  The $N_0$ indices hold the twelve zero-feature records of $D$, the target under $D^+$, and ghost records.  A step holds none of the $12$ records of $D$ with probability $(1-1/M)^{12}=0.910$, and in each setting about $8.7$ million of the $9.6$ million steps add noise to an empty sum (line~\ref{line:noise} of \cref{alg:dpsgd}).  When the target has zero features, the two runs release the same model and the attack is exactly at chance.

\FloatBarrier
\section{Numerical Checks}
\label{app:numerics}

\begin{figure}[H]
\centering
\begin{tikzpicture}
\begin{axis}[
  width=0.62\textwidth, height=5.4cm,
  xmode=log, log basis x=2,
  xlabel={steps per epoch $M$},
  ylabel={attack accuracy on the pair},
  xtick={64,128,256,512,1024},
  xticklabels={64,128,256,512,1024},
  xmin=56, xmax=1160, ymin=0.5, ymax=0.565,
  ytick={0.50,0.51,0.52,0.53,0.54,0.55,0.56},
  yticklabel style={/pgf/number format/fixed,/pgf/number format/precision=2},
  tick label style={font=\footnotesize},
  label style={font=\footnotesize},
  legend style={font=\footnotesize, at={(0.5,-0.28)}, anchor=north,
                legend columns=2, draw=oiGray, /tikz/every even column/.append style={column sep=8pt}},
  axis lines=left,
  every axis plot/.append style={very thick},
]
\addplot[oiGray, dashed, thick, domain=56:1160, samples=2, forget plot] {0.5};
\addplot[oiBlue!55, thin, no marks] coordinates {(64,0.54559) (128,0.54744) (256,0.54744) (512,0.54699) (1024,0.54699)};
\addlegendentry{linear $E\propto M$: GDP}
\addplot[oiBlue, only marks, mark=*, mark size=1.9pt,
  error bars/.cd, y dir=both, y explicit] coordinates {(64,0.54544) +- (0,0.00026) (128,0.54730) +- (0,0.00027) (256,0.54729) +- (0,0.00027) (512,0.54690) +- (0,0.00027) (1024,0.54702) +- (0,0.00027)};
\addlegendentry{linear $E\propto M$: empirical}
\addplot[oiOrange!55, thin, no marks] coordinates {(64,0.55260) (128,0.54366) (256,0.53725) (512,0.53159) (1024,0.52636)};
\addlegendentry{sublinear $E\approx\sqrt M$: GDP}
\addplot[oiOrange, only marks, mark=square*, mark size=1.7pt,
  error bars/.cd, y dir=both, y explicit] coordinates {(64,0.55254) +- (0,0.00030) (128,0.54350) +- (0,0.00025) (256,0.53705) +- (0,0.00022) (512,0.53155) +- (0,0.00019) (1024,0.52630) +- (0,0.00016)};
\addlegendentry{sublinear $E\approx\sqrt M$: empirical}
\node[font=\scriptsize, oiGray, anchor=south west] at (axis cs:56,0.5005) {chance $=\tfrac12$};
\end{axis}
\end{tikzpicture}%
\caption{Monte Carlo estimates of $\operatorname{Acc}^*(P_M^{\otimes E},Q_M^{\otimes E})$ at $\sigma=1.5$ along $E\approx\sqrt M$ (sublinear) and $E\approx0.10M$ (linear), with $2\times10^5$ samples per point, against the prediction $\Phi(\mu/2)$ of \cref{thm:fixed-rate} with $\mu=\sqrt{E(e^{1/\sigma^2}-1)/M}$ at each point (lines).}
\label{fig:attack-empirical}
\end{figure}
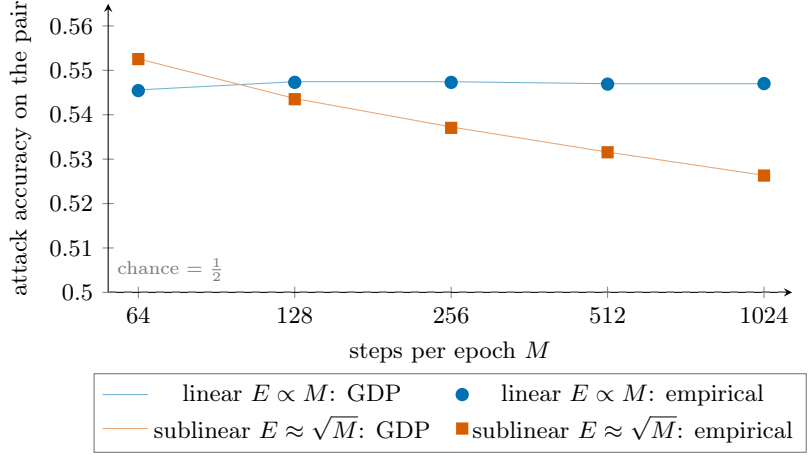

\Cref{fig:attack-empirical} checks \cref{thm:fixed-rate} at finite $M$.  Each point is $\frac12+\frac12\hat d$, with $\hat d$ the Monte Carlo estimate of the total variation described in Appendix~\ref{app:exact-numerics}, and its interval is $\pm1.96$ standard errors.  Along $E\approx\sqrt M$ the attack accuracy falls from $0.553$ at $M=64$ toward chance.  Along $E\approx0.10M$ it stays near $0.547$.  At all ten points the prediction lies inside the $95\%$ Monte Carlo interval, with a largest deviation of $2.1\times10^{-4}$.

\begin{table}[t]
\caption{Comparison of five ways of setting $\sigma$ for a total-variation target at one epoch.
\emph{Bound/exact} divides the method's bound on $d_{\TV}$ by the exact value (\cref{thm:exact-tv}) at four
settings of $(\sigma,M)$, and for the closed form also at its $\sigma$ for the nine targets.  \emph{Noise} is how
much larger the method's $\sigma$ is than the exact $\sigma$, over nine targets.  \emph{Time} is the wall-clock
time to find $\sigma$, with the scan over $q$ for FS with $q$ tuned.
FS is Theorem~4.3 of \citet{FeldmanShenfeld2025} with auxiliary Poisson rate $q$, and its rows are the values that the library dp\_accounting computes.  All rows bound the
same random-allocation pair, and Appendix~\ref{app:numerics} lists the settings.
For the accountant of \citet{FeldmanShenfeld2026} at one epoch, see \cref{tab:accountant-regimes} in Appendix~\ref{app:experiments}.}
\label{tab:methods}
\centering
\footnotesize
\setlength{\tabcolsep}{3pt}
\begin{tabular}{@{}lccc@{}}
\toprule
method & bound/exact & noise & time \\
\midrule
closed form (\cref{cor:inverse}) & $1.2533$--$1.2561$ & $7.4$--$16.7\%$ & $\approx1\,\mu$s \\
\quad at the nine targets & $1.25$--$2.13$ & & \\
exact (\cref{thm:exact-tv}) & $1$ & $0$ & $0.27$--$42$\,s \\
FS, $q=1/M$ & $1.58$--$1.75$ & $9.0$--$34.3\%$ & $104$--$343$\,s \\
FS, $q$ tuned & $1.019$--$1.586$ & $0.2$--$3.5\%$ & $608$--$2181$\,s \\
\citet{VanDijkErtan2026} & $10.7$--$20.0$ & $128$--$267\%$ & $0.19$--$0.20$\,s \\
\bottomrule
\end{tabular}
\end{table}

\textbf{Settings of \cref{tab:methods}.} The bound/exact column uses the four settings $(\sigma,M)=(0.75,16384)$, $(1,16384)$, $(1.5,65536)$ and $(1,1{,}140{,}000)$ at $E=1$.  Its second row takes the closed-form $\sigma$ of each of the nine targets below and the exact value of \eqref{eq:exact-tv-cf} at $4{,}001$ and at $8{,}001$ Hermite nodes, which agree to within $10^{-11}$.  The ratio in that row is largest where $\sigma$ is smallest (Appendix~\ref{app:chi2-underdetermines}).  The FS row with $q=1/M$ uses the accountant's finest grid, $10^{-6}$.  FS is stated for $\eps>0$, and it also holds at $\eps=0$, where it bounds the total variation, because both sides of each of its inequalities are continuous in $\eps$.  The tuned row takes, in each setting, the smallest of $18$ runs, $qM\in\{1,0.5,0.2,0.1,0.05,0.02\}$ crossed with grid widths $\{10^{-4},10^{-5},10^{-6}\}$.  \citet{VanDijkErtan2026} is their Theorem~3.1 with the exact third absolute moment in place of its simplified factor, which is allowed when $\sigma\ge\sqrt{3/\ln M}$ and $M\ge3$ \citep{VanDijkErtan2026}, as in every setting here, and with the Berry--Esseen constant $B$ set to $0.4748$, the upper end of the range $[0.4097,0.4748]$ they give for it.  The noise and time columns use nine targets, $\lambda\in\{0.01,0.05,0.1\}$ crossed with $M\in\{4096,16384,65536\}$, at $E=1$.  The exact $\sigma$ that the noise column compares with solves $d_{\TV}=\lambda$ by root search on \eqref{eq:exact-tv-cf} to within $1.1\times10^{-6}$ in $\sigma$.  The timed root searches of every method stop at a tolerance of $10^{-4}$ in $\sigma$, and their wall-clock times are the fastest of three runs where repeated.  The FS noise and time use an accountant grid of $10^{-5}$.  For these two columns the tuned row solves for $\sigma$ at each of the six rates, takes the smallest $\sigma$, and reports the total time of the six searches.  The smallest $\sigma$ comes from $qM=0.02$ at $\lambda\ge0.05$ and from $qM\in\{0.05,0.1\}$ at $\lambda=0.01$.  With the accountant's grid width set to $10^{-5}$, the exact root search is faster than FS by a factor of $7$ to $384$ at $q=1/M$ and $44$ to $2{,}223$ with tuning.  On the accountant's default grid of $10^{-4}$, the noise and time of FS at $q=1/M$ are \methodsFScoarse.  On that grid, FS at $q=1/M$ is faster than the exact root search in three of the nine cells.  At $\eps=0$ the bound of FS is $\gamma$ times the total variation between $M$ steps that each include the record with probability $q$ and $M$ steps without it, where $\gamma=1/(1-(1-q)^M)$.  The first distribution is a mixture, with weight $(1-q)^M$ on the second distribution and the rest on the distribution given at least one participation, so the bound equals the total variation between that conditioned distribution and the second distribution.  As $q\to0$ the conditioned distribution tends to one participation at a uniformly chosen step, which gives the Gaussian pair, so in exact arithmetic the bound tends to the exact value.  The library dp\_accounting evaluates the Poisson total variation on a grid, and $\gamma$, at least $1/(qM)$, multiplies the numerical error of that evaluation.  With $q$ tuned over the six rates, the library's values need $0.2\%$ to $3.5\%$ more noise than the exact $\sigma$, an excess that comes from both the positive rates and the numerical error. The closed form of \citet{VanDijkErtan2026} is applicable to seven of the nine targets, that is, only for these does some $\sigma$ satisfy both the bound on $\delta$ and condition~(2) of their Theorem~3.1.  FS at $q=1/M$ needs less noise than the closed form at the largest targets and more at the smallest.

\begin{table}[H]
\caption{Exact total variation at the closed-form $\sigma$ on $27$ configurations, checked against the
accountant and Monte Carlo.  Each $\sigma$ is from \cref{cor:inverse} for the target $\lambda$, and \emph{exact
$d_{\TV}$} is that of $(P_M^{\otimes E},Q_M^{\otimes E})$ from \cref{thm:exact-tv}.  \emph{Accountant} is the
interval between the largest lower bound and the smallest upper bound on the total variation that the
privacy loss distribution accountant of \citet{FeldmanShenfeld2026} gives in the add and the remove
direction.  \emph{Monte Carlo} uses $2\times10^7$ samples, $\pm$ one standard error.  \emph{Ratio} is the bound
$\lambda$ divided by the exact value.}
\label{tab:exact27}
\centering
\small
\begin{tabular}{@{}rrrrcccc@{}}
\toprule
$\lambda$ & $M$ & $E$ & $\sigma$ & exact $d_{\TV}$ & accountant & Monte Carlo & ratio \\
\midrule
$0.10$ & $100$ & $10$ & $1.7370$ & $0.07885$ & $[0.07434,0.08359]$ & $0.07884\pm0.00002$ & $1.2683$ \\
$0.10$ & $100$ & $20$ & $2.3621$ & $0.07887$ & $[0.07436,0.08364]$ & $0.07888\pm0.00002$ & $1.2679$ \\
$0.10$ & $100$ & $40$ & $3.2689$ & $0.07888$ & $[0.07439,0.08364]$ & $0.07889\pm0.00002$ & $1.2678$ \\
$0.10$ & $250$ & $10$ & $1.2088$ & $0.07880$ & $[0.07486,0.08286]$ & $0.07877\pm0.00002$ & $1.2691$ \\
$0.10$ & $250$ & $20$ & $1.5826$ & $0.07886$ & $[0.07487,0.08286]$ & $0.07886\pm0.00002$ & $1.2681$ \\
$0.10$ & $250$ & $40$ & $2.1352$ & $0.07887$ & $[0.07491,0.08289]$ & $0.07888\pm0.00002$ & $1.2678$ \\
$0.10$ & $500$ & $10$ & $0.9592$ & $0.07870$ & $[0.07482,0.08269]$ & $0.07865\pm0.00002$ & $1.2706$ \\
$0.10$ & $500$ & $20$ & $1.2093$ & $0.07884$ & $[0.07497,0.08286]$ & $0.07884\pm0.00002$ & $1.2684$ \\
$0.10$ & $500$ & $40$ & $1.5829$ & $0.07887$ & $[0.07500,0.08289]$ & $0.07888\pm0.00002$ & $1.2679$ \\
$0.20$ & $100$ & $10$ & $1.0458$ & $0.15193$ & $[0.14778,0.15612]$ & $0.15200\pm0.00004$ & $1.3164$ \\
$0.20$ & $100$ & $20$ & $1.3403$ & $0.15254$ & $[0.14839,0.15681]$ & $0.15243\pm0.00004$ & $1.3111$ \\
$0.20$ & $100$ & $40$ & $1.7787$ & $0.15270$ & $[0.14852,0.15699]$ & $0.15273\pm0.00004$ & $1.3098$ \\
$0.20$ & $250$ & $10$ & $0.8018$ & $0.15079$ & $[0.14712,0.15449]$ & $0.15086\pm0.00004$ & $1.3263$ \\
$0.20$ & $250$ & $20$ & $0.9752$ & $0.15223$ & $[0.14856,0.15590]$ & $0.15226\pm0.00004$ & $1.3138$ \\
$0.20$ & $250$ & $40$ & $1.2336$ & $0.15262$ & $[0.14897,0.15628]$ & $0.15257\pm0.00004$ & $1.3105$ \\
$0.20$ & $500$ & $10$ & $0.6840$ & $0.14923$ & $[0.14575,0.15277]$ & $0.14923\pm0.00004$ & $1.3402$ \\
$0.20$ & $500$ & $20$ & $0.8025$ & $0.15172$ & $[0.14816,0.15534]$ & $0.15175\pm0.00004$ & $1.3182$ \\
$0.20$ & $500$ & $40$ & $0.9757$ & $0.15249$ & $[0.14890,0.15613]$ & $0.15248\pm0.00004$ & $1.3116$ \\
$0.35$ & $100$ & $10$ & $0.7849$ & $0.24117$ & $[0.23757,0.24485]$ & $0.24118\pm0.00006$ & $1.4513$ \\
$0.35$ & $100$ & $20$ & $0.9521$ & $0.24572$ & $[0.24203,0.24944]$ & $0.24575\pm0.00006$ & $1.4244$ \\
$0.35$ & $100$ & $40$ & $1.2003$ & $0.24724$ & $[0.24352,0.25100]$ & $0.24721\pm0.00006$ & $1.4156$ \\
$0.35$ & $250$ & $10$ & $0.6437$ & $0.23524$ & $[0.23207,0.23846]$ & $0.23527\pm0.00005$ & $1.4878$ \\
$0.35$ & $250$ & $20$ & $0.7459$ & $0.24318$ & $[0.23994,0.24645]$ & $0.24319\pm0.00006$ & $1.4393$ \\
$0.35$ & $250$ & $40$ & $0.8930$ & $0.24643$ & $[0.24314,0.24970]$ & $0.24644\pm0.00006$ & $1.4203$ \\
$0.35$ & $500$ & $10$ & $0.5716$ & $0.22923$ & $[0.22620,0.23233]$ & $0.22919\pm0.00005$ & $1.5268$ \\
$0.35$ & $500$ & $20$ & $0.6449$ & $0.24004$ & $[0.23692,0.24320]$ & $0.24001\pm0.00006$ & $1.4581$ \\
$0.35$ & $500$ & $40$ & $0.7467$ & $0.24523$ & $[0.24206,0.24842]$ & $0.24533\pm0.00006$ & $1.4272$ \\
\bottomrule
\end{tabular}
\end{table}

\FloatBarrier
\section{Related Work in Full}
\label{app:related}
\textbf{Random allocation.} Random allocation gives each record a fixed number of participations \citep{ChuaBallsBins2025,FeldmanShenfeld2025,DongChenOzgur2025}.  \citet{ChoquetteChooEtAl2025} propose the same sampling for one epoch as balls-in-bins batching, reuse the same batches in every epoch, and compute $(\eps,\delta)$ guarantees of matrix mechanisms with correlated noise by Monte Carlo accounting.  \citet{SchuchardtKalinin2026} compute such guarantees without sampling, from R\'enyi divergence bounds and conditional composition.  We study DP-SGD with independent noise and a fresh allocation in every epoch, as \citet{ChuaBallsBins2025} and \citet{FeldmanShenfeld2025} do, and bound the total variation in closed form.  \citet{KoskelaEtAl2023} study the Gaussian pair \eqref{eq:one-record-pair}, scaled by $\sigma$, in the shuffle model and estimate its hockey-stick divergence by Monte Carlo integration in up to seven dimensions.  Random allocation is related to random check-ins, in which records choose their participation times independently and may abstain \citep{BalleEtAl2020}.  \citet{DongChenOzgur2025} and \citet{DongOzgur2026} study balanced iteration subsampling, in which each record takes part in $k$ of $T$ steps chosen uniformly at random.  One epoch of random allocation is the case $k=1$ and $T=M$, and our $E$ epochs compose independent copies of it.  Balanced iteration subsampling with $k=E$ and $T=EM$ instead spreads the $E$ steps of a record over all $EM$ steps and need not use the record in every epoch.  Related schemes are truncated Poisson sampling \citep{Ganesh2025} and subsampling for banded matrix mechanisms \citep{DongGanesh2026}.

\textbf{Attack-based measures of privacy.} Total variation is the optimal advantage $2\operatorname{Acc}^*-1$ \citep{YeomEtAl2018} of a membership test for one pair of neighboring datasets \citep{ChatzikokolakisEtAl2023}, and \citet{GhaziIssa2024} study it as a privacy parameter.  For DP-SGD with Poisson subsampling, \citet{MahloujifarEtAl2022} bound this advantage by the total variation between a Gaussian and a Gaussian mixture, which they estimate by Monte Carlo, and \citet{KulynychEtAl2024} search with a privacy loss distribution accountant for the $\sigma$ that meets a target advantage or target attack error rates. \citet{CherubinEtAl2024} give closed-form record-inference bounds, up to an error term, for DP-SGD with Poisson sampling when all intermediate models are released, and \citet{GomezEtAl2026} advocate reporting guarantees as Gaussian differential privacy.  Gaussian differential privacy and its central limit theorem are due to \citet{DongRothSuGDP}.

\textbf{Setting parameters before training.} \citet{VanDijkEtAl2024ProactiveDP} also choose DP-SGD parameters before training, from a closed-form $(\eps,\delta)$ guarantee for Poisson subsampling derived with a generalized moments accountant.

\FloatBarrier
\section{Limitations in Detail}
\label{app:limitations}

Our bound \eqref{eq:main-bound} concerns one added or removed record, and by the triangle inequality the total-variation bound for a group of $k$ records is $k$ times as large.  A bound $\lambda$ on the total variation says little about tests with a small false-positive rate $\alpha$ \citep{CarliniEtAl2022,KulynychEtAl2024}, because the bound allows such a test a true-positive rate up to $\alpha+\lambda$.  The chi-square value $c_E$ gives a limit that shrinks with $\alpha$.  Every test on the released model with false-positive rate at most $\alpha\le\frac12$ has true-positive rate at most $\alpha+\sqrt{c_E\alpha(1-\alpha)}$ (Appendix~\ref{app:low-fpr}).  At $\alpha=10^{-3}$ and the exact $\sigma$ of our training cells for $\lambda=0.10$, the chi-square limit is at most $0.0092$, while the total variation alone allows $0.101$.  Total variation alone gives no $(\eps,\delta)$ guarantee with $\delta<\lambda$ (Appendix~\ref{app:background}), and a release that needs one can use a privacy-profile accountant.

\end{document}